\pdfoutput=1
\documentclass{article}

\usepackage{arxiv}

\usepackage[numbers]{natbib}
\usepackage{amsmath}
\usepackage{amssymb}
\usepackage{multirow}
\usepackage{amsfonts}
\usepackage{mathrsfs}
\usepackage{booktabs}
\usepackage{algorithmicx}
\usepackage{algorithm}
\usepackage{booktabs}
\usepackage{algpseudocode}
\usepackage{amsthm}
\usepackage{latexsym}
\usepackage{graphicx}
\usepackage{rotating}
\usepackage{subfig}
\usepackage{color,varioref}
\usepackage{longtable}
\allowdisplaybreaks[4]
\usepackage{hyperref}
\hypersetup{colorlinks,linkcolor=blue}
\def\keywords{\vspace{.5em}
	{\bf{Keywords}:\,\relax%
}}

\newtheorem{lemma}{Lemma}
\newtheorem{remark}{Remark}

\usepackage{geometry}
\title{An Improved Optimal Proximal Gradient Algorithm for Non-Blind Image Deblurring}

\author{
 Qingsong Wang \\
  School of Mathematics and Computational Science\\
  Xiangtan University\\
  Hunan, China 411105 \\
  \texttt{nothing2wang@hotmail.com} \\
   \And
 Shengze Xu \\
  Department of Mathematics\\
  The Chinese University of Hong Kong\\
  Shatin, Hong Kong \\
  \texttt{szxu@math.cuhk.edu.hk} \\
  \And
 Xiaojiao Tong \\
  School of Mathematics and Computational Science\\
  Xiangtan University\\
  Hunan, China 411105 \\
  \texttt{dysftxj@hnfnu.edu.cn} \\
     \And
 Tieyong Zeng\thanks{Corresponding author. Email: \texttt{zeng@math.cuhk.edu.hk}} \\
  Department of Mathematics\\
  The Chinese University of Hong Kong\\
  Shatin, Hong Kong \\
  \texttt{zeng@math.cuhk.edu.hk} \\
}

\begin{document}
\maketitle
\begin{abstract}
Image deblurring remains a central research area within image processing, critical for its role in enhancing image quality and facilitating clearer visual representations across diverse applications. This paper tackles the optimization problem of image deblurring, assuming a known blurring kernel. We introduce an improved optimal proximal gradient algorithm (IOptISTA), which builds upon the optimal gradient method and a weighting matrix, to efficiently address the non-blind image deblurring problem. Based on two regularization cases, namely the $l_1$ norm and total variation norm, we perform numerical experiments to assess the performance of our proposed algorithm. The results indicate that our algorithm yields enhanced PSNR and SSIM values, as well as a reduced tolerance, compared to existing methods. 
\end{abstract}


\section{Introduction}
Image deblurring, a critical process for restoring clarity to images obscured by factors such as motion blur or defocus, extends its utility across a multitude of domains, including image/video processing \cite{ZhangLGWL24, ZhangQZ24, LiSXZ22}, signal processing \cite{ZhangWZZL20, TangXCZ18, WangJJGY24, EqtedaeiA24}, traffic monitoring \cite{AliyanB12, YangFZWYY23, WangZZLZ24}, and remote sensing \cite{ZhangZLHC22, LiLZGW24, LiuYLJCP23}. Typically, image deblurring can be formulated as a linear inverse problem \cite{BenningB18, ArridgeMOS19}, that is,
\begin{eqnarray}
	b=Ax+\varepsilon, \label{linear-equation}
\end{eqnarray}
where $A\in\mathbb{R}^{m\times n}$ represents the blur kernel, $b\in\mathbb{R}^{m}$ is observed the blurred image, $\varepsilon$ accounts for noise. However, solving this inverse problem to recover the original image is often fraught with challenges due to its ill-conditioned nature, which can result in non-unique solutions. To mitigate this, we formulate an optimization problem:
\begin{eqnarray}
	\min_{x\in\mathbb{R}^{n}}  f(x)+h(x), \label{target-function}
\end{eqnarray}
where $f(x):=\frac{1}{2}\|Ax-b\|^{2}$, $h(x)$ is the regularization component, which addresses the ill-posed nature of the problem by imposing specific properties or constraints on the solution $x$, such as $l_{1}$ norm \cite{Tibshirani96}, total variation (TV) norm \cite{RudinOF92}. In this paper, we only consider the convex case for simplicity.

Over the years, a multitude of algorithms has been developed to address the non-blind deblurring problem. Broadly, deblurring methods can be categorized into optimization-based approaches \cite{Beck17, Nesterov18} and deep learning techniques \cite{QuanLXNJ22, BarmanD24, HuiLLJPCZ24}. In this paper, we focus exclusively on optimization-based methods, which provide a structured framework for solving the convex optimization formulation of the problem. One of the most prominent algorithms in this category is the iterative shrinkage-thresholding algorithm (ISTA) \cite{DaubechiesDM04}, which is also known as the proximal gradient algorithm and is defined as follows:
\begin{eqnarray}
x_{k+1} = \mathrm{Prox}_{\frac{1}{L}h} (x_{k}-\frac{1}{L}\nabla f(x_{k})), \label{ISTA}
\end{eqnarray}
where $1/L>0$ is the step size, $L$ is the gradient Lipschitz constant of function $f$, and $\mathrm{Prox}_{\eta h}(\cdot)$ is the proximal operator which is defined as $\mathrm{Prox}_{\eta h}(x):=\mathrm{argmin}_{y\in\mathbb{R}^{d}}\, h(y) +\frac{1}{2\eta}\|y-x\|^{2}$. The convergence of the ISTA is relatively slow, achieving a rate of $\mathcal{O}(\frac{1}{k})$. To address this issue, Beck et al. \cite{BeckT09} introduced a modified version of ISTA called the fast iterative shrinkage thresholding algorithm (FISTA), namely the accelerated proximal gradient algorithm, which is defined as:
\begin{eqnarray}
\begin{aligned}
x_{k+1} =& \mathrm{Prox}_{\frac{1}{L} h} (y_{k}-\frac{1}{L}\nabla f(y_{k})),\\		\alpha_{k+1}=&\frac{1+\sqrt{1+4\alpha_{k}^{2}}}{2},\\
y_{k+1}=&x_{k}+\frac{\alpha_{k}-1}{\alpha_{k+1}}(x_{k}-x_{k-1}),
\end{aligned} \label{FISTA}	
\end{eqnarray}
where $\alpha_{1}=1$, and FISTA archives a faster convergence rate of $\mathcal{O}(\frac{1}{k^{2}})$. Based on the proximal gradient framework, the  optimized gradient method (OGM) \cite{KimF16} can also applied to solve the optimization problem \eqref{target-function} as:
\begin{eqnarray}
\begin{aligned}
y_{k+1}=&\mathrm{Prox}_{\eta h}(x_{k}-\frac{1}{L}\nabla f(x_{k})),\\
x_{k+1}=&y_{k+1}+\frac{\alpha_{k}-1}{\alpha_{k+1}}(y_{k+1}-y_{k})+ \frac{\alpha_{k}}{\alpha_{k+1}}(y_{k+1}-x_{k}).
\end{aligned} \label{POGM}
\end{eqnarray}

We denote this algorithm as the proximal optimized gradient method (POGM). We know the term $\frac{\alpha_{k}-1}{\alpha_{k+1}}(y_{k+1}-y_{k})$ can be viewed as the momentum term, while $\frac{\alpha_{k}}{\alpha_{k+1}}(y_{k+1}-x_{k})$ is called the correction term. However, the convergence of this algorithm is unclear now. Recently, to improve the convergence rate of the FISTA algorithm, Jang et al. \cite{JangGR23} proposed the optimal iterative shrinkage thresholding algorithm (OptISTA), which is defined as:
\begin{eqnarray}
\begin{aligned}
y_{k+1}=&\mathrm{Prox}_{\gamma_{k}\frac{1}{L} h}(y_{k}-\gamma_{k}\frac{1}{L} \nabla f(x_{k})),\\
z_{k+1}=&x_{k}+\frac{1}{\gamma_{k}}(y_{k+1}-y_{k}),\\
x_{k+1}=&z_{k+1}+\frac{\alpha_{k}-1}{\alpha_{k+1}}(z_{k+1}-z_{k})+\frac{\alpha_{k}}{\alpha_{k+1}}(z_{k+1}-x_{k}),
\end{aligned}\label{OptISTA}
\end{eqnarray}
where $\gamma_{k} > 0$. It improves the leading convergence rate of FISTA by a factor of 2 in the $\mathcal{O}\left(\frac{1}{k^2}\right)$ term, although their analysis method is quite complex. When $h(x) = 0$, the OptISTA algorithm simplifies to the optimized gradient method (OGM) \cite{KimF16}. However, the OptISTA algorithm was only analyzed theoretically, and the efficiency of the algorithm was not verified through any numerical experiments.
Other accelerated algorithms, such as FPGM2 \cite{Taylor17} and APIRL1-AM \cite{AdamMHMN23}, can also be employed to solve the optimization problem \eqref{target-function}.

Alternatively, to further enhance the convergence speed of FISTA, Zulfiquar et al. \cite{BhottoAS15} introduced the improved FISTA algorithm (IFISTA). IFISTA accelerates the optimization process by advancing the least-squares gradient descent by $n$ steps, formulated as: $x_{k+1} = \mathrm{Prox}_{\frac{1}{L} h}(y_{k} - \frac{1}{L} W_{n} \nabla f(y_{k}))$ in \eqref{FISTA}, where the weighting matrix is given by:
\begin{eqnarray}
W_{n} &:=& \sum_{i=1}^{n} \binom{n}{i} (-1)^{i-1} \left(\frac{1}{L} A^{T}A\right)^{i-1} \label{weighting_matrix}.
\end{eqnarray}

Based on this framework, Wang et al. \cite{WangWWGC18} proposed the IFISTA beyond Nesterov's momentum (IFISTA-BN) method to accelerate IFISTA by using a series of over-relaxation parameters instead of using the momentum term. Note that IFISTA and IFISTA-BN’s effectiveness was primarily demonstrated in scenarios with low additive white Gaussian noise. To address this limitation, Kumar and Sahoo \cite{KumarS24} proposed the enhanced FISTA (EFISTA) algorithm, which achieves a similar convergence rate.

From the details of the algorithms mentioned above, we observe that the use of momentum techniques or weighting matrix can significantly accelerate convergence and enhance numerical performance. Furthermore, the optimized gradient method (OGM) achieves a faster convergence rate than the Nesterov acceleration method.

In this paper, we introduce two innovative algorithms that integrate the weighting matrix technique within the OGM framework to tackle the optimization problem \eqref{target-function}. Specifically,
\begin{itemize}
    \item \textbf{Well-defined Efficient Algorithm}: To expedite the algorithm's flexibility and convergence, we propose an improved optimal proximal gradient algorithm (referred to as Algorithm \ref{IOptISTA}) designed to address the convex optimization problem \eqref{target-function} effectively.
    \item \textbf{Superior Performance in Non-blind Image Deblurring}: Leveraging the $l_1$ norm and total variation (TV) norm, we demonstrate the numerical efficacy of our proposed algorithms on non-blind image deblurring tasks. Our methods consistently outperform existing algorithms in terms of numerical results.
\end{itemize}

The structure of the subsequent sections of this paper is as follows: Section \ref{algorithms-part} delineates the intricacies of the proposed algorithm (Algorithm \ref{IOptISTA}). Section \ref{numerical_part} showcases experimental outcomes on real-world datasets, thereby highlighting the superior performance of the proposed algorithm across two distinct regularization paradigms. Finally, Section \ref{conclusion-part} encapsulates our conclusions and proposes avenues for future research.


\section{Algorithm} \label{algorithms-part}
In this section, we develop an efficient algorithm to address the convex optimization problem \eqref{target-function}, aiming to achieve superior numerical performance. A natural approach is to incorporate the weighting matrix $W_{n}$ from \eqref{weighting_matrix} into the proximal gradient descent step as shown in \eqref{OptISTA}, i.e.,
\begin{eqnarray}
y_{k+1} &=& \mathrm{Prox}_{\gamma_{k}\eta_{k} h}(y_{k} - \gamma_{k}\eta_{k} W_{n} \nabla f(x_{k})),
\end{eqnarray}
which is equivalent to:
\begin{eqnarray}
y_{k+1} &=& \underset{y}{\text{argmin}} \, h(y) + \langle y - y_{k}, \nabla f(x_{k}) \rangle + \frac{1}{2\gamma_{k}\eta_{k}} \|y - y_{k}\|^{2}_{W_{n}^{-1}}, 
\end{eqnarray}
where $\|x\|^{2}_{W_{n}} = x^{T}W_{n}x$. As far as we know, the OptISTA algorithm \cite{JangGR23} was only analyzed theoretically, and the efficiency of the algorithm was not verified through any numerical experiments. This paper should be the first attempt, and it is combined with the weighting matrix technique to solve practical problems, such as the non-blind image deblurring task \eqref{target-function}.

Based on this modification, we propose the following improved OptISTA (IOptISTA) algorithm. See Figure \ref{details_Alg} for an illustration of the differences and connections between the OptISTA and IOptISTA algorithms.

\begin{algorithm}[H]
\caption{\textbf{IOptISTA:} An improved OptISTA.}
\label{IOptISTA}
{\bfseries Input:} $b,A$, $\{\eta_{k}\}_{k\ge 0}\in(0,\lambda_{\max}^{-1}(A^{T}A)]$, $\{\alpha_{k}\}_{k\ge0}>0$ and $\{\gamma_{k}\}_{k\ge0}>0$. \\
{\bfseries Initialization:} $W_{n}\succ0$,  $x_{0}=y_{0}$.  
\begin{algorithmic}[1] 
\While {stopping criteria is not met} 
\State $y_{k+1}=\mathrm{Prox}_{\gamma_{k}\eta_{k} h}(y_{k}-\gamma_{k}\eta_{k} W_{n}\nabla f(x_{k}))$.
\State $z_{k+1}=x_{k}+\frac{1}{\gamma_{k}}(y_{k+1}-y_{k})$.
\State $x_{k+1}=z_{k+1}+\frac{\alpha_{k}-1}{\alpha_{k+1}}(z_{k+1}-z_{k})+\frac{\alpha_{k}}{\alpha_{k+1}}(z_{k+1}-x_{k})$.
\State $k=k+1$.
\EndWhile
\end{algorithmic} 
{\bfseries Output:}  $x_{k}$ 
\end{algorithm}
In addition, we also given the setting of $\alpha_{k}$ and $\gamma_{k}$ from \cite{JangGR23} in Algorithms \ref{IOptISTA} as following, 
\begin{eqnarray}
\alpha_{k}=\begin{cases}
	1, &\text{ if } k=0,\\
	\frac{1+\sqrt{1+4\alpha_{k-1}^{2}}}{2}, &\text{ if } 1\le k \le K-1,\\
	\frac{1+\sqrt{1+8\alpha_{K-1}^{2}}}{2}, &\text{ if } k=K,
\end{cases}
\end{eqnarray}
where $K$ is the maximum number of iterations in Algorithm \ref{IOptISTA}, and 
\begin{eqnarray}
\gamma_{k}=\frac{2\alpha_{k}}{\alpha_{K}^{2}}(\alpha_{K}^{2} - 2\alpha_{k}^{2}+\alpha_{k}), \quad k=0,1,\dots,K-1.
\end{eqnarray}
We know that $\gamma_{k}$ in Algorithm \ref{IOptISTA} depends on $\alpha_{K}$, thus the total iteration count $K$ must be chosen before the start of the method. In practice, we select the maximum number of iterations $K$, compute the sequences $\{\alpha_{k}\}$ and $\{\gamma_{k}\}$, and then iterate the algorithm. These two steps are not nested but sequential. Generally speaking, the calculation time will hardly increase because other algorithms, such as FISTA and MFISTA, also need time to calculate related parameters. There are many ways to select the step size $\eta_{k}$, such as constant step size, BB step size, adaptive step size, etc. However, this paper studies the non-blind deblurring optimization problem, that is, the operator $A$ in \eqref{target-function} is known. We can directly compute the parameter $L$ of the $L$-smooth function $f$. We set $\eta_{k}=1/L$ in the numerical part (Section \ref{numerical_part}) for simplicity. One advantage of the IOptISTA algorithm is that it is more convenient to use in practice.

Figure \ref{details_Alg} demonstrates that the IOptISTA framework (Algorithm \ref{IOptISTA}) offers greater flexibility and possibly achieves better acceleration than the OptISTA algorithm \cite{JangGR23}. 

\begin{figure}[!ht]
\setlength\tabcolsep{2pt}
\centering
\begin{tabular}{c} 
\includegraphics[width=0.8\textwidth, height=8cm]{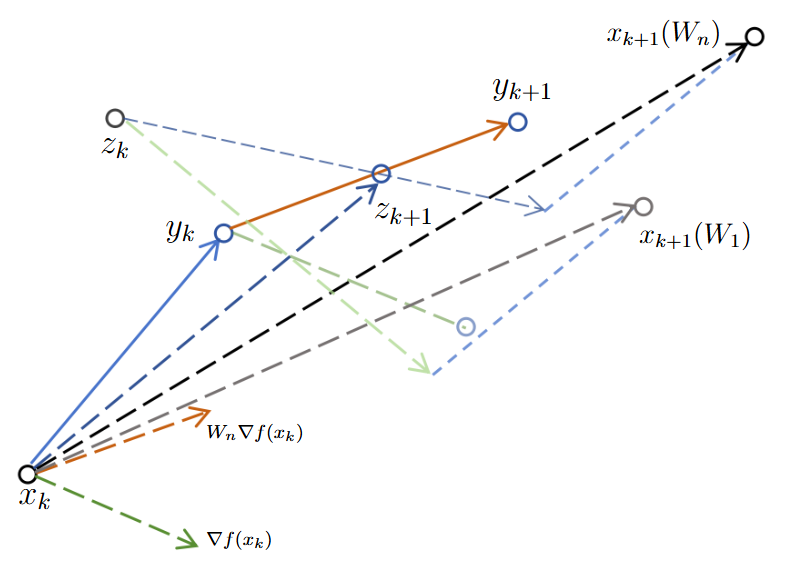} 
\end{tabular}
\caption{The illustration of the IOptISTA and OptISTA algorithms. Where ``$x_{k+1} (W_{1})$'' is the output of OptISTA, ``$x_{k+1}(W_{n})$'' is the output of IOptISTA.  }
\label{details_Alg}
\end{figure}

\begin{remark}
From the IOptISTA framework in Algorithm \ref{IOptISTA}, we have the following observations.
\begin{itemize}
    \item When $W_{n} = W_{1} = I$ (identity matrix), the IOptISTA algorithm simplifies to the OptISTA algorithm \cite{JangGR23}.
    \item When $W_{n} = I$ and $h = 0$ in Algorithm \ref{IOptISTA}, it reduces to the optimized gradient method (OGM) \cite{KimF16}.
\end{itemize}
\end{remark}
The convergence analysis of the IOptISTA algorithm (Algorithm \ref{IOptISTA}) is beyond the scope of this paper, and we leave it as one of the future research topics. It is worth noting that through our experiments, we have found that the proposed algorithm is consistently convergent and can almost always achieve superior numerical results.
\begin{remark}
To ensure the convergence of the algorithm, a viable approach is to incorporate the monotonicity of the objective function ($\phi(x) := f(x) + h(x)$) into the IOptISTA algorithm. For instance, for the $4$-th step in the IOptISTA algorithm (Algorithm \ref{IOptISTA}), we propose the following modification:
\begin{eqnarray}
\begin{aligned}
&\hat{x}_{k+1}=z_{k+1}+\frac{\alpha_{k}-1}{\alpha_{k+1}}(z_{k+1}-z_{k})+\frac{\alpha_{k}}{\alpha_{k+1}}(z_{k+1}-x_{k}),\\
&\text{If } \phi(\hat{x}_{k+1})< \phi(x_{k}) \text{ then } x_{k+1}=\hat{x}_{k+1} \text{ else } x_{k+1}=x_{k}.
\end{aligned}
\end{eqnarray}
We denote this modification based on Algorithm \ref{IOptISTA} as the MOptISTA algorithm. In this way, we can always ensure the algorithm converges, meaning that the objective function $\phi(x_k)$ forms a descending sequence. However, our experiments have revealed that the iteration time for the modified algorithm increases, and the numerical results are slightly inferior to those obtained with Algorithm \ref{IOptISTA}.  See Figure \ref{algorithms_two} for details.  We can see from Figure \ref{algorithms_two} that compared with the MOptISTA algorithm, Algorithm \ref{IOptISTA} can obtain better results in terms of error, PNSR, and SSIM values.
\begin{figure}[!ht]
\setlength\tabcolsep{2pt}
\centering
\begin{tabular}{ccc} 
\includegraphics[width=0.32\textwidth, height=3.5cm]{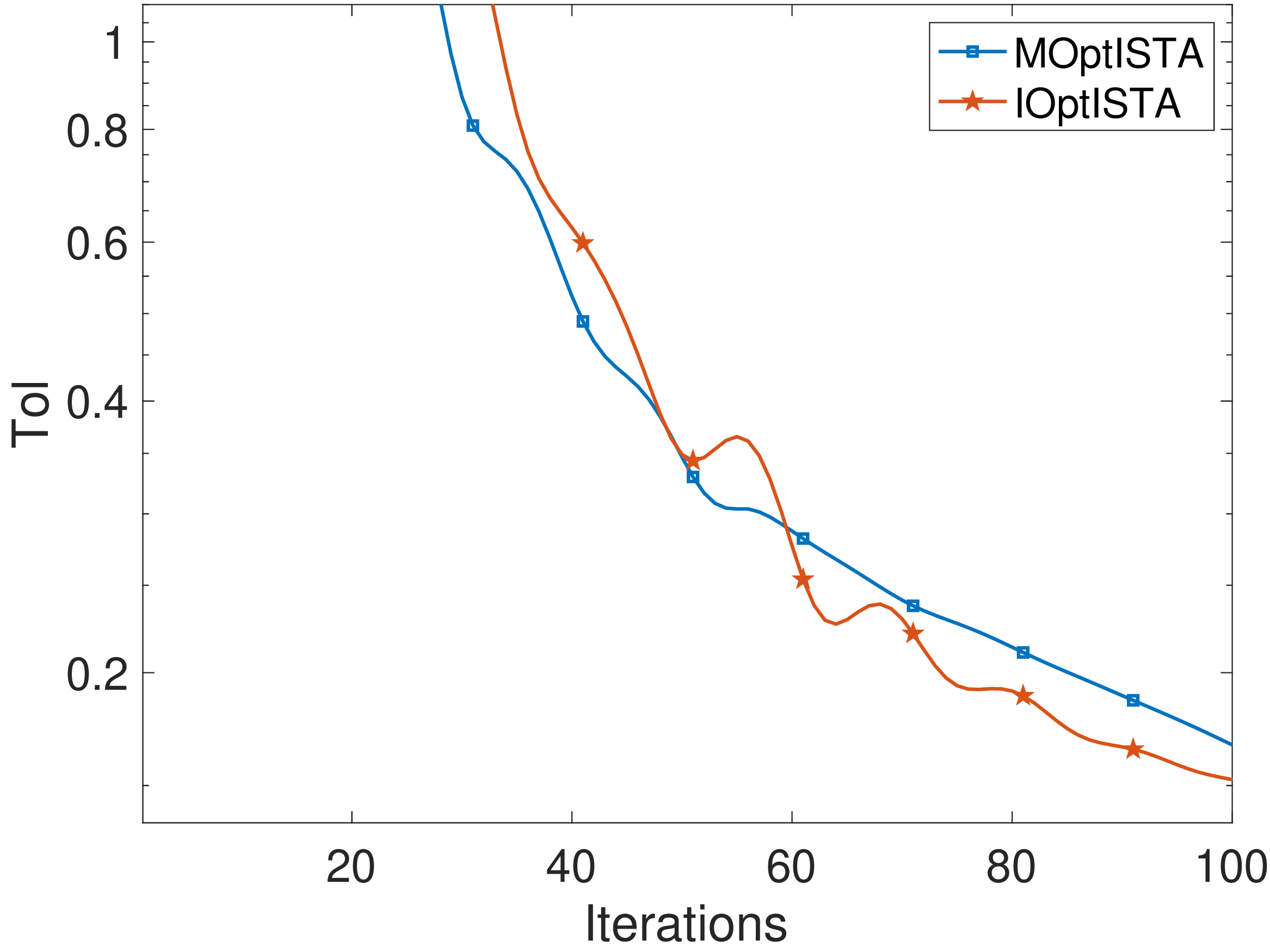} & \includegraphics[width=0.32\textwidth, height=3.5cm]{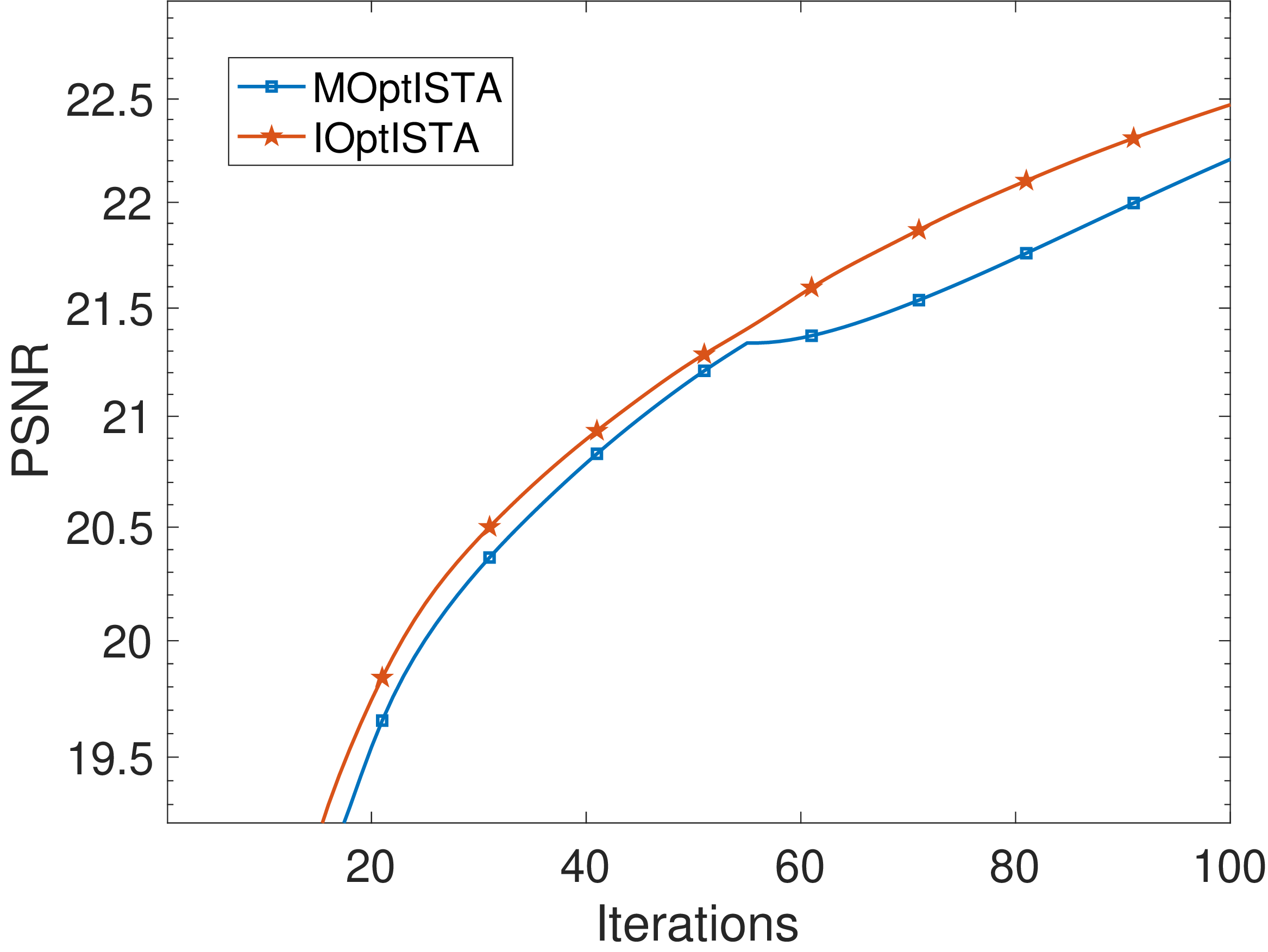}& \includegraphics[width=0.32\textwidth, height=3.5cm]{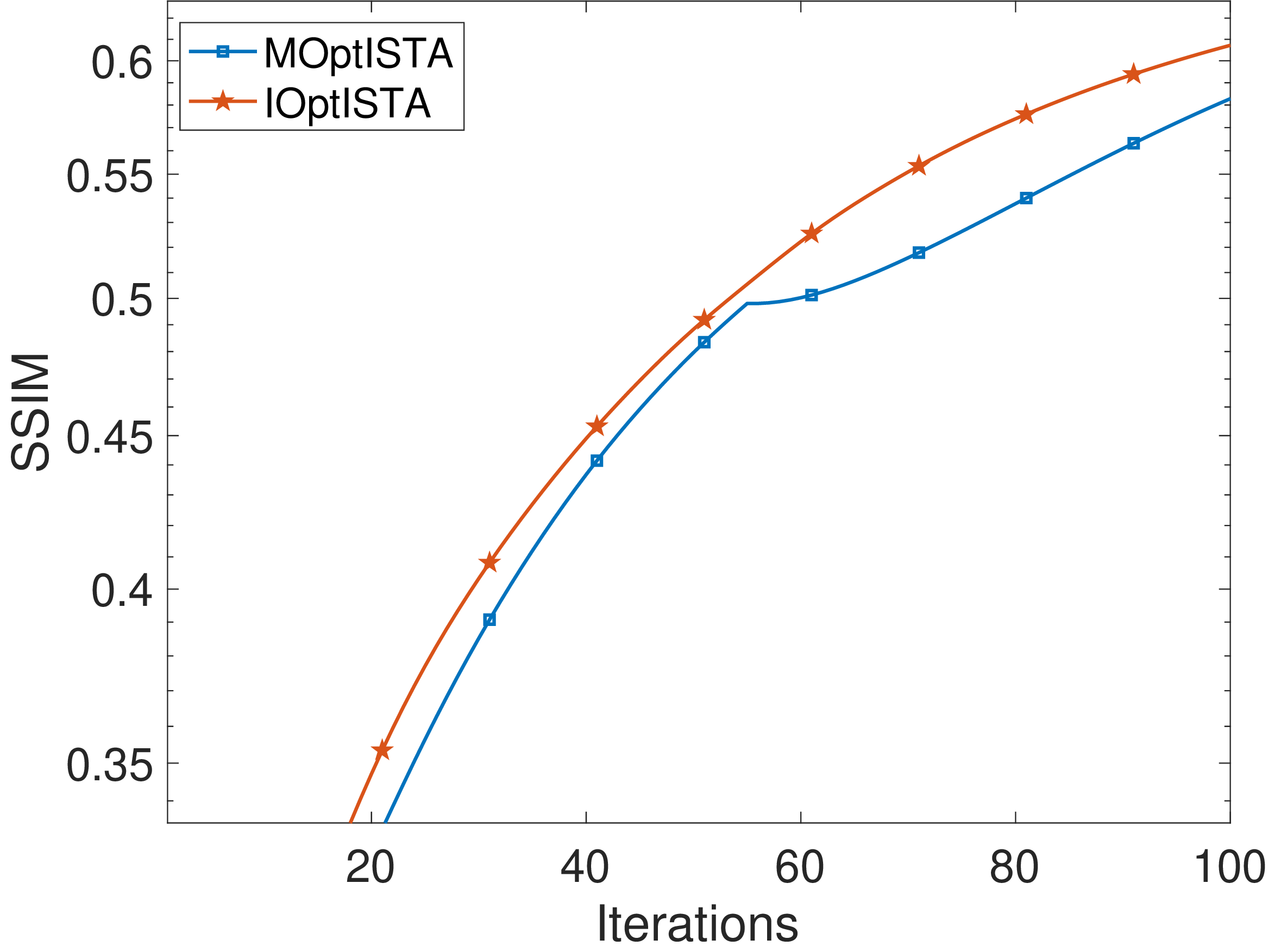}\\
(a) Tol & (b) PNSR & (c) SSIM
\end{tabular}
\caption{The numerical results for Figure \ref{details_images}(a)  with ``$k = \text{fspecial}(\text{`disk'}, 13)$'' and $\varepsilon\sim\mathcal{N}(0, 10^{-3})$.}
\label{algorithms_two}
\end{figure}
Consequently, for the sake of simplicity, we opt to utilize Algorithm \ref{IOptISTA} in this paper.
\end{remark}

In addition, we establish the following property of the IOptISTA algorithm.
\begin{lemma}\label{lemma1}
In Algorithm \ref{IOptISTA}, we have $x_K = y_K$.
\end{lemma}
\begin{proof}
The proof of this lemma is detailed in \ref{pf_lemma1}.
\end{proof}
Based on Lemma \ref{lemma1}, we define the output of Algorithm \ref{IOptISTA} as $x_k$, rather than $y_k$, which is the convention in other acceleration algorithm frameworks, such as FISTA \cite{BeckT09}, IFISTA \cite{BhottoAS15}, and EFISTA \cite{KumarS24}.

\section{Numerical experiments} \label{numerical_part}
In this section, we conduct a numerical study to evaluate the performance of the proposed algorithm under various regularization settings for non-blind image deblurring tasks. Furthermore, we compare the performance of our algorithm (Algorithm \ref{IOptISTA}) with several well-established methods, including ISTA \cite{DaubechiesDM04}, IISTA (Improved ISTA), FISTA \cite{BeckT09}, EFITSA \cite{KumarS24}, and OptISTA \cite{JangGR23}. We do not compare with the IFISTA algorithm \cite{BhottoAS15} because it is a special case of the EFISTA algorithm \cite{KumarS24}. The numerical experiments are implemented using MATLAB R2019a and executed on a computer equipped with an Intel CORE i7-14700KF @ 3.40GHz processor and 64GB of RAM. 

The algorithm terminates under one of the following three conditions:
\begin{itemize}
    \item[(1)] The maximum run time ($T$) is reached.
    \item[(2)] The maximum number of iterations ($K$) is reached.
    \item[(3)] All compared algorithms terminate when the tolerance level is satisfied:
    \begin{eqnarray}
        \text{Tol} := \frac{1}{2} \|Ax - b\|^2 \leq 10^{-4}.
    \end{eqnarray}
\end{itemize}

To quantitatively assess the quality of the restored images, we utilize the peak signal-to-noise ratio (PSNR)\footnote{$\text{PSNR} = 20 \log_{10} \left( \frac{255^2}{\text{MSE}} \right)$, where $\text{MSE} = \|\hat{x} - x\|^2$, $\hat{x}$ and $x$ are the restored and original images, respectively.} and the structural similarity index (SSIM)\footnote{$\text{SSIM} = \frac{(2u_{\hat{x}}u_{x} + c_1)(2\sigma_{\hat{x}x} + c_2)}{(u_{\hat{x}}^2 + u_x^2 + c_1)(\sigma_{\hat{x}}^2 + \sigma_x^2 + c_2)}$, where all parameters are defined as in \cite{WangBSS04}.} \cite{WangBSS04}.

We utilize six images to demonstrate the numerical performance of the proposed algorithms, as depicted in Figure \ref{details_images}. 
\begin{figure}[!ht]
	\setlength\tabcolsep{2pt}
	\centering
	\begin{tabular}{cccccc} 
		\includegraphics[width=0.15\textwidth]{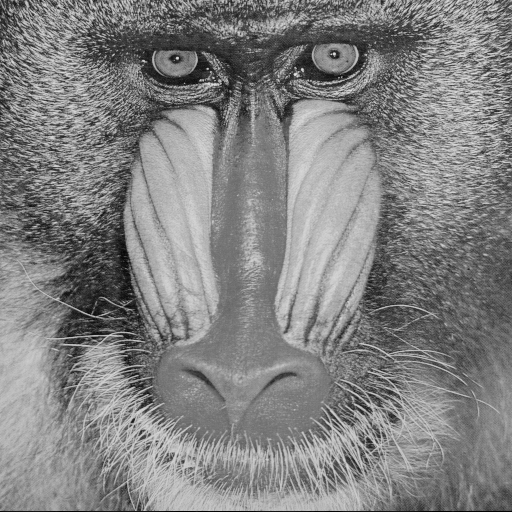} & \includegraphics[width=0.15\textwidth]{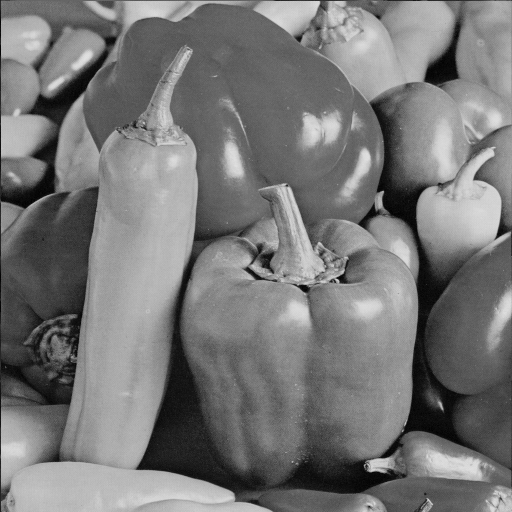}&
		\includegraphics[width=0.15\textwidth]{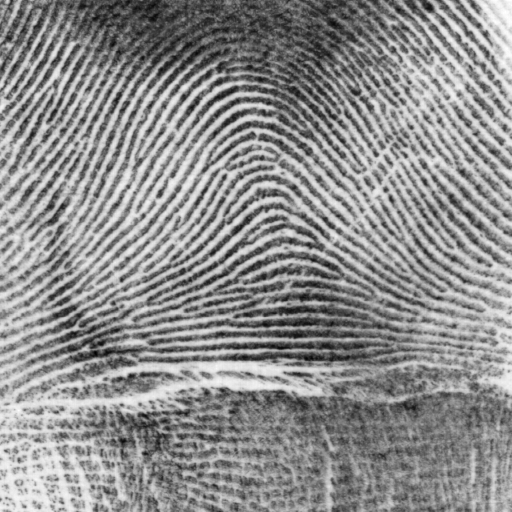}&
		\includegraphics[width=0.15\textwidth]{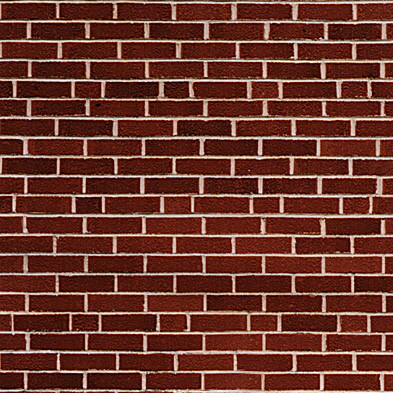}& \includegraphics[width=0.15\textwidth]{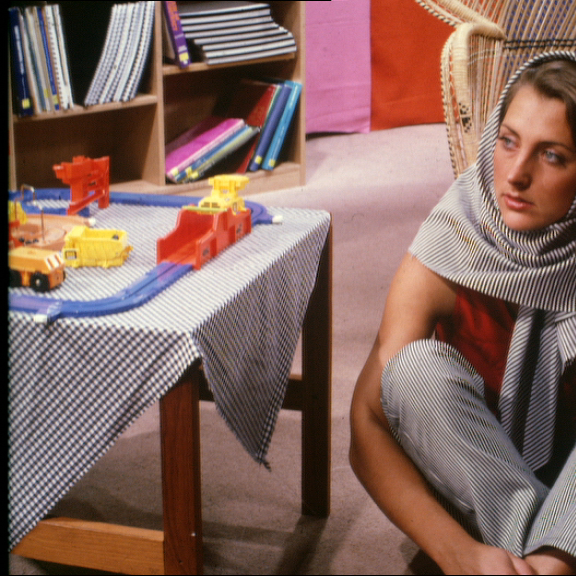}&
		\includegraphics[width=0.15\textwidth]{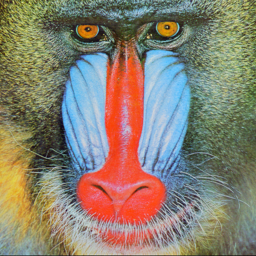}\\
		(a) Img01 & (b) Img02 & (c) Img03&(d) Img04 & (e) Img05 & (f) Img06
	\end{tabular}
	\caption{Details of the test images.} 
	\label{details_images}
\end{figure}

In addition, we use the commands ``disk'' and ``gaussian'' in MATLAB to blur the image, i.e., ``$s = \text{fspecial}(\text{`disk'}, a)$'' or ``$s = \text{fspecial}(\text{`gaussian'}, b, c)$'' with positive constants $a, b, c > 0$, and then processed with ``$\text{imfilter}(\text{im}, s, \text{`circular'})$'' to produce blurred images.

In this section, we consider two regularization cases for $h(x)$: the $l_1$ norm \cite{Tibshirani96} and the total variation (TV) norm \cite{RudinOF92}.
\subsection{Setting of \texorpdfstring{$n$}{n} in \texorpdfstring{$W_{n}$}{Wn}}

In this subsection, we utilize Figure \ref{details_images}(a) and (d) to assess the numerical performance of $W_n$ (as defined in \eqref{weighting_matrix}) under various values of $n$ within Algorithm \ref{IOptISTA}. For the sake of simplicity, we focus on the scenario where $h = \lambda \|x\|_1$ with $\lambda = 0.0001$. Across all compared algorithms, the parameters are configured as $T = 20$ seconds, $K = 300$ iterations, and the initial iteration point is set to zero. 

Subsequently, for this deblurring model, we employ Gaussian noise and disk blur algorithms for validation. Specifically, for Gaussian noise, we examine two cases: $\varepsilon \sim \mathcal{N}(0, 10^{-4})$ and $\varepsilon \sim \mathcal{N}(0, 5 \times 10^{-4})$. Regarding the blur kernel, we consider different levels of blurring. For the selection of $n$ in $W_n$, we explore 7 scenarios, namely $n = \{2, 4, 6, 8, 10, 12, 14\}$. 

To demonstrate the enhancement achieved by incorporating $W_n$ in IOptISTA (Algorithm \ref{IOptISTA}), we also contrast it with the OptISTA algorithm, which corresponds to the case of $W_n = W_1 = I$ in Algorithm \ref{IOptISTA}. Refer to Table \ref{real_Wn_table} for the outcomes. Additionally, we provide a detailed view of the last two columns in Table \ref{real_Wn_table}, as shown in Figure \ref{tps_results_Wn_img01}.

\begin{table*}[thp]
\begin{center}
	\fontsize{6}{11}\selectfont
	\setlength{\tabcolsep}{1.2mm}
	\caption{The numerical results for the optimization problem \eqref{target-function} with $h(x)=\lambda\|x\|_{1}$ for disk blurring kernel, and  different $n$ in $W_{n}$ and noises. 
	} 
\label{real_Wn_table}
\begin{tabular}{c|c|c c | cc | cc| cc cc } 
	\hline 
	 & Images & \multicolumn{4}{c|}{Img01}  & \multicolumn{4}{c}{Img04}  \\\cline{2-10}
& Noises & \multicolumn{2}{c|}{$\varepsilon\sim\mathcal{N}(0,10^{-4})$} & \multicolumn{2}{c|}{$\varepsilon\sim\mathcal{N}(0,5*10^{-4})$} & \multicolumn{2}{c|}{$\varepsilon\sim\mathcal{N}(0,10^{-4})$} &  \multicolumn{2}{c}{$\varepsilon\sim\mathcal{N}(0,5*10^{-4})$}\\\cline{2-10}
	Algorithm & Kernels& 12 & 14.5 & 12 & 13.5 & 7 & 12  & 7 & 11  \\\hline
\multirow{3}{*}{OptISTA} & Tol & 6.3e-2 & 5.5e-2 &  8.8e-2 & 8.4e-2 & 1.3e-1 & 1.6e-1 & 1.4 & 1.7e-1\\
& PSNR &22.60 &  22.28&  22.57 & 22.38 & 22.11 & 18.76 & 22.09 & 18.75\\
& SSIM &0.6227 & 0.5901&  0.6190 & 0.6018 & 0.6679 & 0.5738 & 0.6670 & 0.5726\\ \hline
\multirow{3}{*}{IOptISTA, $n=2$} & Tol & 3.2e-2 & 2.9e-2&  5.5e-2 & 5.3e-2 & 4.4e-2 & 4.6e-2& 5.4e-2 & 5.9e-2\\
& PSNR & 23.44 & 23.11 &  23.37 & 23.17 & 24.14 & 20.65& 24.05 & 20.60\\
& SSIM & 0.6938 & 0.6660&  0.6829 & 0.6684 & 0.7378 & 0.6419& 0.7342 & 0.6388\\ \hline
\multirow{3}{*}{IOptISTA, $n=4$} & Tol & 1.6e-2 & 1.5e-2&  3.6e-2 & 3.6e-2 & 1.7e-2 & 1.5e-2 & 2.5e-2 & 2.6e-2\\
& PSNR & 24.46 & 24.14 &  24.20 & 23.94 & 26.14 & 22.18 & 25.79 & 22.01\\
& SSIM & 0.7593 & 0.7369 &  0.7305 & 0.7164 & 0.7995 & 0.6935& 0.7884 & 0.6850\\ \hline
\multirow{3}{*}{IOptISTA, $n=6$} & Tol & 1.0e-2 & 9.5e-3 &  2.8e-2 & 2.8e-2 & 8.2e-3 & 8.8e-3 & 1.5e-2 & 1.9e-2\\
& PSNR & 25.13 & 24.83 &  24.62 & \textbf{24.21} & 27.20 & 22.89 & 26.51 & 22.56\\
& SSIM & 0.7927 & 0.7743 &  0.7448 & \textbf{0.7291} & 0.8296 & 0.7168& 0.8096 & 0.7016\\ \hline
\multirow{3}{*}{IOptISTA, $n=8$} & Tol & 7.4e-3 & 7.5e-3 &  2.4e-2 & 2.5e-2 & 6.5e-3 & 6.1e-3& 1.3e-2 & 1.5e-2\\
& PSNR & 25.62 & 25.34 &  24.83 & \textbf{24.21}& 27.79 & 23.33& 26.74 & 22.81\\
& SSIM & 0.8139 & 0.7973 &  0.7475 & 0.7286 & 0.8447 & 0.7311& 0.8163 & 0.7086\\ \hline
\multirow{3}{*}{IOptISTA, $n=10$} & Tol & 5.9e-3 & 6.7e-3 &  2.1e-2 & 2.3e-2& 5.6e-3 & 5.1e-3& 1.1e-2 & 1.3e-2\\
& PSNR &26.03 & 25.72&  \textbf{24.91} & 24.06 & 28.24 & 23.64& 26.83 & 22.92\\
& SSIM &0.8284 & 0.8130 &  \textbf{0.7442} & 0.7224 & 0.8548 &0.7438 & 0.8186 & 0.7132\\ \hline
\multirow{3}{*}{IOptISTA, $n=12$} & Tol & 4.8e-3 & \textbf{7.0e-3}&  1.9e-2 & 2.1e-2& 4.7e-3 & \textbf{4.1e-3}& 1.0e-2 & \textbf{1.2e-2}\\
& PSNR &26.36 & \textbf{26.00}&  24.90 & 23.83 & 28.65 & \textbf{23.89}& 26.88 & \textbf{22.96}\\
& SSIM &0.8392 & \textbf{0.8284}&  0.7383 & 0.7136 & 0.8632 & \textbf{0.7539}& 0.8196 & \textbf{0.7154}\\ \hline
\multirow{3}{*}{IOptISTA, $n=14$} & Tol & \textbf{4.0e-3} & 1.6e+31&  \textbf{1.6e-2} & \textbf{2.0e-2} & \textbf{3.7e-3} & 2.5e+19& \textbf{8.6e-3} & 2.5e+19\\
& PSNR &\textbf{26.64} & $<0$&  24.84 & 23.58& \textbf{29.05} & $<0$& \textbf{26.91} & $<0$\\
& SSIM &\textbf{0.8475} & $<0$&  0.7311 & 0.7039& \textbf{0.8719} & $<0$& \textbf{0.8208} & $<0$\\ \hline
\end{tabular}
\end{center}
\end{table*}

\begin{figure}[!ht]
\setlength\tabcolsep{2pt}
\centering
\begin{tabular}{ccc} 
\includegraphics[width=0.32\textwidth, height=3.5cm]{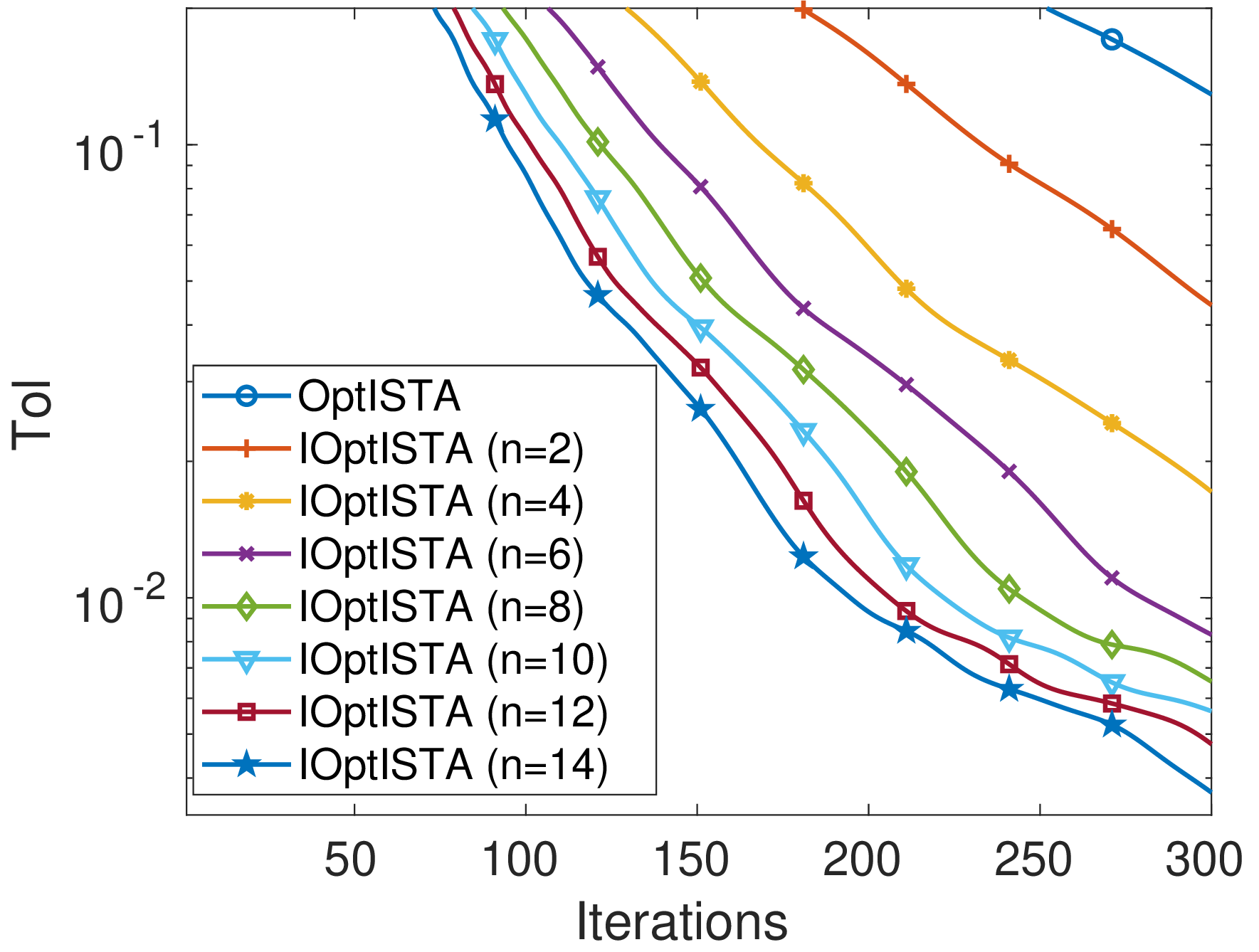} & \includegraphics[width=0.32\textwidth, height=3.5cm]{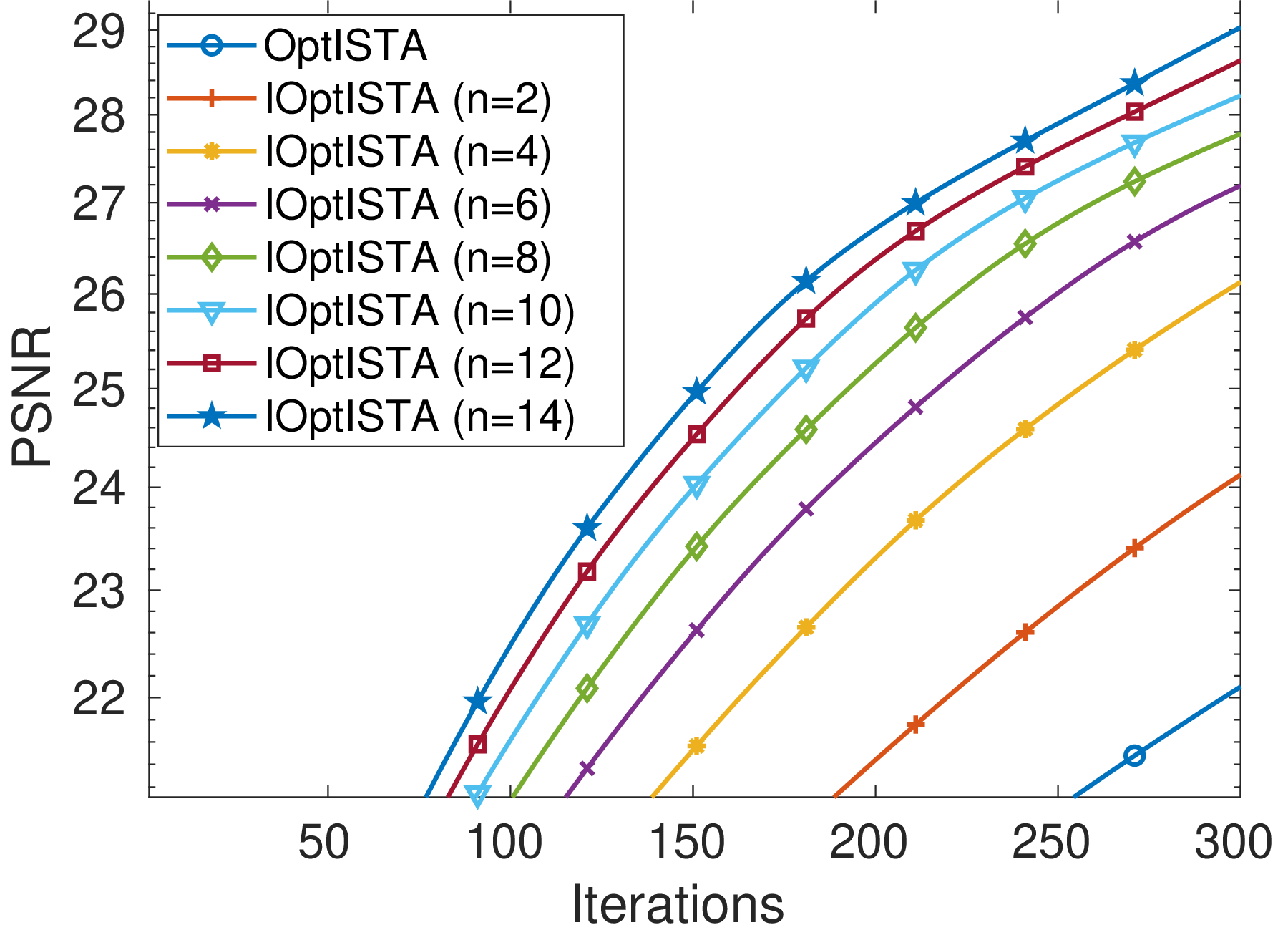}& \includegraphics[width=0.32\textwidth, height=3.5cm]{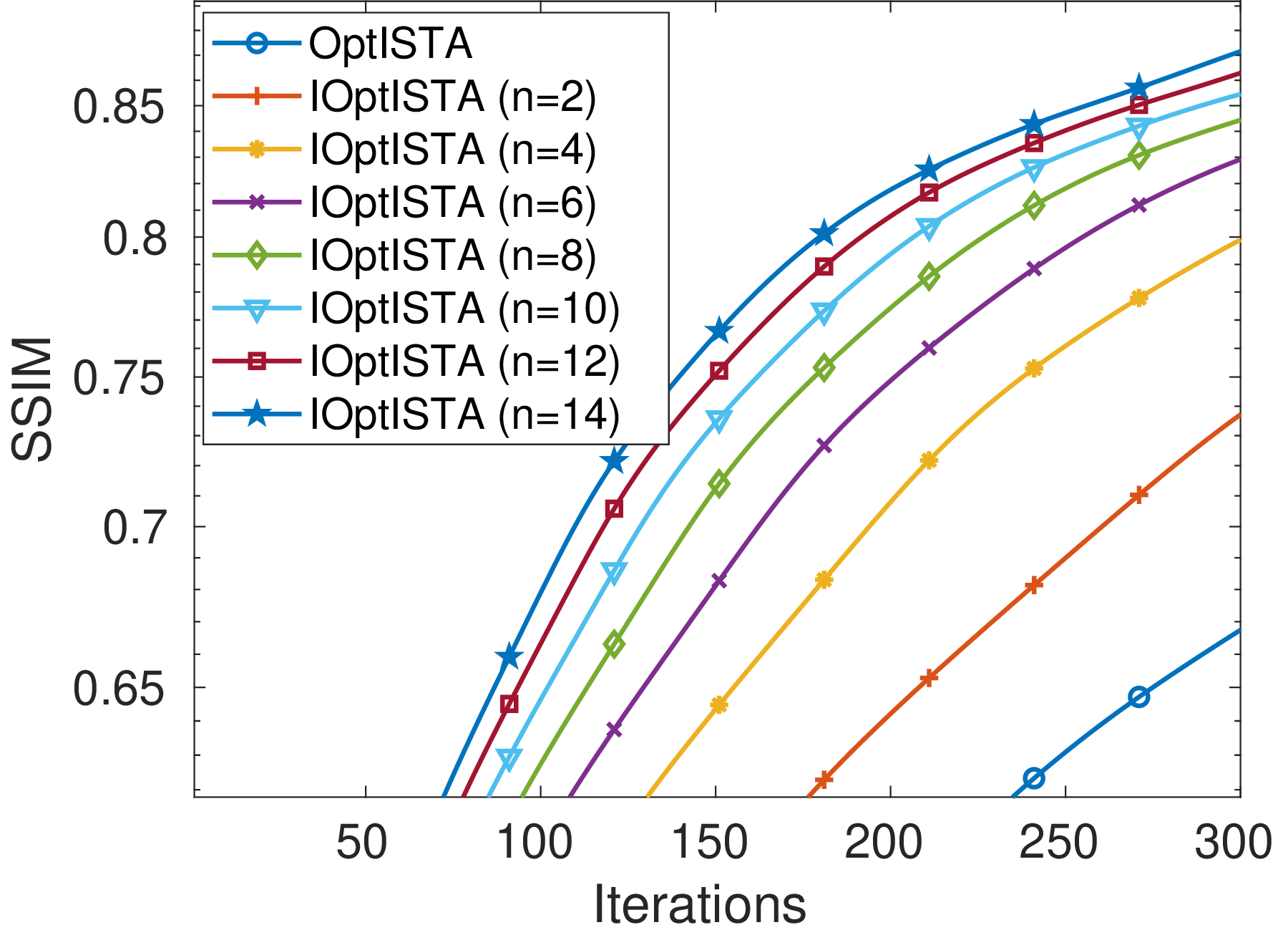}\\
(a) Tol & (b) PNSR & (c) SSIM\\
\includegraphics[width=0.32\textwidth, height=3.5cm]{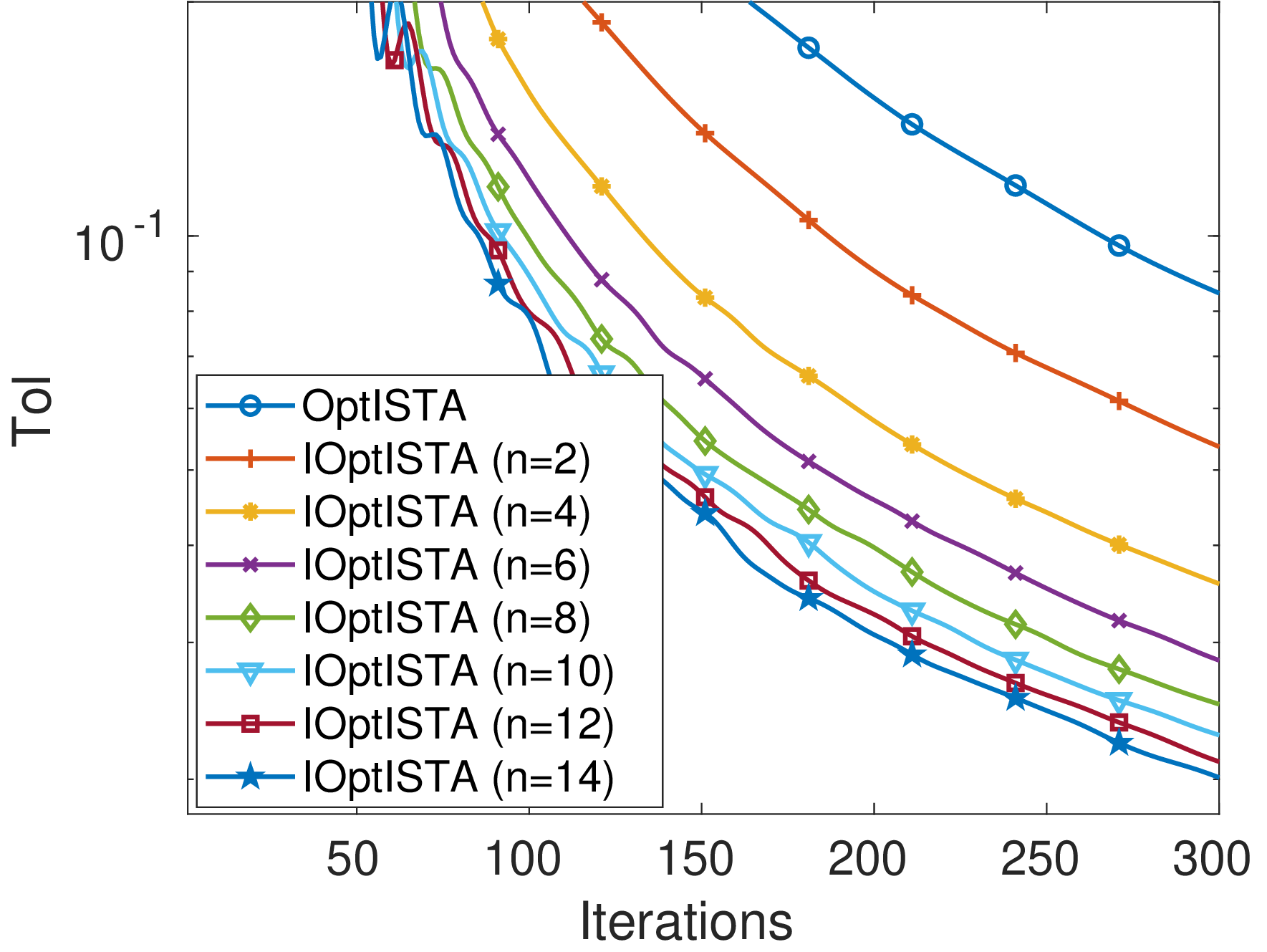} & \includegraphics[width=0.32\textwidth, height=3.5cm]{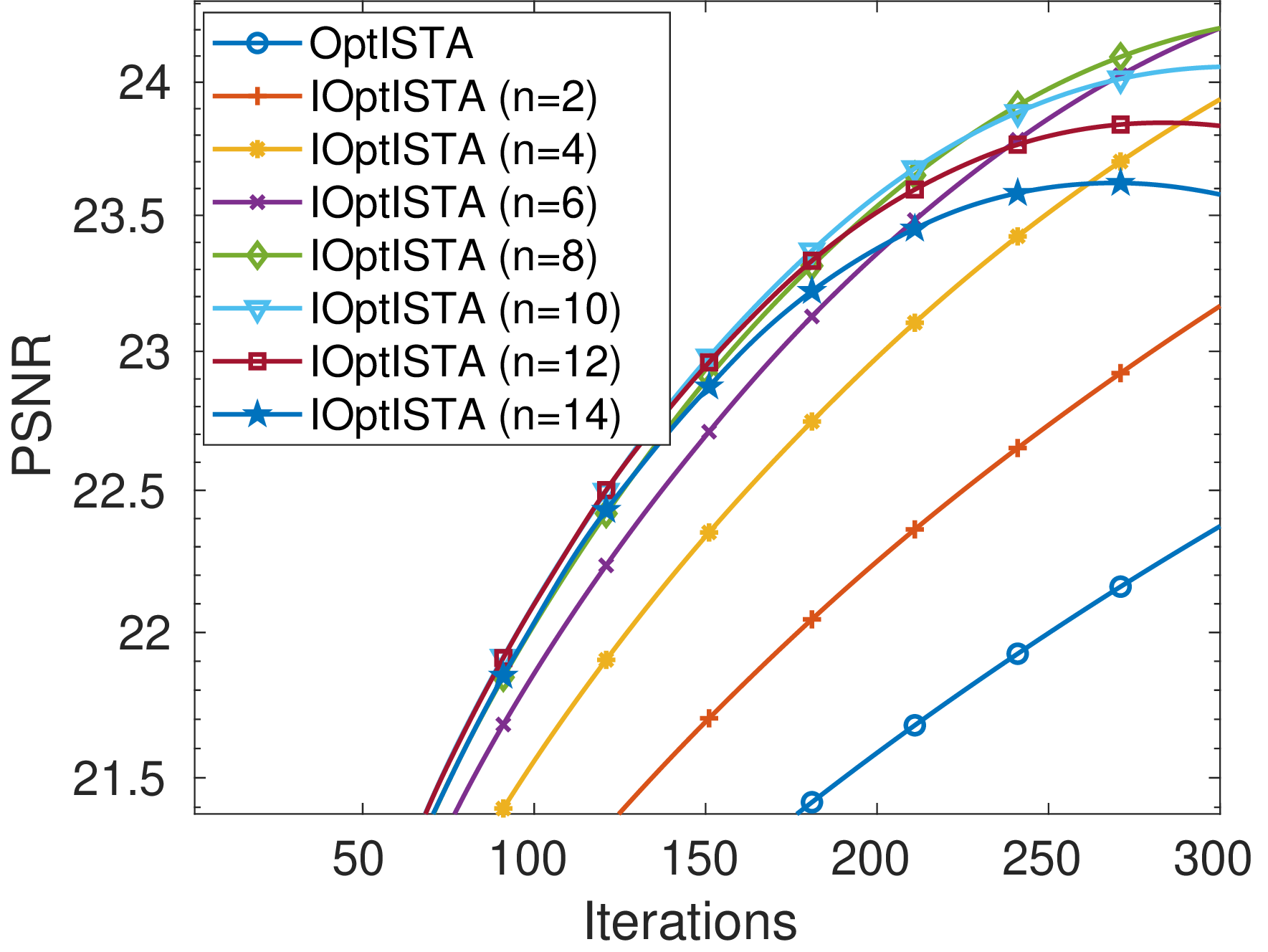}& \includegraphics[width=0.32\textwidth, height=3.5cm]{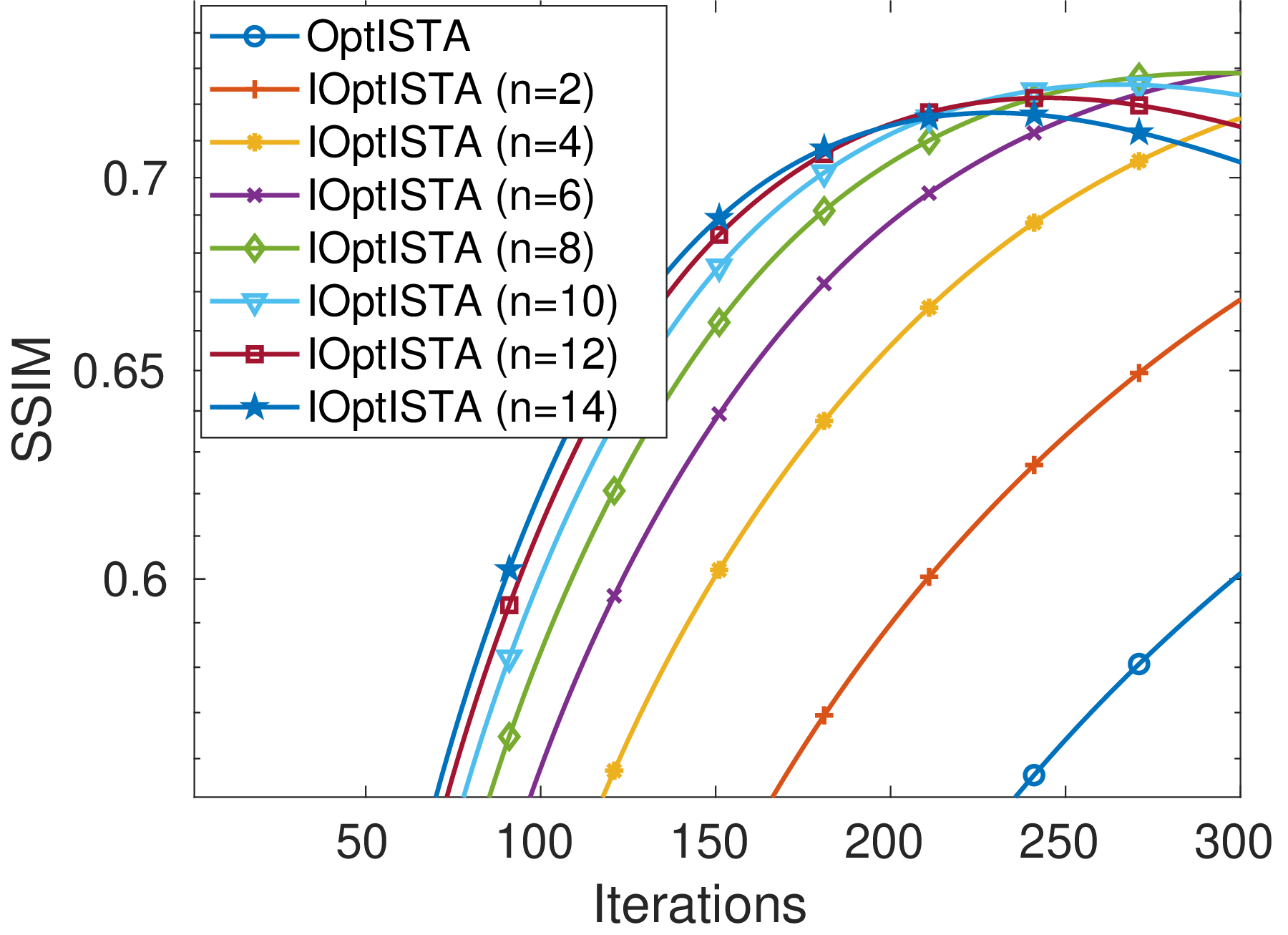}\\
(d) Tol & (e) PNSR & (f) SSIM\\
\includegraphics[width=0.32\textwidth, height=3.5cm]{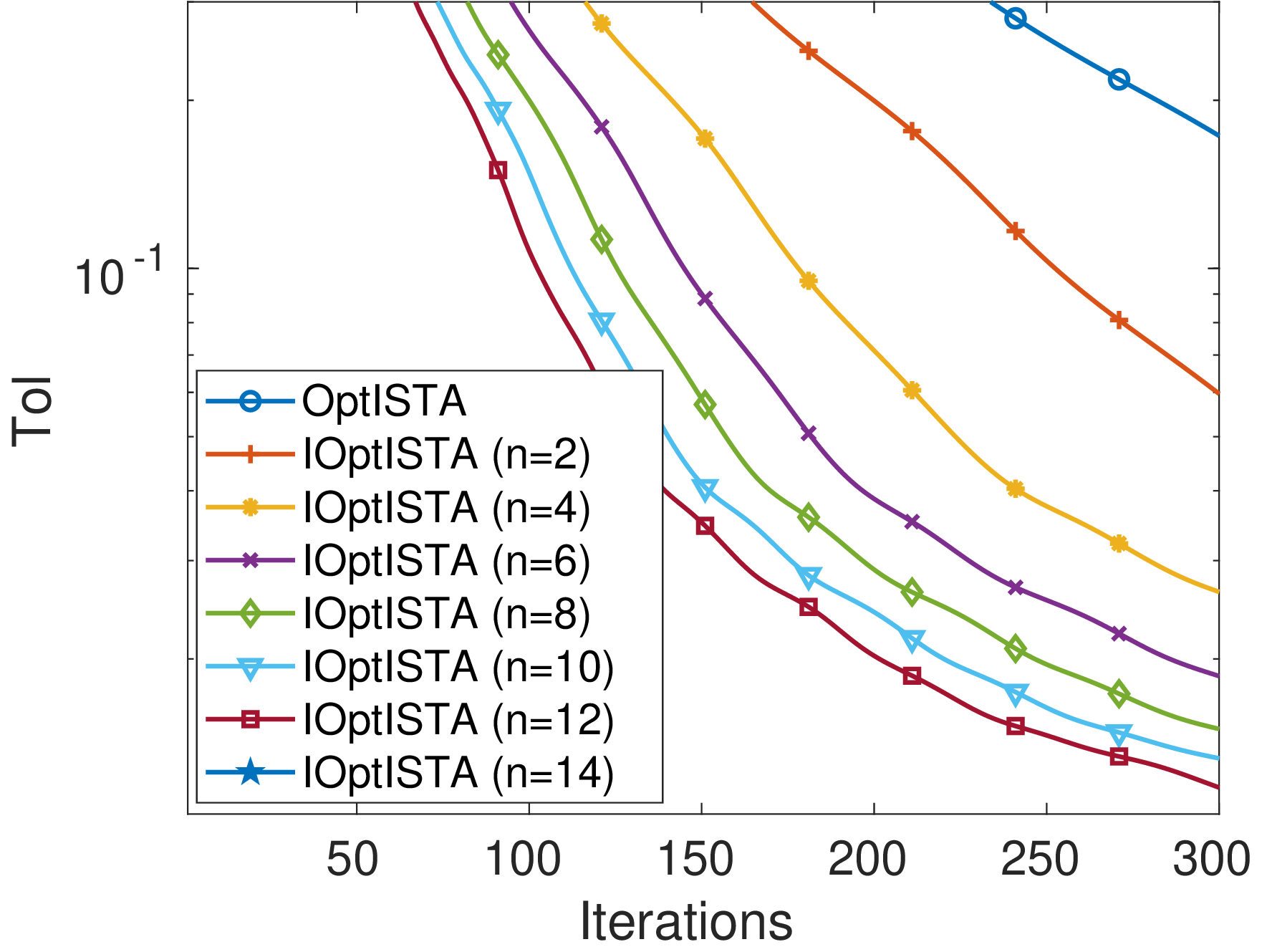} & \includegraphics[width=0.32\textwidth, height=3.5cm]{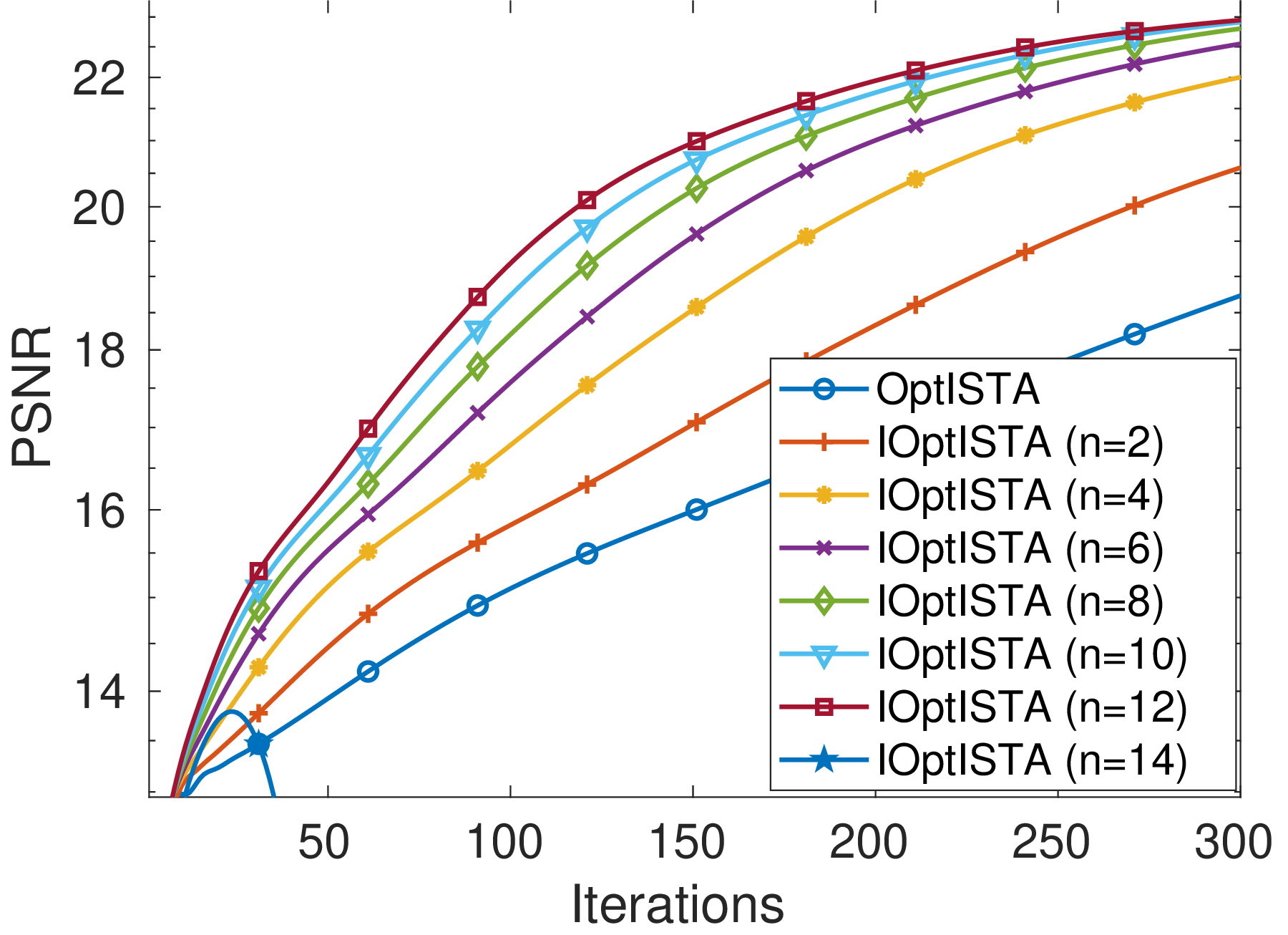}& \includegraphics[width=0.32\textwidth, height=3.5cm]{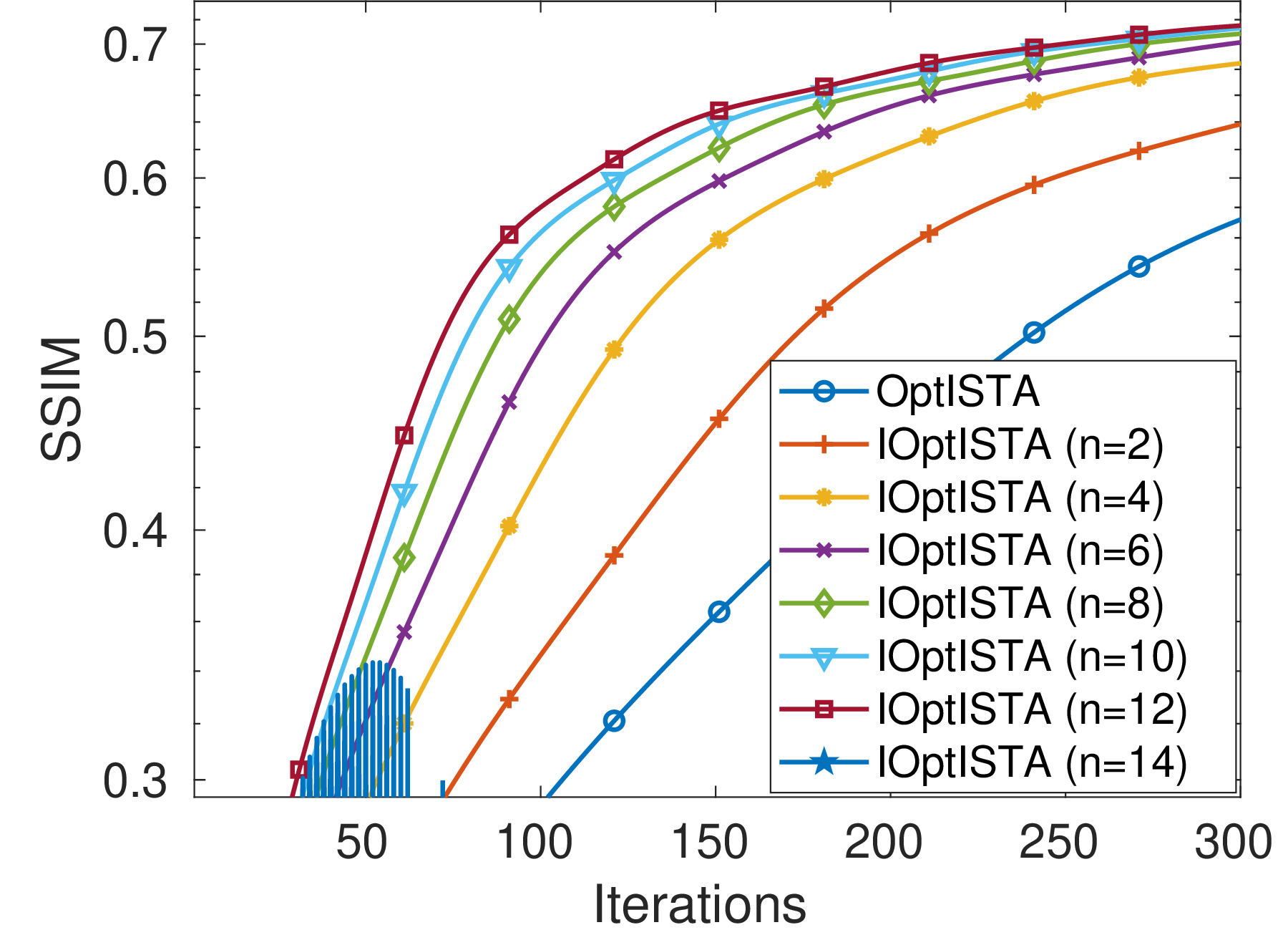}\\
(g) Tol & (h) PNSR & (i) SSIM
\end{tabular}
\caption{The Tol, PSNR and SSIM results for different images, noises and  $n$ in $W_{n}$. First row: Figure \ref{details_images}(d) with ``$k = \text{fspecial}(\text{`disk'}, 7)$'' and $\varepsilon\sim\mathcal{N}(0,1e-4)$. Second row:  Figure \ref{details_images}(a) with ``$k = \text{fspecial}(\text{`disk'}, 13.5)$'' and $\varepsilon\sim\mathcal{N}(0,5e-4)$. Third row: Figure \ref{details_images}(d) with ``$k = \text{fspecial}(\text{`disk'}, 11)$'' and $\varepsilon\sim\mathcal{N}(0,5e-4)$.
}
\label{tps_results_Wn_img01}
\end{figure}

By comparing Table \ref{real_Wn_table} and Figure \ref{tps_results_Wn_img01}, we can discern the following observations:
\begin{itemize}
    \item When the noise variance and blur degree are relatively low, a larger $n$ in $W_n$ yields lower error and superior PSNR and SSIM values for the proposed algorithm.
    \item Under a constant noise variance, as the blur level increases, increasing $n$ within a certain range leads to improved numerical results. However, beyond a certain threshold, an excessive value of $n$ in $W_n$ can degrade the results or even cause the algorithm to diverge and fail.
    \item If the blur degree is held constant, a similar numerical effect is observed as the noise variance increases.
\end{itemize}

Through numerical experiments, we have determined that a larger $n$ in $W_n$ of Algorithm \ref{IOptISTA} does not automatically result in better numerical performance. This relationship is dependent on both the noise variance and the degree of blur.

We observed that the impact of increasing $n$ on the algorithm's numerical performance is not straightforward; it is influenced by the variance of the noise, the degree of blur and the step size of the proposed algorithm. When $n=14$, one of the main reasons for the algorithm divergence should be that the step size is too large. In summary, under fixed step size, for scenarios with low noise variance and mild blur, selecting a relatively large $n$ can further enhance the algorithm's numerical performance. Conversely, for cases with higher noise variance or more severe blur, a smaller $n$ is preferred to ensure algorithm stability, even if it does not yield the best numerical results. Nonetheless, all selections of $W_{n}$ lead to improved performance compared to algorithms without the use of $W_n$. 

For example, algorithms such as IEFISTA \cite{BhottoAS15}, IEFISTA-BN \cite{WangWWGC18}, and EFISTA \cite{KumarS24}, which have convergence theoretical guarantees, will also diverge in numerical experiments when $n$ is large and the step size is not changed, just like the last case in Table \ref{real_Wn_table}.

The numerical disparity between employing $W_n$ and not employing it within an algorithm remains significantly large, thus warranting further exploration into the selection methodology for $n$ as part of our future research endeavours. For simplicity, we fixed $n=12$ in $W_{n}$ for all compared algorithms in the following numerical experiments.

\subsection{Case one: \texorpdfstring{$l_{1}$}{l1} norm regularization} \label{l1-regularization}
In this subsection, we focus on the case where $h(x) = \lambda \|x\|_1$ \cite{Tibshirani96,Lu17c} in the optimization problem \eqref{target-function}, which is formulated as:
\begin{eqnarray}
	\min_{x \in \mathbb{R}^n} \quad \frac{1}{2} \|Ax - b\|^2 + \lambda \|x\|_1, \label{l1-model}
\end{eqnarray}
where $\lambda > 0$ is a constant parameter, and $\|x\|_1 := \sum_{i=1}^{n} |x_i|$. We set $\lambda = 0.0001$ to evaluate the performance of all compared algorithms. For these algorithms, the parameters are configured with $T = 20$ seconds, $K = 300$ iterations, and the initial iteration point is set to zero.

We employ six images within the optimization problem \eqref{l1-model} to assess the numerical performance of the proposed algorithms under various blurred kernels and noise conditions. The results are compiled in Table \ref{real_l1_table}.

\begin{table*}[thp]
\begin{center}
	\fontsize{6}{11}\selectfont
	\setlength{\tabcolsep}{1.2mm}
	\caption{The numerical results for the optimization problem \eqref{l1-model} with different blurred kernels and noises under six compared algorithms. 
	} 
\label{real_l1_table}
\begin{tabular}{c|c|c c | cc | cc| cc| cc| cc cc } 
	\hline 
	 & Images & \multicolumn{2}{c|}{Img01} & \multicolumn{2}{c|}{Img02} & \multicolumn{2}{c|}{Img03} &  \multicolumn{2}{c|}{Img04} &\multicolumn{2}{c|}{Img05} &\multicolumn{2}{c}{Img06}\\\cline{2-14}
	Algorithm & Kernels& 12.5 &(24,40) & 12 & (24,40) & 13 & (20,30)  & 8 & (15,30) & 9  & (25,35) & 8 & (10,23)  \\\hline
    \multicolumn{14}{c}{Noise $\varepsilon\sim \mathcal{N}(0,10^{-4})$} \\\hline
\multirow{3}{*}{ISTA} & Tol & 2.6e0 &2.3e0 & 3.0e0 & 3.0e0  & 7.9e0 & 8.9e0 & 1.4e+1 & 2.2e+1 & 6.8e0 & 6.9e0 & 6.5e-1 & 9.0e-1  \\
	  & PSNR & 19.70 &19.64 & 23.89 & 23.56  & 16.43 & 16.61 & 14.93 & 15.79 & 21.29 & 20.31 &21.57 & 22.40\\ 
      & SSIM & 0.2934 &0.2906 & 0.6669 & 0.6609  & 0.1911 & 0.2221 & 0.2602 & 0.3784 & 0.5269 & 0.4821 & 0.3937 & 0.4644  \\\hline
\multirow{3}{*}{FISTA} & Tol & 1.2e-1 &7.9e-2 & 4.7e-2 & 4.4e-2  & 2.5e-1 & 3.0e-1 & 2.2e-1 & 1.8e-1 & 1.9e-1 & 1.0e-1 & 1.8e-2 & 1.7e-2  \\
	  & PSNR & 21.85 &21.23 & 28.59 & 27.69  & 21.56 & 19.64 & 19.72 & 22.81 & 25.27 & 23.26 & 25.19 & 25.72\\ 
      & SSIM & 0.5458 &0.4768 & 0.7816 & 0.7500  & 0.7466 & 0.6258 & 0.6152 & 0.6885 & 0.7362 & 0.6266 & 0.7200 & 0.7528  \\\hline
\multirow{3}{*}{OptISTA} & Tol & 6.3e-2 &4.2e-2 & 2.0e-2 & 2.1e-2  & 1.0e-1 & 1.3e-1 & 1.0e-1 & 5.1e-2 & 8.9e-2 & 5.3e-2 & 8.2e-3 & 8.5e-3  \\
	  & PSNR & 22.60 &21.65 & 29.66 & 28.62  & 22.91 & 20.63 & 21.15 & 24.66 & 26.55 & 23.86 & 26.30 & 26.56\\ 
      & SSIM & 0.6223 & 0.5244 & 0.8040 & 0.7689  & 0.8152 & 0.7035 & 0.6616 & 0.7478 & 0.7846 & 0.6583 & 0.7827 & 0.7961 \\\hline
\multirow{3}{*}{IISTA} & Tol & 4.0e-1 &1.0e0 & 2.4e-1& 1.1e0  & 1.6e0 & 1.0e0 & 1.1e0 & 8.6e-1 & 7.9e-1 & 7.4e-1 & 6.2e-2 & 6.7e-2  \\
	  & PSNR & 20.84 &18.93 & 26.78 & 21.36  & 18.89 & 18.24 & 17.61 & 19.61 & 23.39 & 21.54 & 23.65 & 24.48\\ 
      & SSIM & 0.4255 &0.3595 & 0.7447 & 0.5814  & 0.5342 & 0.4715 & 0.4960 & 0.5761 & 0.6513 & 0.5443 & 0.5956 & 0.6630  \\\hline
\multirow{3}{*}{EFISTA} & Tol & 1.0e-2 &9.3e-3 & 3.1e-3 & 4.5e-3  & 9.9e-3 & 1.6e-2 & 1.2e-2 & 6.0e-3 & 1.0e-2 & 9.2e-3 & 1.0e-3 & 1.1e-3   \\
	  & PSNR & 25.34 & 22.99 & 32.28 & 30.76  & 25.88 & 23.44 & 26.31 & 27.77 & 30.42 & 25.93 & 29.75 & 29.18\\ 
      & SSIM & 0.8003 & 0.6496 & 0.8539 & 0.8121  & 0.9106 & 0.8425 & 0.8023 & 0.8434 & 0.8778 & 0.7486 & 0.8950 & 0.8848  \\\hline
\multirow{3}{*}{IOptISTA} & Tol & \textbf{4.5e-3} & \textbf{6.2e-3} & \textbf{1.7e-3} & \textbf{3.8e-3}  & \textbf{4.2e-3} & \textbf{7.1e-3} &\textbf{4.5e-3} & \textbf{2.7e-3} & \textbf{3.8e-3} & \textbf{4.9e-3} &\textbf{5.0e-4} & \textbf{5.0e-4}  \\
	  & PSNR & \textbf{26.70} & \textbf{23.66} &  \textbf{33.28} & \textbf{31.48}  & \textbf{26.80} & \textbf{24.63} &\textbf{27.89} & \textbf{28.86} & \textbf{31.81} & \textbf{26.85} &\textbf{31.05} & \textbf{30.28}\\ 
      & SSIM & \textbf{0.8467} &\textbf{0.6978} & \textbf{0.8671} & \textbf{0.8244}  & \textbf{0.9285} & \textbf{0.8789} &\textbf{0.8412} & \textbf{0.8835} & \textbf{0.8973} & \textbf{0.7811} &\textbf{0.9164} & \textbf{0.9066} \\\hline
\multicolumn{14}{c}{Noise $\varepsilon\sim \mathcal{N}(0,5\times 10^{-4})$} \\\hline
\multirow{3}{*}{ISTA} & Tol & 2.6e0 & 2.3e0 & 3.1e0 & 3.1e0  & 7.9e0 & 8.9e0 & 1.4e+1 & 2.2e+1 & 6.9e0 & 6.9e0 & 6.5e-1 & 9.1e-1   \\
	  & PSNR & 19.70 &  19.64 & 23.89 & 23.56  & 16.43 & 16.61 & 14.93 & 15.79 & 21.29 & 20.31 & 21.57 & 22.40 \\ 
      & SSIM & 0.2934 & 0.2906 & 0.6669 & 0.6609  & 0.1911 & 0.2221 & 0.2602 & 0.3784 & 0.5269 & 0.4821 & 0.3938 & 0.4644  \\\hline
\multirow{3}{*}{FISTA} & Tol & 1.5e-1 & 1.1e-1 & 7.4e-2 & 7.2e-2  & 2.8e-1 & 3.3e-1 & 2.4e-1 & 1.9e-1 & 2.2e-1 & 1.4e-1 & 2.3e-2 & 2.3e-2   \\
	  & PSNR & 21.84 & 21.22 & 28.56 & 27.67  & 21.55 & 19.64 & 19.71 & 22.80 & 25.26 & 23.25 & 25.15 & 25.69 \\ 
      & SSIM & 0.5445 & 0.4761 & 0.7780 & 0.7472  & 0.7460 & 0.6255 & 0.6147 & 0.6880 & 0.7340 & 0.6250 & 0.7167 & 0.7499  \\\hline
    \multirow{3}{*}{OptISTA} & Tol & 8.9e-2 & 6.9e-2 & 4.6e-2 & 4.8e-2  & 1.3e-1 & 1.6e-1 & 1.1e-1 & 6.5e-2 & 1.2e-1 & 8.8e-2 & 1.3e-2 & 1.3e-2   \\
	  & PSNR & 22.57 & 21.64 & 29.54 & 28.56  & 22.88 & 20.62 & 21.13 & 24.62 & 26.51 & 23.84 & 26.16 & 26.46 \\ 
      & SSIM & 0.6182 & 0.5222 & 0.7932 & 0.7616  & 0.8136 & 0.7025 & 0.6605 & 0.7465 & 0.7783 & 0.6542 & 0.7732 & 0.7896  \\\hline
\multirow{3}{*}{IISTA} & Tol & 4.1e-1 & 9.2e-1 & 2.5e-1 & 1.0e0  & 1.6e0 & 1.1e0 & 1.2e0 & 8.8e-1 & 8.3e-1 & 7.7e-1 & 6.9e-2 & 7.3e-2   \\
	  & PSNR & 20.84 & 19.04 & 26.77 & 21.36  & 18.89 & 18.24 & 17.61 & 19.61 & 23.39 & 21.54 & 23.64 & 24.47 \\ 
      & SSIM & 0.4255 & 0.3614 & 0.7440 & 0.5811  & 0.5341 & 0.4715 & 0.4958 & 0.5761 & 0.6510 & 0.5441 & 0.5952 & 0.6623  \\\hline
\multirow{3}{*}{EFISTA} & Tol & 4.6e-2 & 9.5e-2 & 3.9e-2 & 3.7e-2  & 5.1e-2 & 5.7e-2 & 3.0e-2 & 2.9e-2 & 5.1e-2 & 1.1e-1 & 9.4e-3 & 1.2e-2   \\
	  & PSNR & 24.74 & 22.37 & 30.48 & 29.36  & 24.27 & 23.25 & 25.69 & 27.17 & 29.21 & 24.39 & 27.65 & 27.96 \\ 
      & SSIM & 0.7469 & 0.6179 & 0.7952 & 0.7589  & 0.8803 & 0.8320 & 0.7827 & 0.8237 & 0.8207 & 0.6722 & 0.8237 & 0.8374  \\\hline
\multirow{3}{*}{IOptISTA} & Tol & \textbf{1.8e-2} & \textbf{2.8e-2} & \textbf{1.6e-2} & \textbf{2.6e-2} & \textbf{1.9e-2} & \textbf{2.7e-2} & \textbf{1.0e-2} & \textbf{1.2e-2} & \textbf{1.9e-2} & \textbf{3.4e-2} & \textbf{2.4e-3} & \textbf{2.8e-3}   \\
	  & PSNR & \textbf{24.99} & \textbf{23.17} & \textbf{30.67} & \textbf{29.82} & \textbf{25.34} & \textbf{24.03} & \textbf{26.21} & \textbf{27.26}  & \textbf{29.40} & \textbf{26.06} & \textbf{27.83} & \textbf{28.30} \\ 
      & SSIM & \textbf{0.7484} &\textbf{ 0.6429} & \textbf{0.8098} & \textbf{0.7674} & \textbf{0.8904} & \textbf{0.8529} & \textbf{0.7969} & \textbf{0.8332} & \textbf{0.8338} & \textbf{0.7115} & \textbf{0.8302} & \textbf{0.8453}  \\\hline
\end{tabular}
\end{center}
\end{table*}

The results presented in Table \ref{real_l1_table} demonstrate that, when compared to other algorithms, the IOptISTA consistently attains the lowest error values. Across all experiments, the IOptISTA algorithm outperforms other methods in delivering superior PSNR and SSIM values. Collectively, the proposed algorithm exhibits enhanced numerical performance. For further insights, refer to Figures \ref{tps_results_L1_img02} and \ref{tps_results_L1_img05}.

\begin{figure}[!ht]
\setlength\tabcolsep{2pt}
\centering
\begin{tabular}{ccc} 
\includegraphics[width=0.32\textwidth, height=3.5cm]{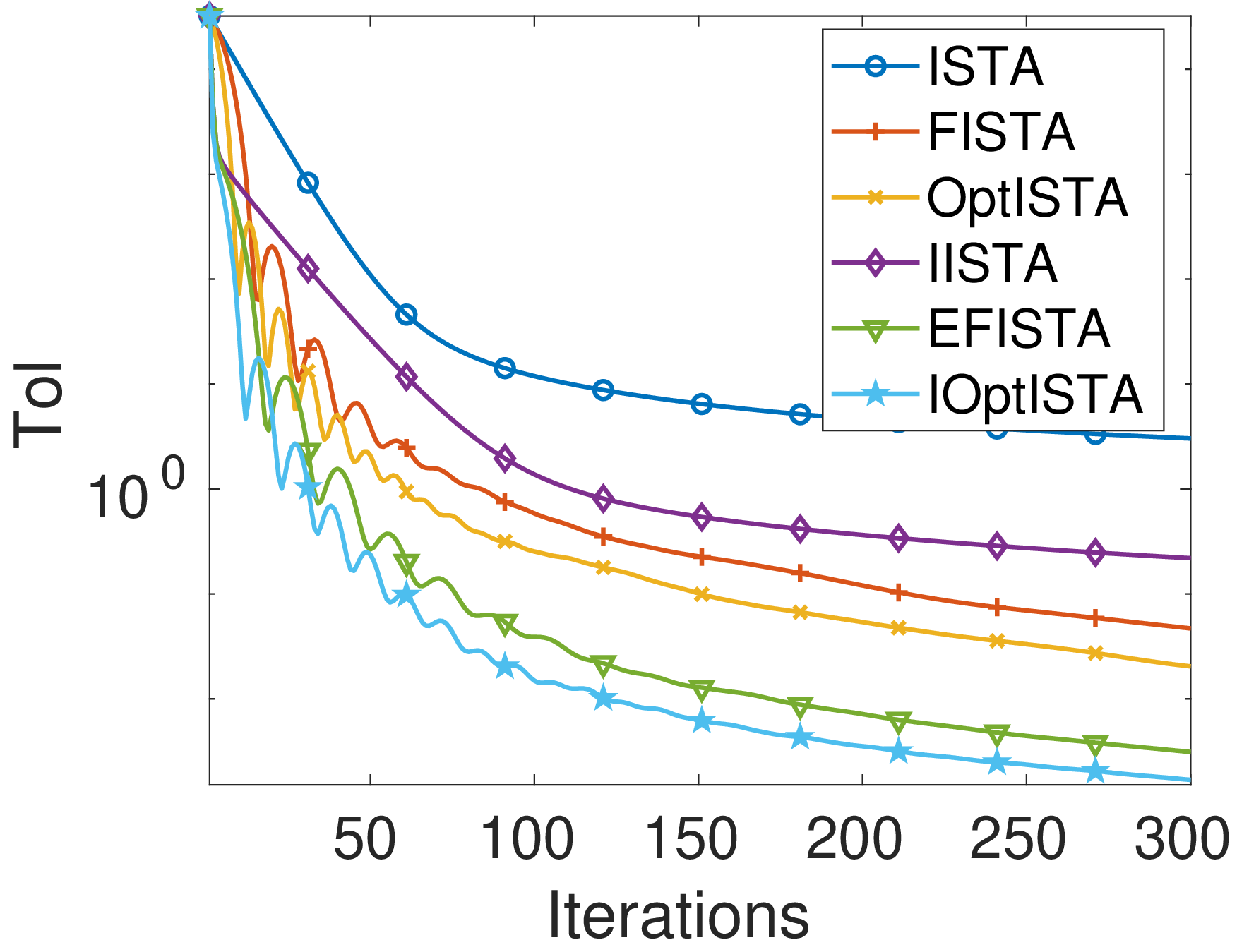} & \includegraphics[width=0.32\textwidth, height=3.5cm]{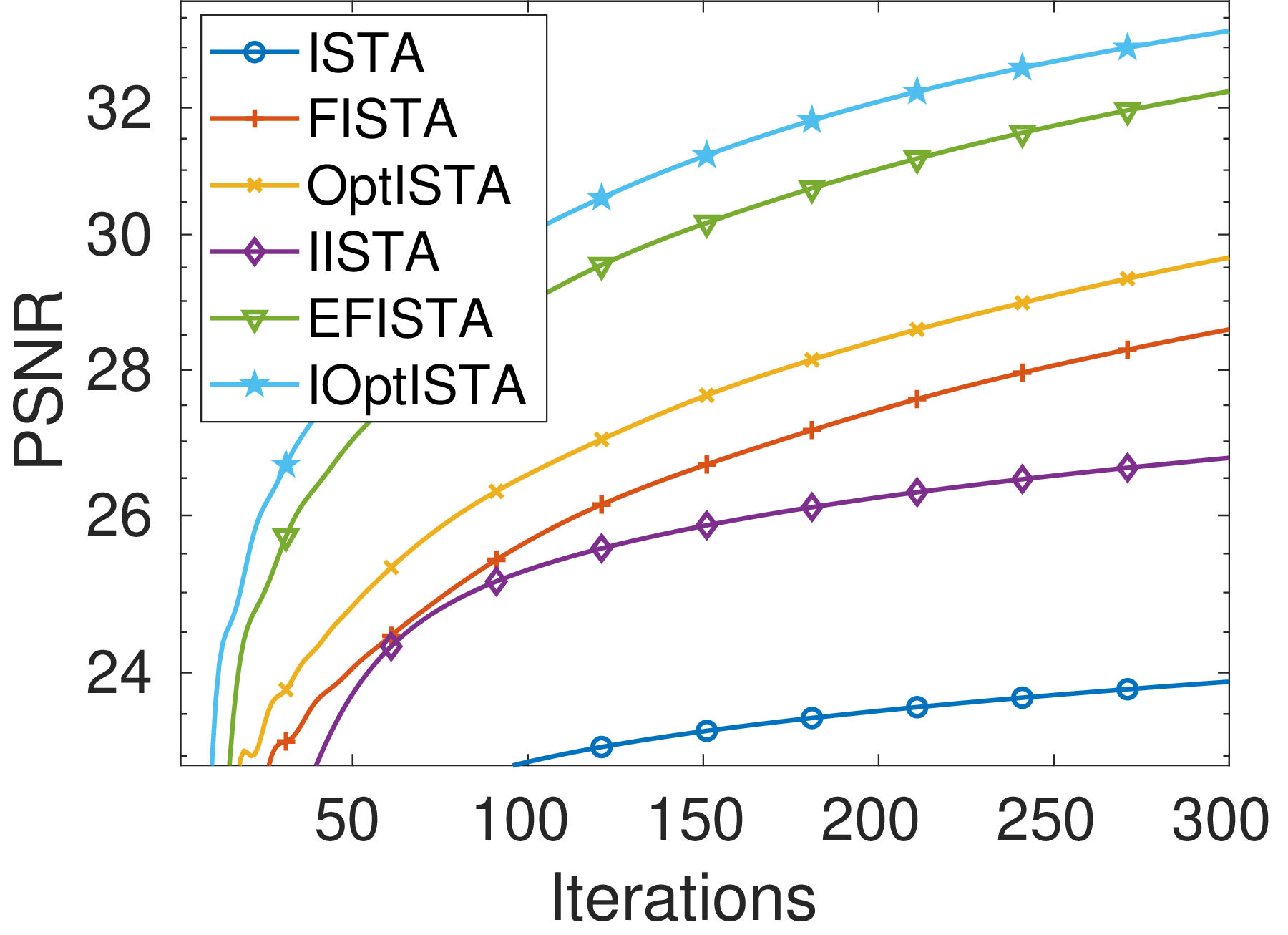}& \includegraphics[width=0.32\textwidth, height=3.5cm]{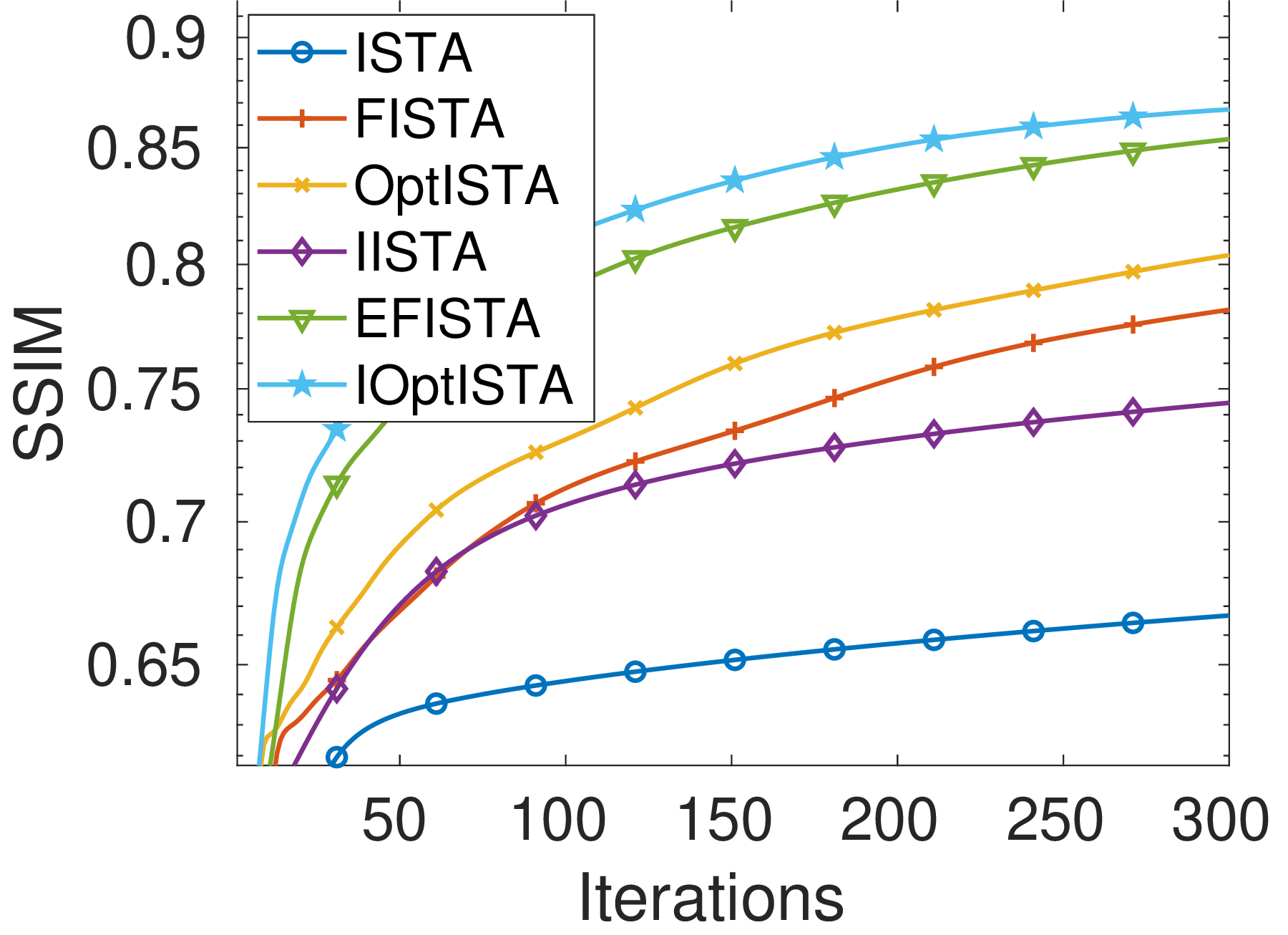}\\
	(a) Tol & (b) PNSR & (c) SSIM
\end{tabular}
\caption{The Tol, PSNR and SSIM results for Figure \ref{details_images}(b)  with ``$k = \text{fspecial}(\text{`disk'}, 12)$'' and $\varepsilon\sim\mathcal{N}(0, 1e-4)$.}
\label{tps_results_L1_img02}
\end{figure}
\begin{figure}[!ht]
\setlength\tabcolsep{2pt}
\centering
\begin{tabular}{ccc} 
\includegraphics[width=0.32\textwidth, height=3.5cm]{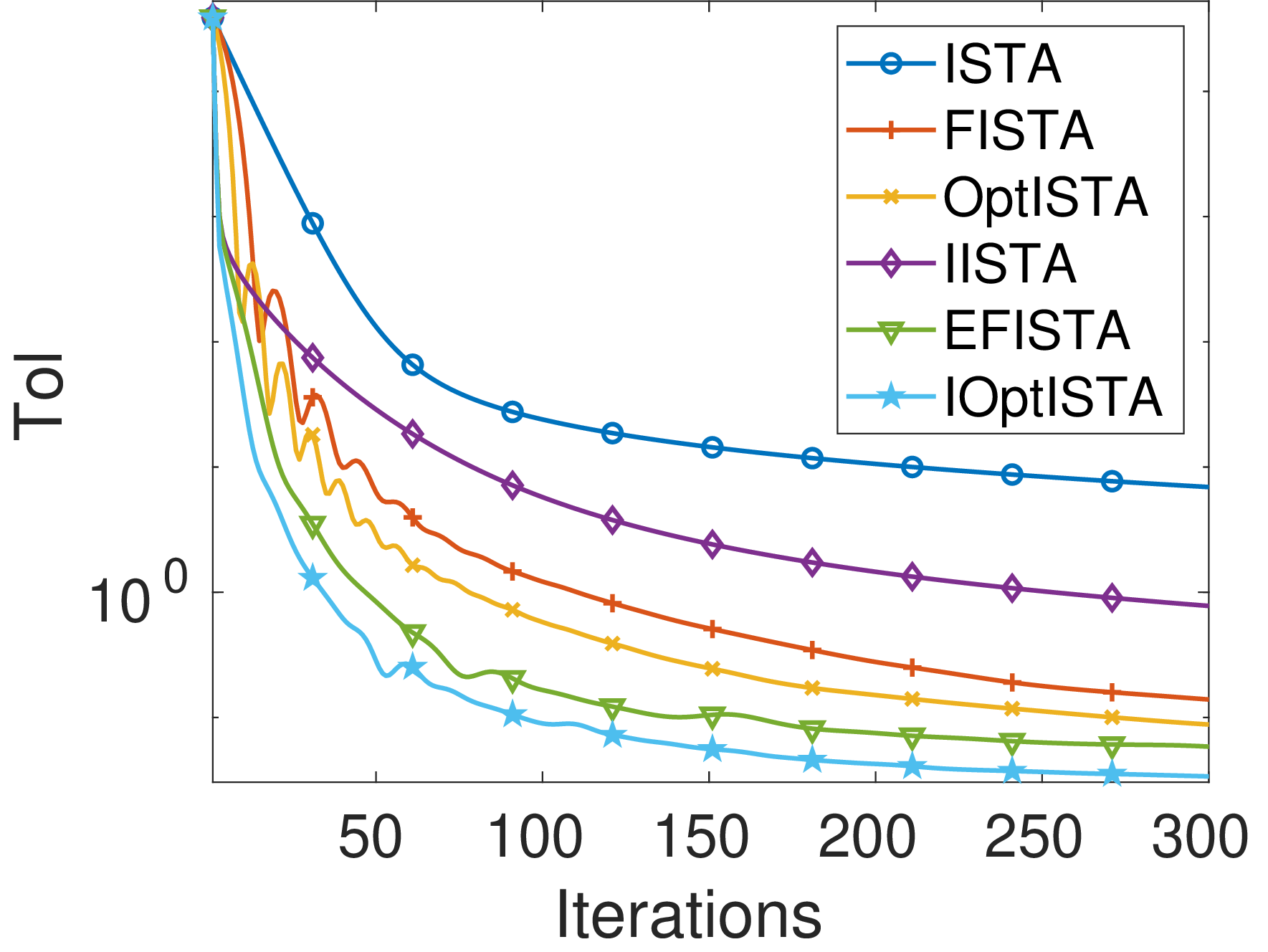} & \includegraphics[width=0.32\textwidth, height=3.5cm]{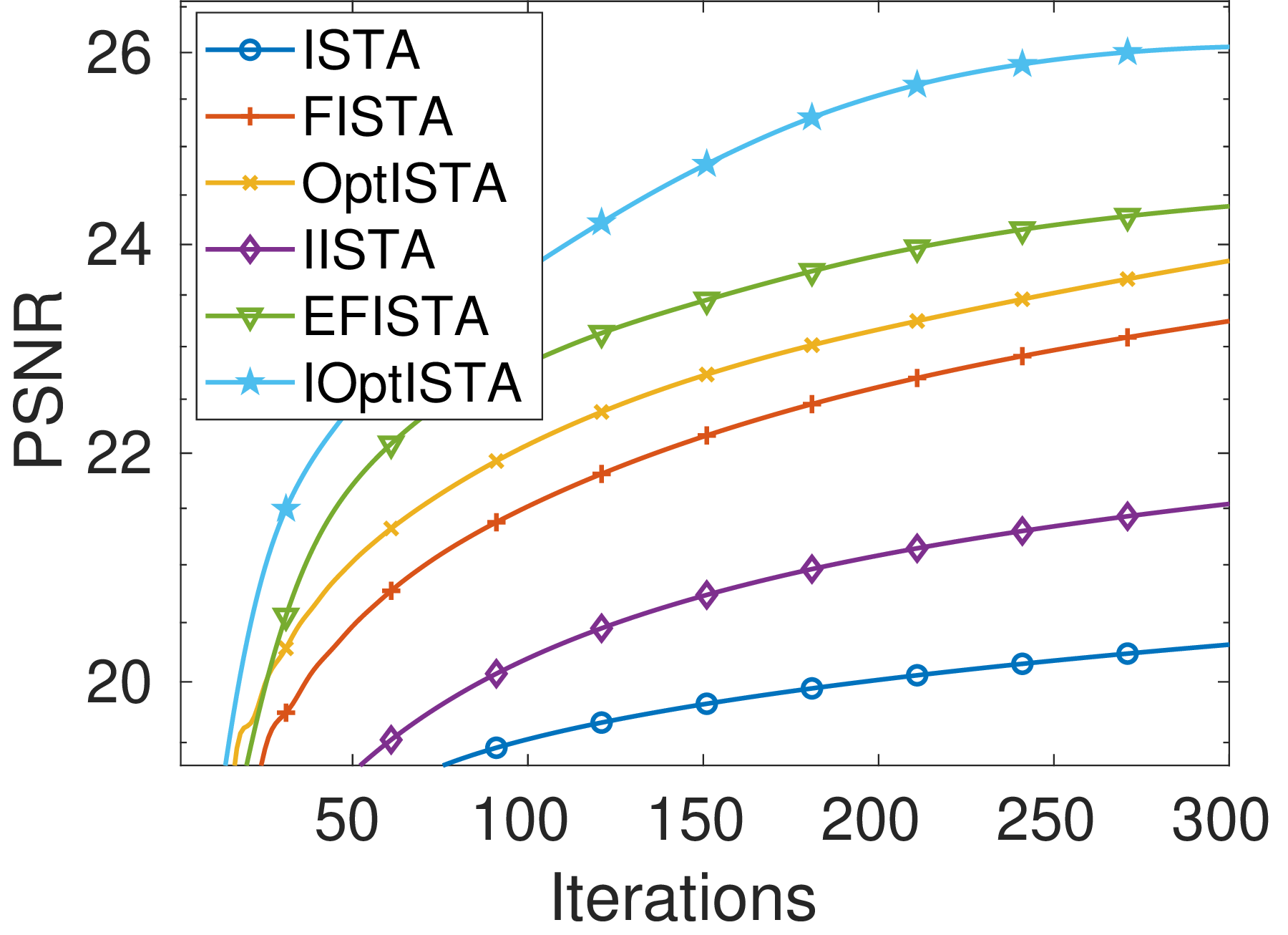}& \includegraphics[width=0.32\textwidth, height=3.5cm]{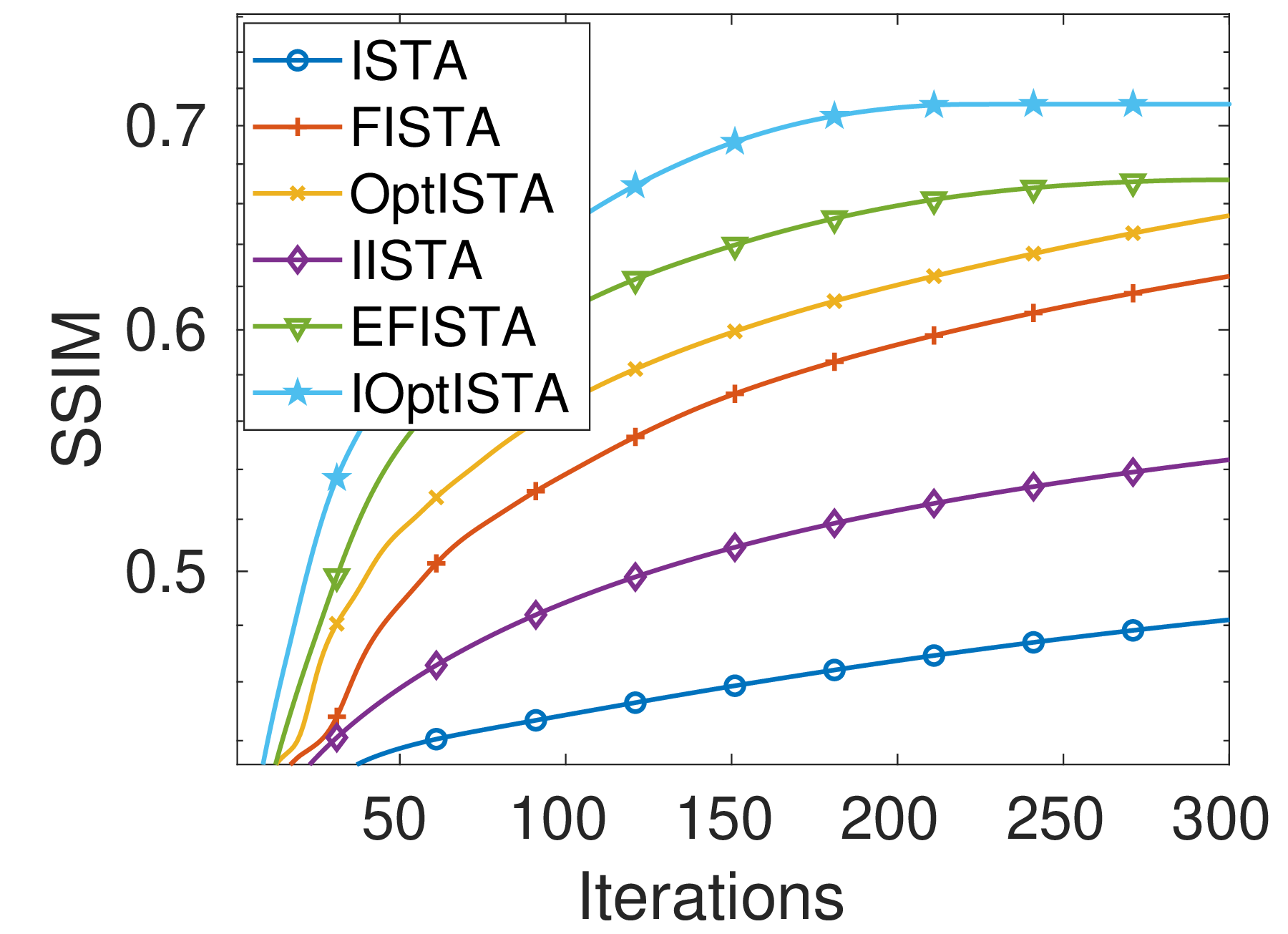}\\
	(d) Tol & (e) PNSR & (f) SSIM
\end{tabular}
\caption{The Tol, PSNR and SSIM results for Figure \ref{details_images}(e)  with ``$k = \text{fspecial}(\text{`gaussian'}, 25, 35)$'' and $\varepsilon\sim\mathcal{N}(0, 5e-4)$.} 
\label{tps_results_L1_img05}
\end{figure}

Figures \ref{tps_results_L1_img02} and \ref{tps_results_L1_img05} illustrate the trends in error (Tol), PSNR, and SSIM values across iterations for Figure \ref{details_images}(b) with ``$k = \text{fspecial}(\text{`disk'}, 12)$'' and $\varepsilon\sim\mathcal{N}(0, 1e-4)$, and for the Figure \ref{details_images}(e)  with ``$k = \text{fspecial}(\text{`gaussian'}, 25, 35)$'' and $\varepsilon\sim\mathcal{N}(0, 5e-4)$, respectively. These figures indicate that from the initial stages, the IOptISTA algorithm outperforms other methods. Notably, IOptISTA achieves superior results in terms of both error, PSNR, and SSIM values. We can also observe that when the blur operator is fixed, the performance of all algorithms degrades as the noise level increases, and the numerical performance gap between the algorithms becomes smaller. Nonetheless, the algorithm proposed in this paper consistently achieves better numerical results.

\begin{figure}[!ht]
\setlength\tabcolsep{2pt}
\centering
\begin{tabular}{cccccc} 
\includegraphics[width=0.23\textwidth]{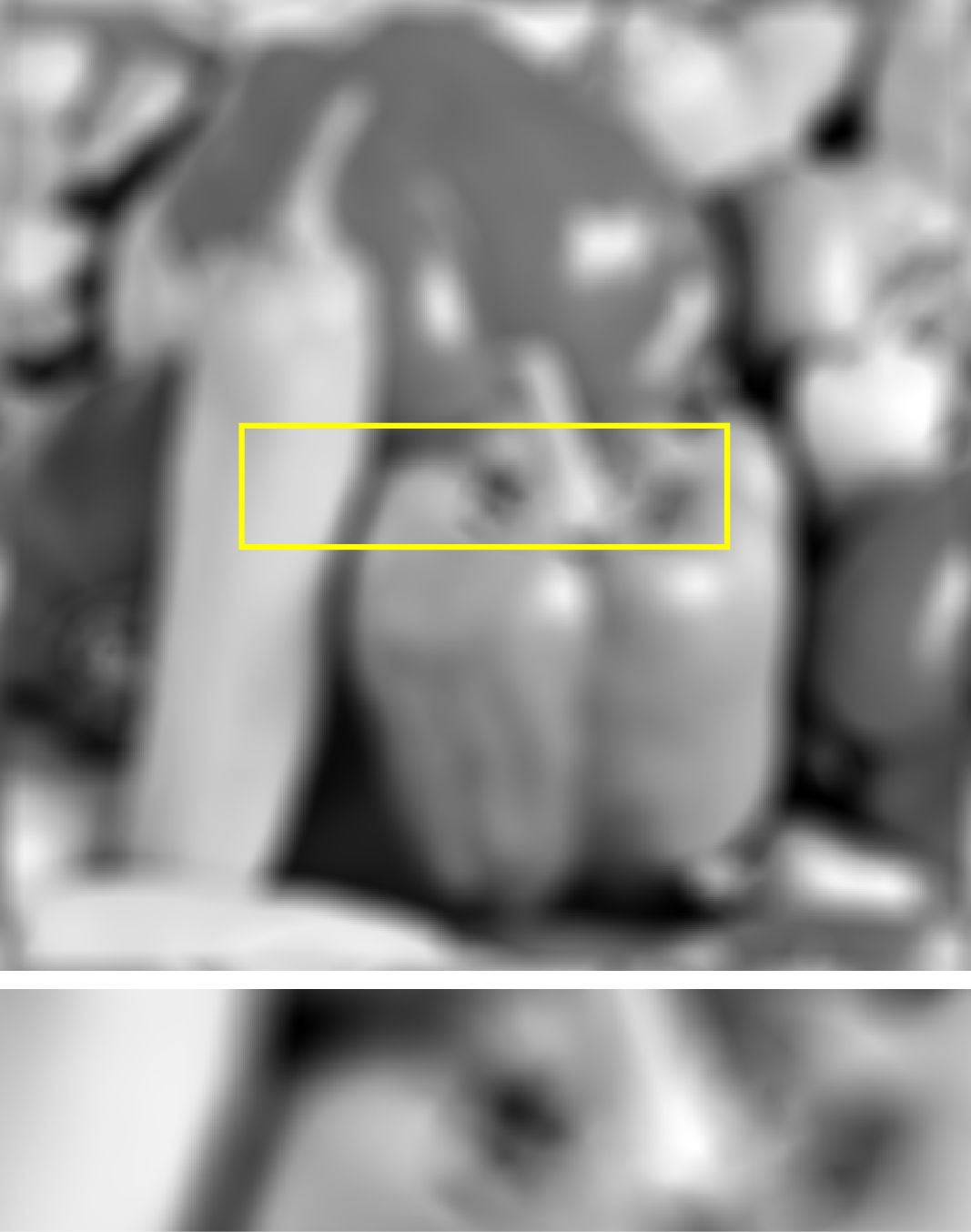} & \includegraphics[width=0.23\textwidth]{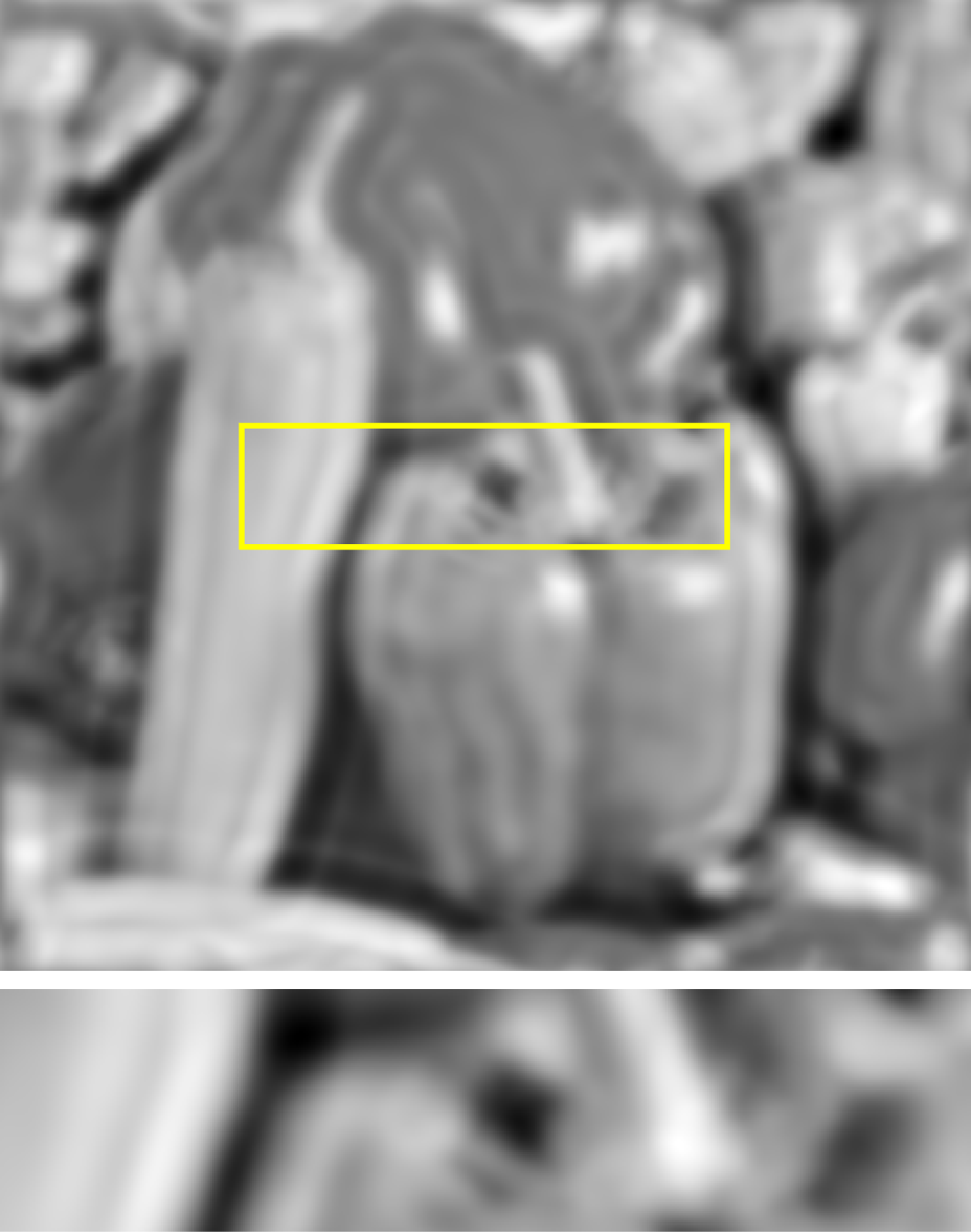}&
\includegraphics[width=0.23\textwidth]{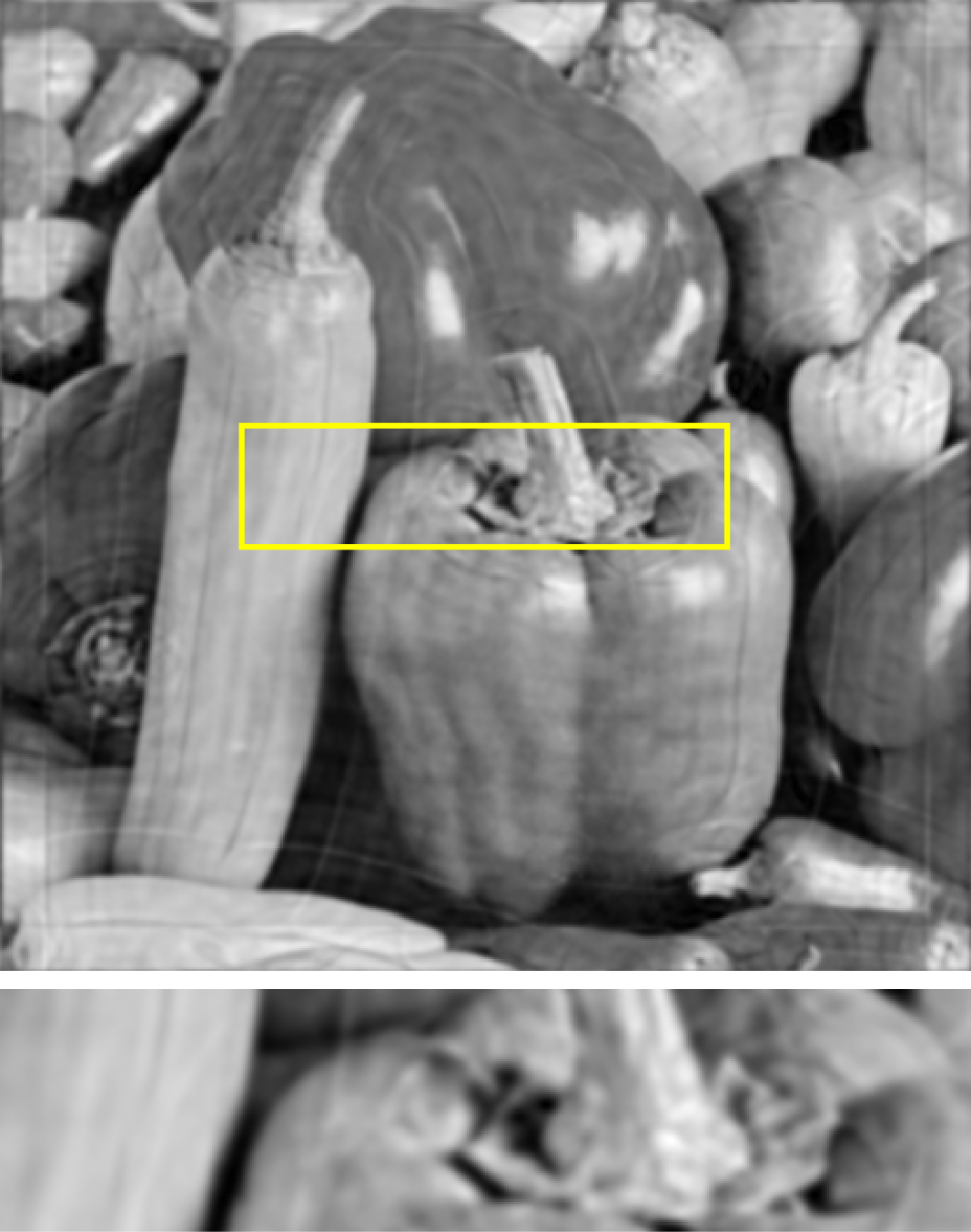}&
\includegraphics[width=0.23\textwidth]{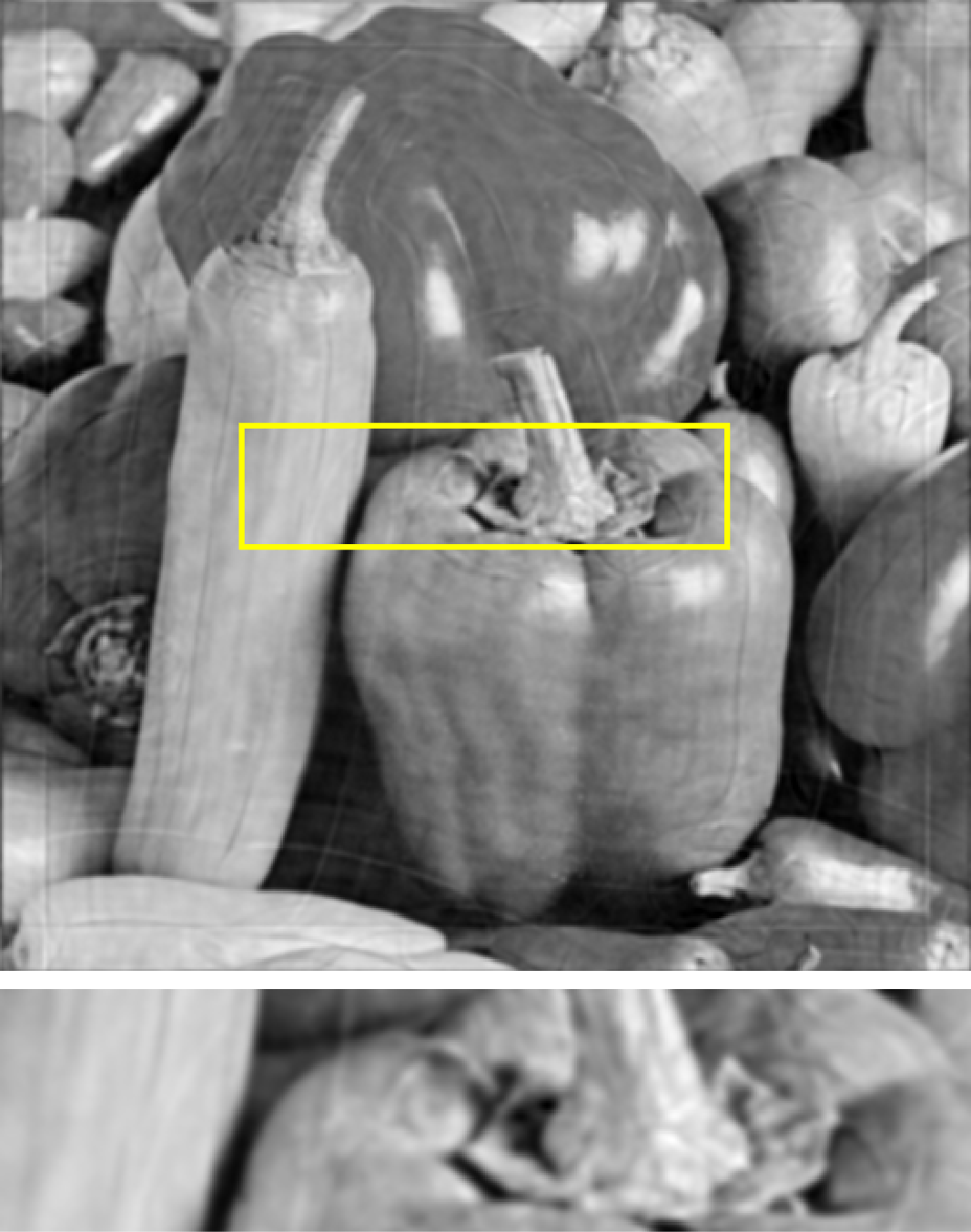}\\
(a) Noised & (b) ISTA & (c) FISTA&(d) OptISTA \\
\includegraphics[width=0.23\textwidth]{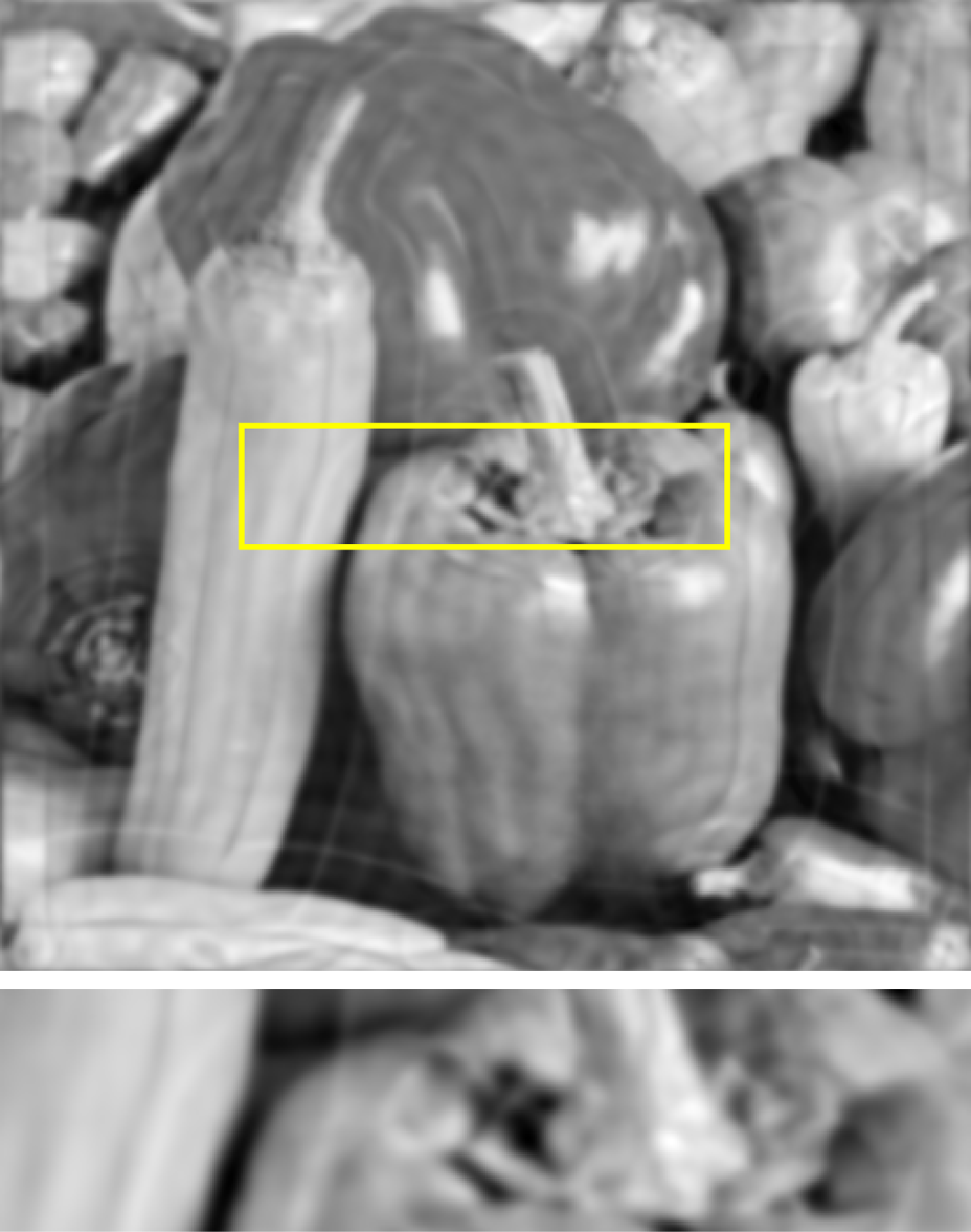} & \includegraphics[width=0.23\textwidth]{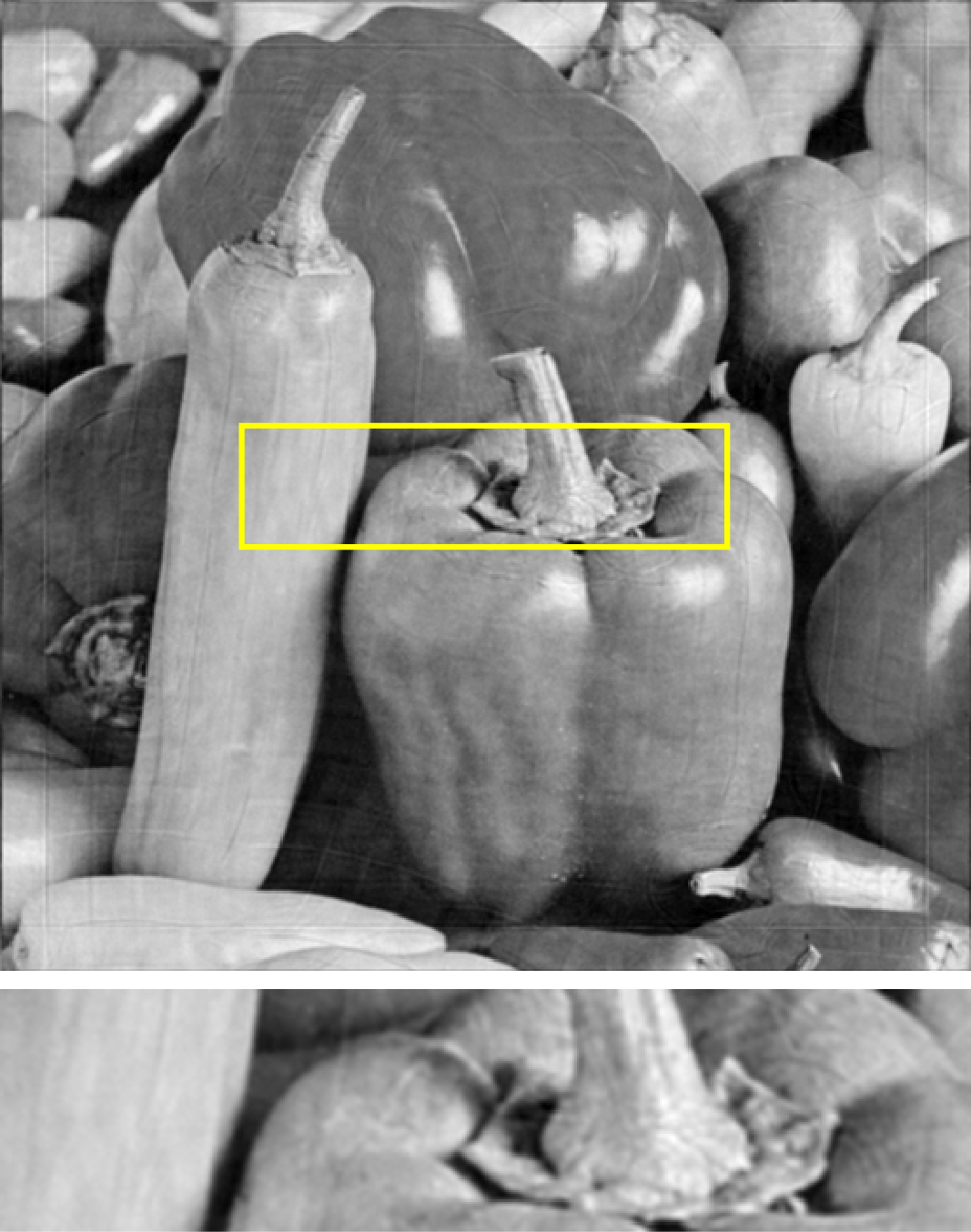}&
\includegraphics[width=0.23\textwidth]{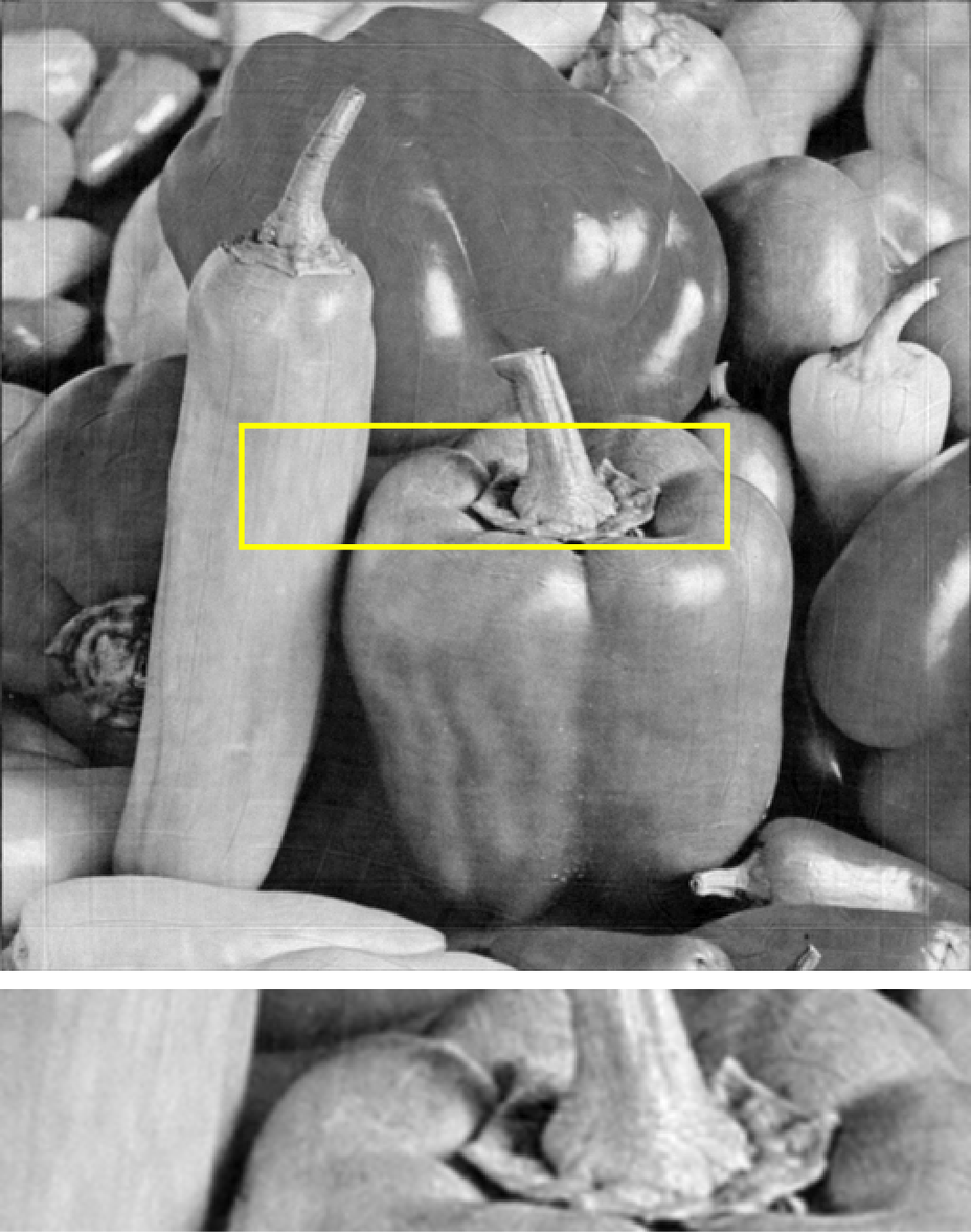}&
\includegraphics[width=0.23\textwidth]{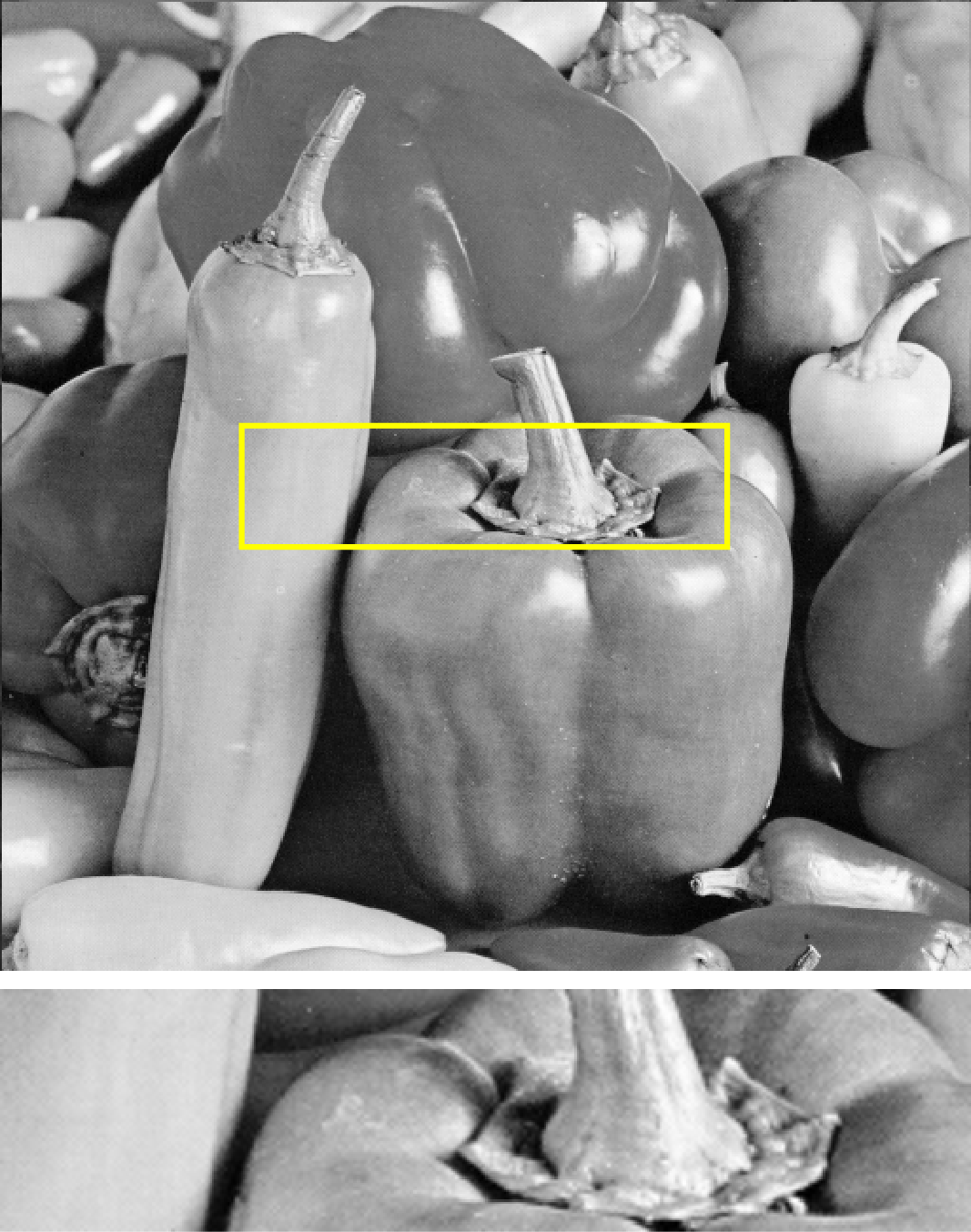}\\
(e) IISTA & (f) EFISTA & (g) IOptISTA & (h) Original 
\end{tabular}
\caption{Blurred and restored images for the Figure \ref{details_images}(b) with ``$k = \text{fspecial}(\text{`disk'}, 12)$'' and $\varepsilon\sim\mathcal{N}(0, 1e-4)$.} 
\label{results_img02_L1}
\end{figure}

\begin{figure}[!ht]
\setlength\tabcolsep{2pt}
\centering
\begin{tabular}{cccccc} 
\includegraphics[width=0.23\textwidth]{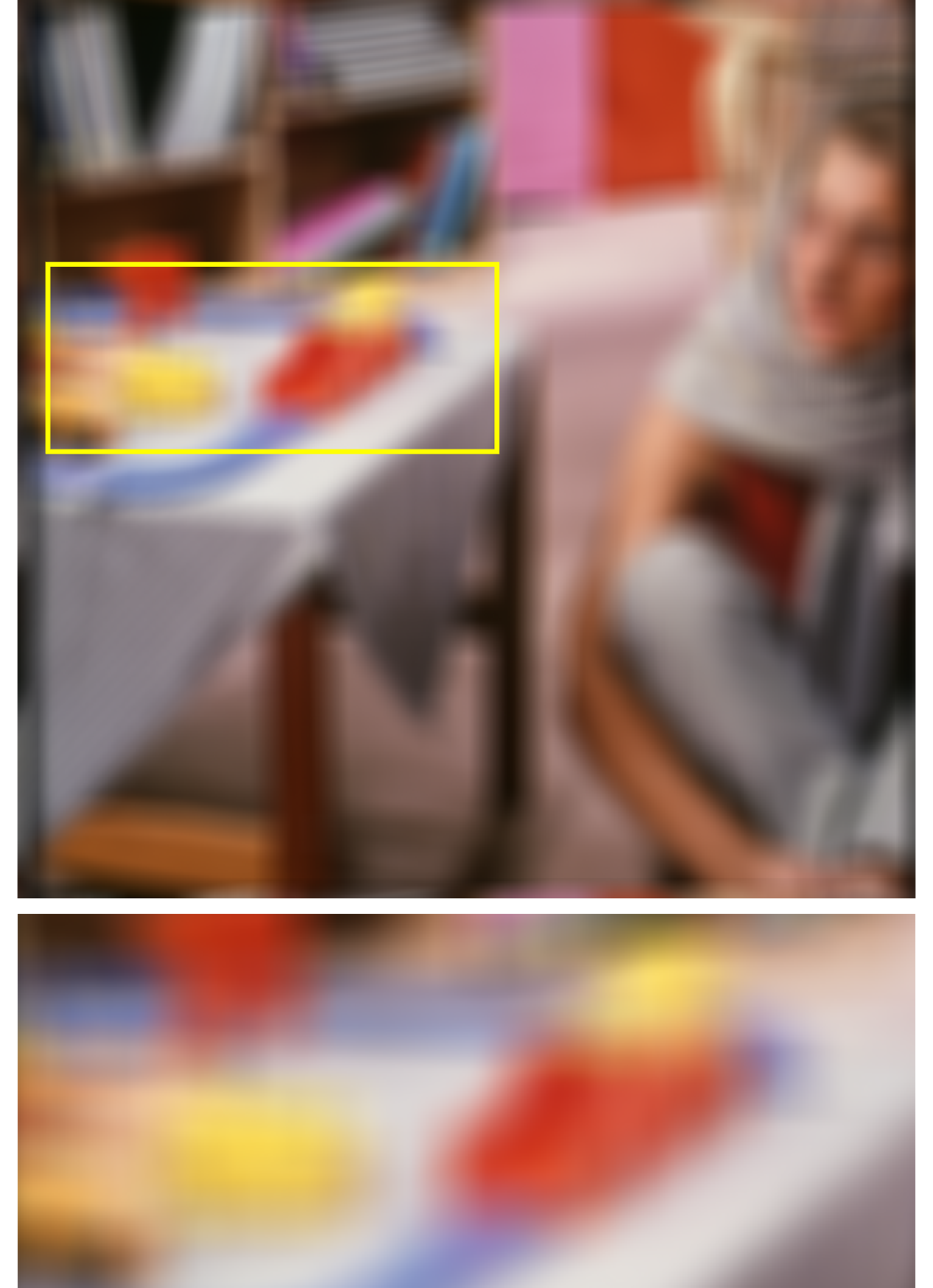} & \includegraphics[width=0.23\textwidth]{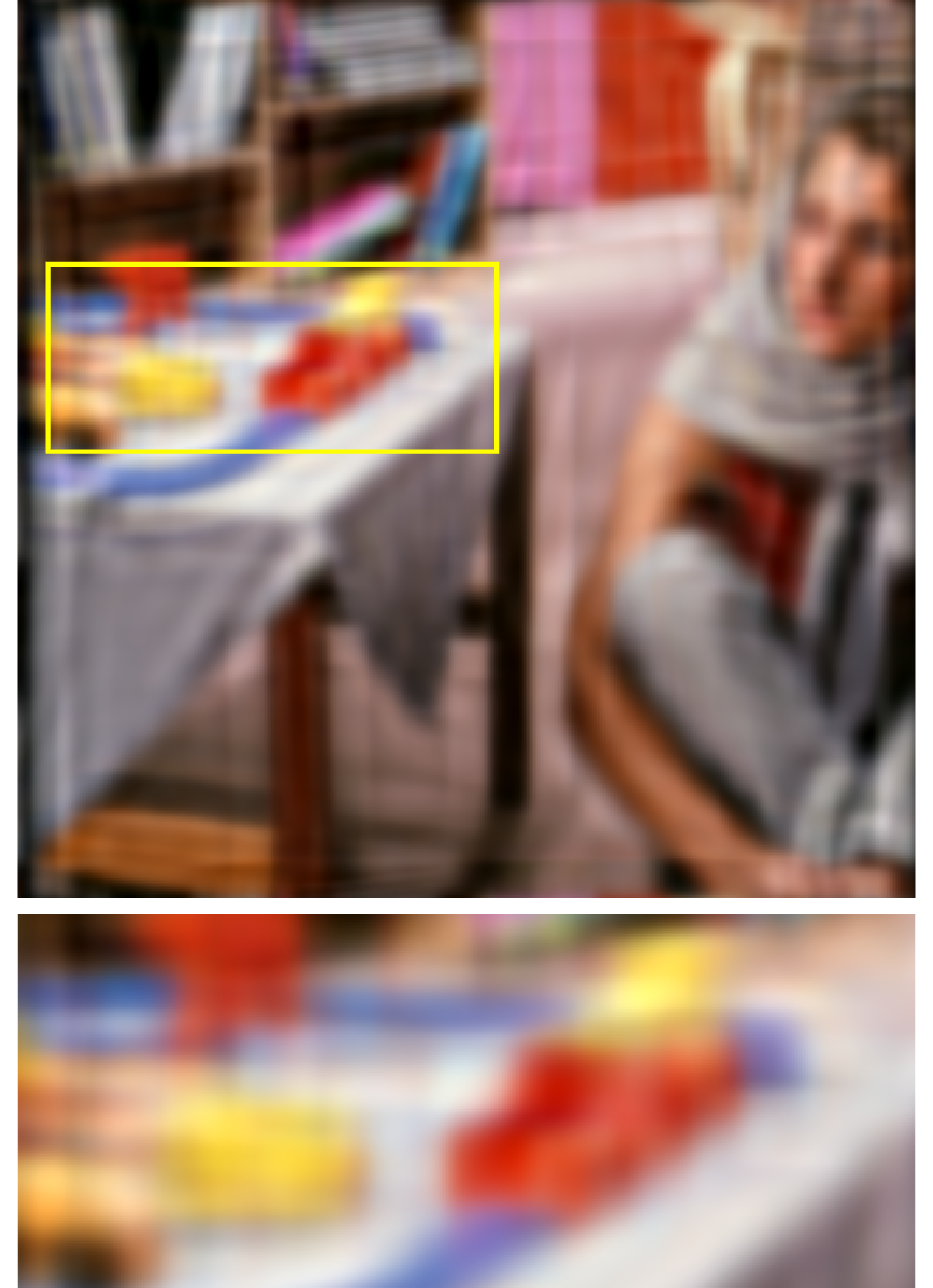}&
\includegraphics[width=0.23\textwidth]{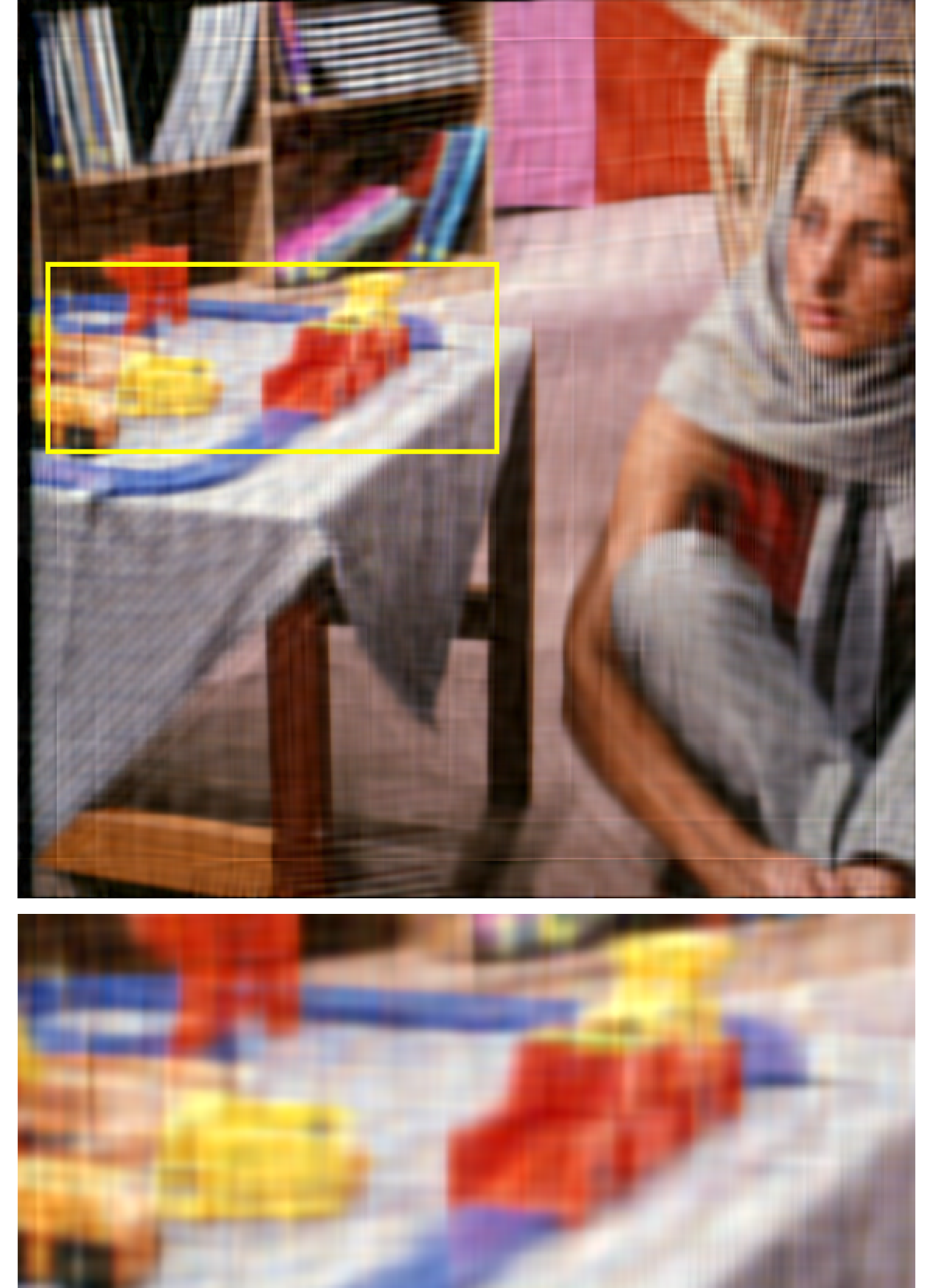}&
\includegraphics[width=0.23\textwidth]{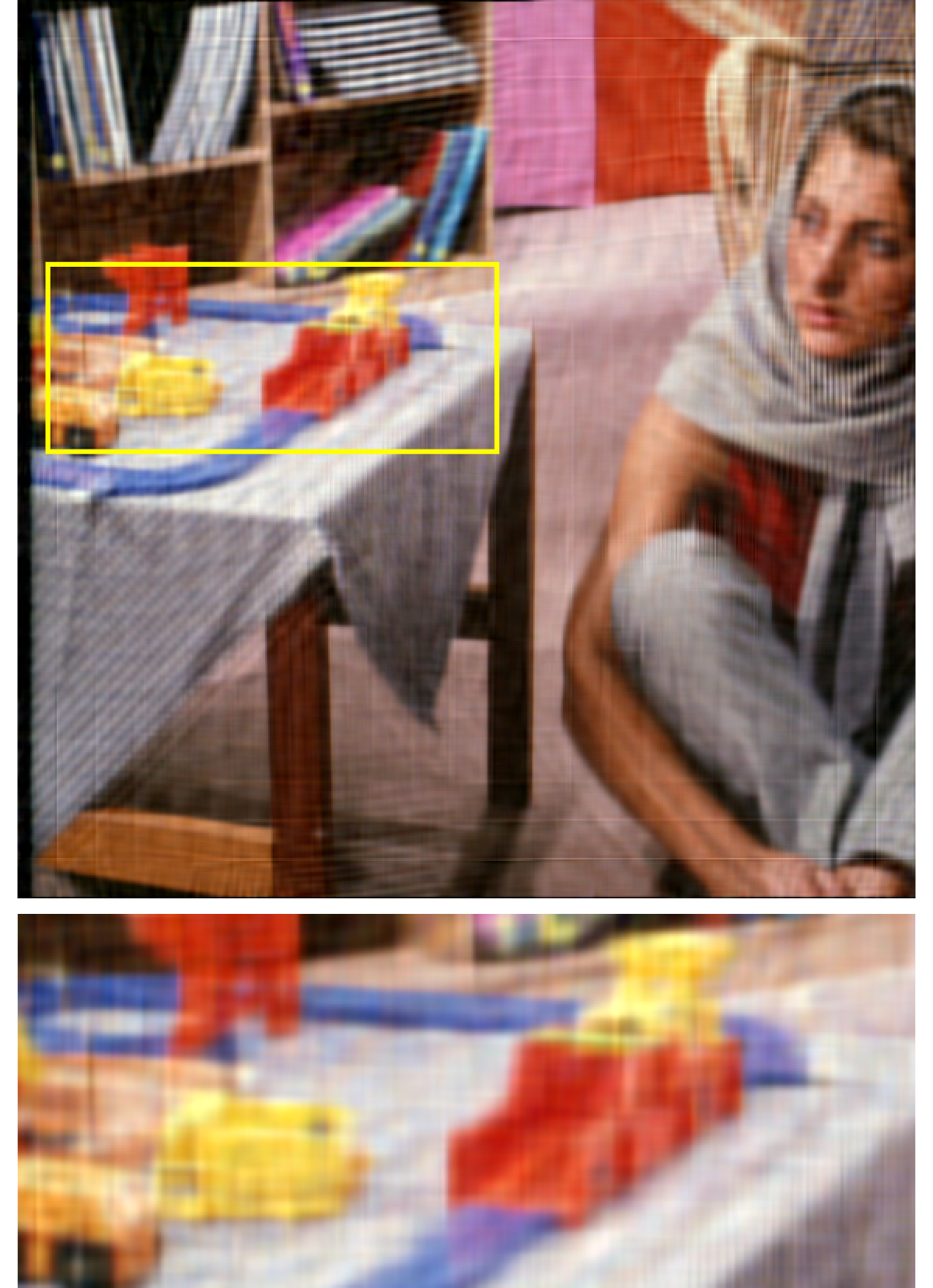}\\
(a) Noised & (b) ISTA & (c) FISTA&(d) OptISTA \\
\includegraphics[width=0.23\textwidth]{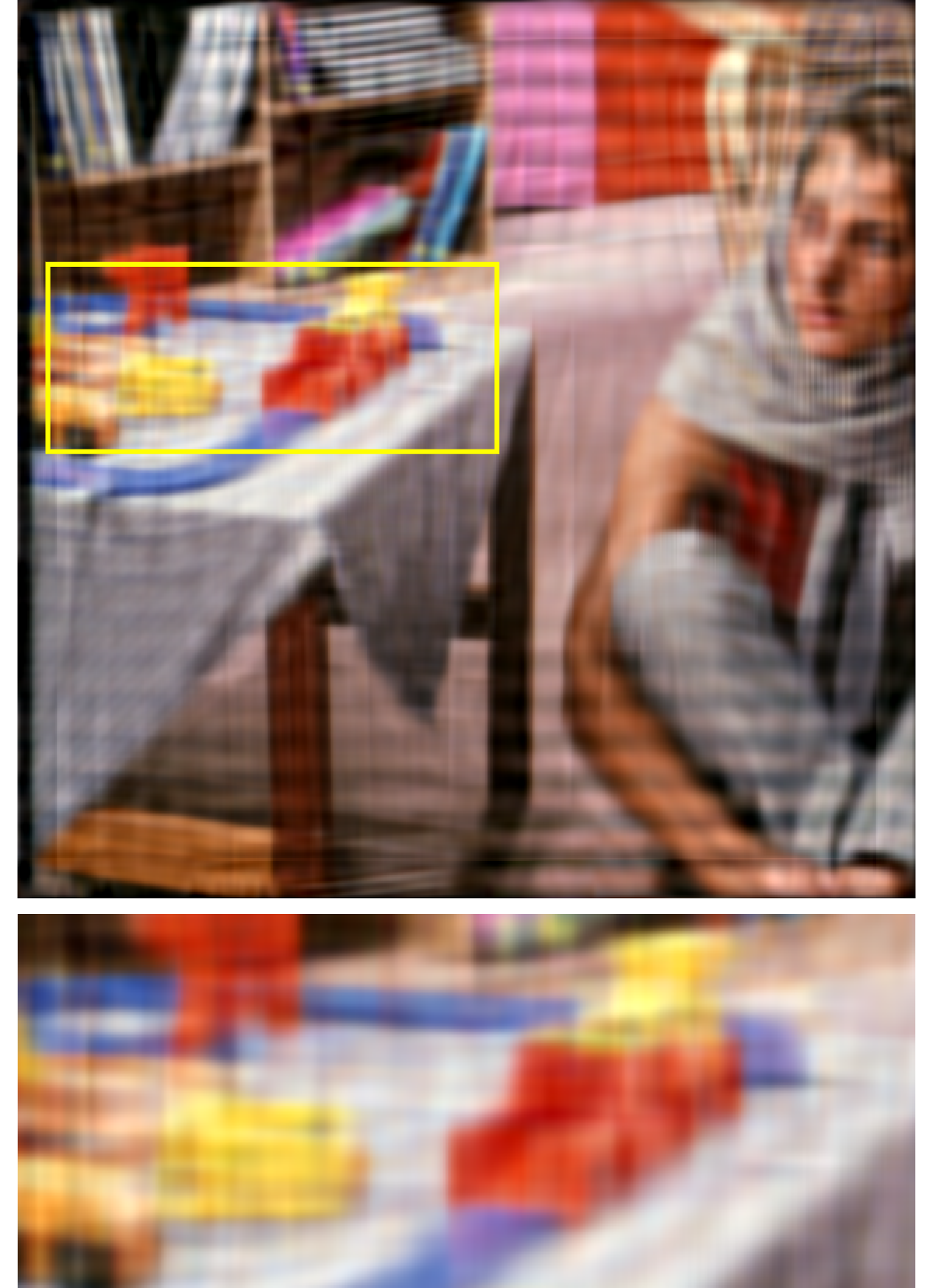} & \includegraphics[width=0.23\textwidth]{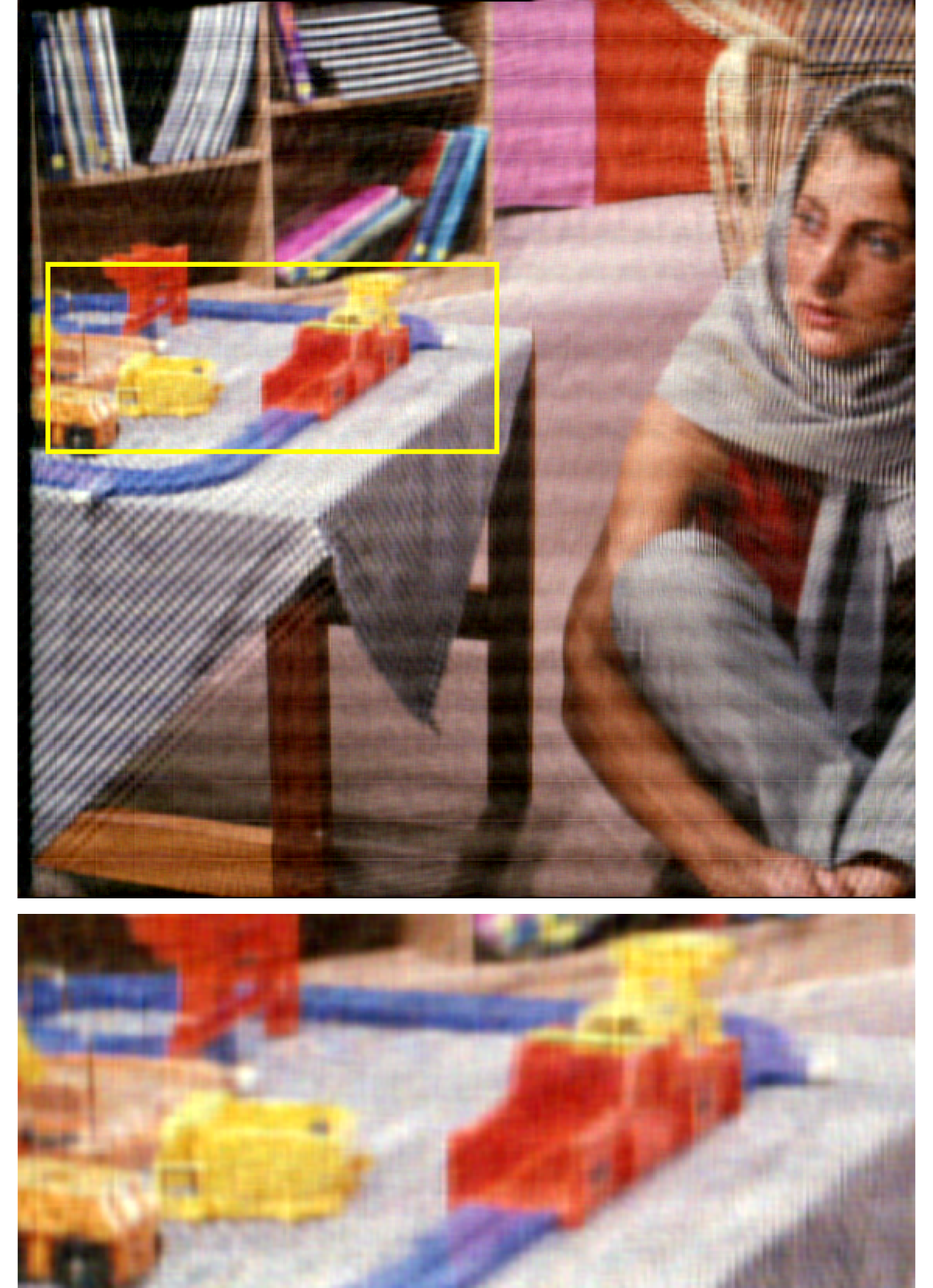}&
\includegraphics[width=0.23\textwidth]{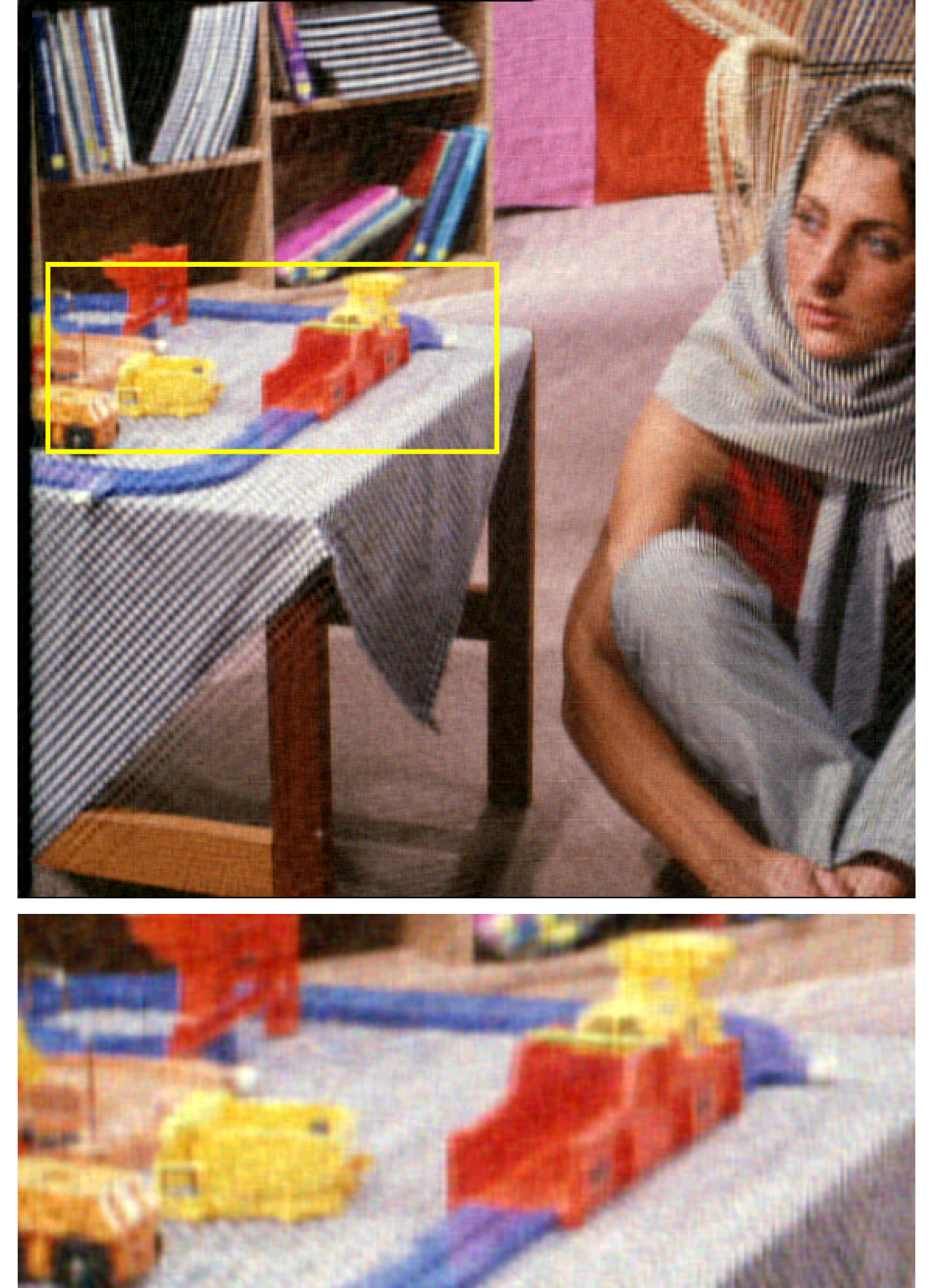}&
\includegraphics[width=0.23\textwidth]{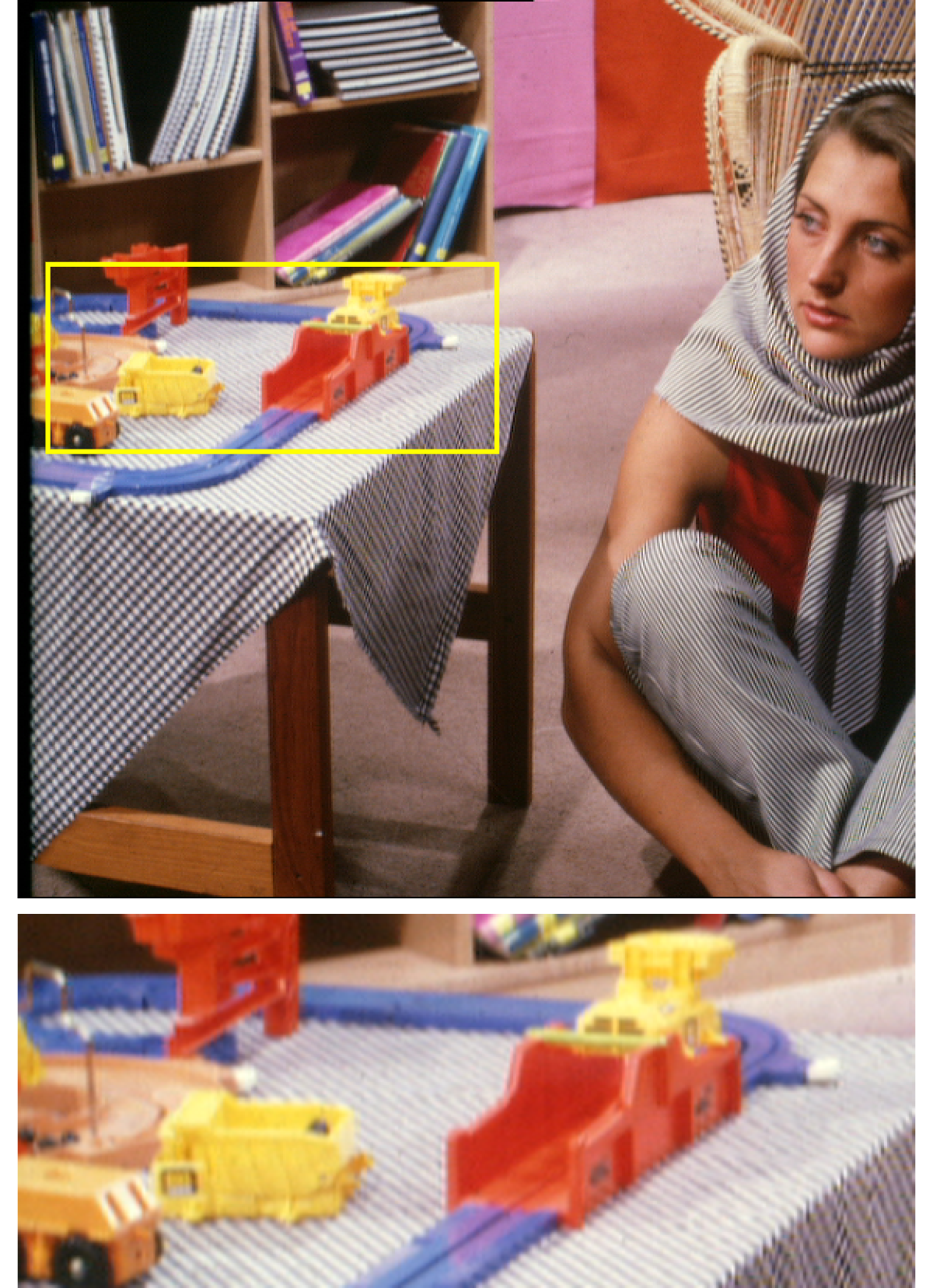}\\
(e) IISTA & (f) EFISTA & (g) IOptISTA & (h) Original 
\end{tabular}
\caption{Blurred and restored images for the Figure \ref{details_images}(e)  with ``$k = \text{fspecial}(\text{`gaussian'}, 25, 35)$'' and $\varepsilon\sim\mathcal{N}(0, 5e-4)$.} 
\label{results_img_05_L1}
\end{figure}

In Figures \ref{results_img02_L1} and \ref{results_img_05_L1}, we present the blurred and restored images resulting from the compared algorithms. As can be observed, the numerical outcomes of the IOptISTA algorithm are markedly superior.

\subsection{Case two: Total variation (TV) regularization}\label{tv-regularization}

In this subsection, we focus on the case where $h(x) = \lambda \|x\|_{\mathrm{TV}}$ \cite{ChanCW99, JiangL21} in the optimization problem \eqref{target-function}, which is formulated as:
\begin{eqnarray}
\min_{x \in \mathbb{R}^n} \quad \frac{1}{2} \|Ax - b\|^2 + \lambda \|x\|_{\mathrm{TV}}, \label{tv-model}
\end{eqnarray}
where $\lambda > 0$ and $\|x\|_{\mathrm{TV}} := \sum_{i=2}^{n} |x_i - x_{i-1}|$ represents the total variation (TV) norm of $x$ \cite{RudinOF92}. To implement the proximal operator of the TV norm, we utilize the direct algorithm\footnote{\url{https://lcondat.github.io/software.html}} from \cite{Condat13}. We set $\lambda = 0.0001$ to assess the performance of all compared algorithms. For these algorithms, the parameters are configured with $T = 20$ seconds, $K = 200$ iterations, and the initial iteration point is zero.

We apply all images from Figure \ref{details_images} to the optimization problem \eqref{tv-model} to evaluate the performance of the proposed algorithms. Detailed numerical results are presented in Table \ref{real_tv_table}.

\begin{table*}[thp]
\begin{center}
	\fontsize{6}{11}\selectfont
	\setlength{\tabcolsep}{1.2mm}
	\caption{The numerical results for the optimization problem \eqref{tv-model} with different blurred kernels and noises under six compared algorithms. 
	} 
\label{real_tv_table}
\begin{tabular}{c|c|c c | cc | cc| cc| cc| cc cc } 
	\hline 
	 & Images & \multicolumn{2}{c|}{Img01} & \multicolumn{2}{c|}{Img02} & \multicolumn{2}{c|}{Img03} &  \multicolumn{2}{c|}{Img04} &\multicolumn{2}{c|}{Img05} &\multicolumn{2}{c}{Img06}\\\cline{2-14}
	Algorithm & Kernels& 12 &(24,40) & 12 & (24,40) & 13 & (20,30)  & 7 & (14,30) & 8  & (24,34) & 7 & (10,20)  \\\hline
    \multicolumn{14}{c}{Noise $\varepsilon\sim \mathcal{N}(0,10^{-4})$} \\\hline
\multirow{3}{*}{ISTA} & Tol & 3.6e0 & 3.2e0 & 4.6e0 & 4.8e0  & 9.3e0 & 1.2e+1 & 1.6e+1 & 3.2e+1 & 1.1e+1 & 1.1e+1 & 9.1e-1 & 1.3e0  \\
	  & PSNR & 19.62 & 19.53 & 23.54 & 23.18  & 16.26 & 16.42 & 14.64 & 15.25 & 21.34 & 20.12 & 21.54 & 22.11 \\ 
      & SSIM & 0.2854 & 0.2804 & 0.6573 & 0.6512  & 0.1660 & 0.1924 & 0.2221 & 0.3177 & 0.5302 & 0.4712 & 0.3858 & 0.4352  \\\hline
\multirow{3}{*}{FISTA} & Tol & 2.9e-1 & 2.2e-1 & 1.2e-1 & 1.1e-1  & 8.5e-1 & 7.2e-1 & 9.4e-1 & 5.2e-1 & 5.0e-1 & 2.5e-1 & 6.0e-2 & 5.7e-2  \\
	  & PSNR & 21.02 & 20.63 & 27.53 & 26.75  & 19.71 & 18.55 & 18.05 & 21.82 & 23.89 & 22.62 & 24.00 & 24.59 \\ 
      & SSIM & 0.4406 & 0.3968 & 0.7568 & 0.7336  & 0.5985 & 0.5088 & 0.4940 & 0.6543 & 0.6737 & 0.5939 & 0.6165 & 0.6655   \\\hline
\multirow{3}{*}{OptISTA} & Tol & 2.0e-1 & 1.6e-1 & 6.0e-2 & 5.9e-2  & 3.6e-1 & 4.2e-1 & 5.1e-1 & 1.9e-1 & 2.8e-1 & 1.2e-1 & 4.5e-2 & 4.5e-2  \\
	  & PSNR & 21.32 & 20.85 & 28.58 & 27.63  & 20.82 & 19.19 & 19.35 & 23.91 & 24.67 & 23.10 & 24.43 & 24.91 \\ 
      & SSIM & 0.4748 & 0.4203 & 0.7824 & 0.7540  & 0.6832 & 0.5726 & 0.5670 & 0.7298 & 0.7133 & 0.6194 & 0.6476 & 0.6889  \\\hline
\multirow{3}{*}{IISTA} & Tol & 5.3e-1 & 1.7e0 & 3.6e-1 & 2.0e0  & 2.5e0 & 1.5e0 & 2.5e0 & 2.2e0 & 1.2e0 & 9.2e-1 & 1.1e-1 & 1.1e-1  \\
	  & PSNR & 20.66 & 18.63 & 26.25 & 20.60  & 18.18 & 17.90 & 16.87 & 19.20 & 23.05 & 21.84 & 23.42 & 24.10 \\ 
      & SSIM & 0.3987 & 0.3371 & 0.7308 & 0.5788  & 0.4554 & 0.4248 & 0.4183 & 0.5560 & 0.6328 & 0.5566 & 0.5739 & 0.6290  \\\hline
\multirow{3}{*}{EFISTA} & Tol & 4.3e-2 & 3.6e-2 & 8.8e-3 & 1.3e-2  & 3.6e-2 & 4.8e-2 & 4.4e-2 & 2.2e-2 & 3.8e-2 & 2.4e-2 & 8.5e-3 & 7.1e-3    \\
	  & PSNR & 22.98 & 21.91 & 31.22 & 29.81  & 24.04 & 21.79 & 24.27 & 26.89 & 27.83 & 24.36 & 26.86 & 26.87 \\ 
      & SSIM & 0.6441 & 0.5393 & 0.8364 & 0.7997  & 0.8505 & 0.7634 & 0.7468 & 0.8208 & 0.8223 & 0.6886 & 0.8022 & 0.8068  \\\hline
\multirow{3}{*}{IOptISTA} & Tol & \textbf{2.4e-2} & \textbf{2.0e-2} & \textbf{4.8e-3} & \textbf{6.0e-2}  & \textbf{1.6e-2} & \textbf{2.2e-2} & \textbf{1.7e-2} & \textbf{9.3e-3} & \textbf{1.7e-2} & \textbf{1.4e-2} & \textbf{4.4e-3} & \textbf{3.5e-3}  \\
	  & PSNR & \textbf{23.80} & \textbf{22.38} & \textbf{32.23} & \textbf{30.88}  & \textbf{25.24} & \textbf{22.86} & \textbf{26.35} & \textbf{27.89} & \textbf{29.34} & \textbf{24.98} & \textbf{28.03} & \textbf{27.75} \\ 
      & SSIM & \textbf{0.7067} & \textbf{0.5869} & \textbf{0.8538} & \textbf{0.8210}  & \textbf{0.8856} & \textbf{0.8118} & \textbf{0.8040} & \textbf{0.8447} & \textbf{0.8587} & \textbf{0.7183} & \textbf{0.8490} & \textbf{0.8429} \\\hline
\multicolumn{14}{c}{Noise $\varepsilon\sim \mathcal{N}(0,10^{-3})$} \\\hline
\multirow{3}{*}{ISTA} & Tol & 3.8e0 &3.4e0 & 4.7e0 & 4.9e0  & 9.5e0 & 1.2e+1 & 1.6e+1 & 3.2e+1 & 1.2e+1 & 1.1e+1 & 9.4e-1 & 1.3e0    \\
	  & PSNR & 19.62 & 19.53 & 23.54 & 23.18  & 16.26 & 16.42 & 14.64 & 15.25 & 21.34 & 20.12 & 21.54 & 22.11  \\ 
      & SSIM & 0.2854 & 0.2803 & 0.6573 & 0.6512  & 0.1660 & 0.1923 & 0.2221 & 0.3177 & 0.5302 & 0.4712 & 0.3858 & 0.4352   \\\hline
\multirow{3}{*}{FISTA} & Tol & 4.1e-1 & 3.5e-1 & 2.4e-1 & 2.3e-1  & 9.8e-1 & 8.4e-1 & 1.0e0 & 5.8e-1 & 6.5e-1 & 4.0e-1 & 8.7e-2 & 8.3e-2    \\
	  & PSNR & 21.01 & 20.62 & 27.50 & 26.73  & 19.71 & 18.55 & 18.04 & 21.81 & 23.88 & 22.61 & 23.98 & 24.56  \\ 
      & SSIM & 0.4399 &0.3958 & 0.7545 & 0.7311  & 0.5982 & 0.5085 & 0.4938 & 0.6537 & 0.6722 & 0.5921 & 0.6146 & 0.6636  \\\hline
    \multirow{3}{*}{OptISTA} & Tol & 3.2e-1 & 2.8e-1 & 1.8e-1 & 1.8e-1  & 4.8e-1 & 5.4e-1 & 5.7e-1 & 2.5e-1 & 4.3e-1 & 2.7e-1 & 7.1e-2  & 6.9e-2    \\
	  & PSNR & 21.31 & 20.84 & 28.52 & 27.57  & 20.80 & 19.18 & 19.34 & 23.87 & 24.65 & 23.09 & 24.39 & 24.86  \\ 
      & SSIM & 0.4734 & 0.4184 & 0.7774 & 0.7481  & 0.6823 & 0.5717 & 0.5662 & 0.7281 & 0.7099 & 0.6160 & 0.6440 & 0.6854   \\\hline
\multirow{3}{*}{IISTA} & Tol & 6.5e-1 &1.8e0 & 4.9e-1 & 2.2e0  & 2.6e0 & 1.6e0 & 2.6e0 & 2.3e0 & 1.4e0 & 1.1e0 & 1.4e-1 & 1.3e-1    \\
	  & PSNR & 20.65 &18.62 & 26.24 & 20.60  & 18.17 & 17.89 & 16.87 & 19.20 & 23.05 & 21.84 & 23.41 & 24.09  \\ 
      & SSIM & 0.3984 & 0.3367 & 0.7298 & 0.5788  & 0.4553 & 0.4247 & 0.4182 & 0.5558 & 0.6322 & 0.5562 & 0.5728 & 0.6280   \\\hline
\multirow{3}{*}{EFISTA} & Tol & 1.4e-1 &1.5e-1 & 1.2e-1 & 1.3e-1  & 1.4e-1 & 1.5e-1 & 8.6e-2 & 7.5e-2 & 1.5e-1 & 1.7e-1 & 2.8e-2 & 2.6e-2    \\
	  & PSNR & 22.80 &21.81 & 30.26 & 29.23  & 23.70 & 21.66 & 23.81 & 26.38 & 27.37 & 24.21 & 26.18 & 26.34  \\ 
      & SSIM & 0.6246 &0.5251 & 0.7880 & 0.7571  & 0.8364 & 0.7546 & 0.7300 & 0.8037 & 0.7863 & 0.6617 & 0.7663 & 0.7789   \\\hline
\multirow{3}{*}{IOptISTA} & Tol & \textbf{1.2e-1} &\textbf{1.3e-1} & \textbf{1.0e-2} & \textbf{1.2e-1}  & \textbf{1.1e-1} & \textbf{1.2e-1} & \textbf{5.2e-2} & \textbf{5.8e-2} & \textbf{1.2e-1} & \textbf{1.5e-1} & \textbf{2.0e-2} & \textbf{2.0e-2}    \\
	  & PSNR & \textbf{23.27} &\textbf{22.11} & \textbf{30.43} & \textbf{29.40}  & \textbf{24.24} & \textbf{22.50} & \textbf{24.95} & \textbf{26.58} & \textbf{27.90} & \textbf{24.62} & \textbf{26.53} & \textbf{26.59}  \\ 
      & SSIM & \textbf{0.6572} &\textbf{0.5552} & \textbf{0.7912} & \textbf{0.7515}  & \textbf{0.8528} & \textbf{0.7914} & \textbf{0.7646} & \textbf{0.8086} & \textbf{0.7916} & \textbf{0.6643} & \textbf{0.7830} & \textbf{0.7895}  \\\hline
\end{tabular}
\end{center}
\end{table*}

The table further corroborates the conclusions drawn in Subsection \ref{l1-regularization}: the IOptISTA algorithm consistently attains lower error rates under identical stopping conditions when compared to other methods. Additionally, IOptISTA typically achieves higher PSNR and SSIM values, see Figures \ref{tps_results_tv_img02} and \ref{tps_results_tv_img05} for details. 

\begin{figure}[!ht]
\setlength\tabcolsep{2pt}
\centering
\begin{tabular}{ccc} 
\includegraphics[width=0.32\textwidth, height=3.5cm]{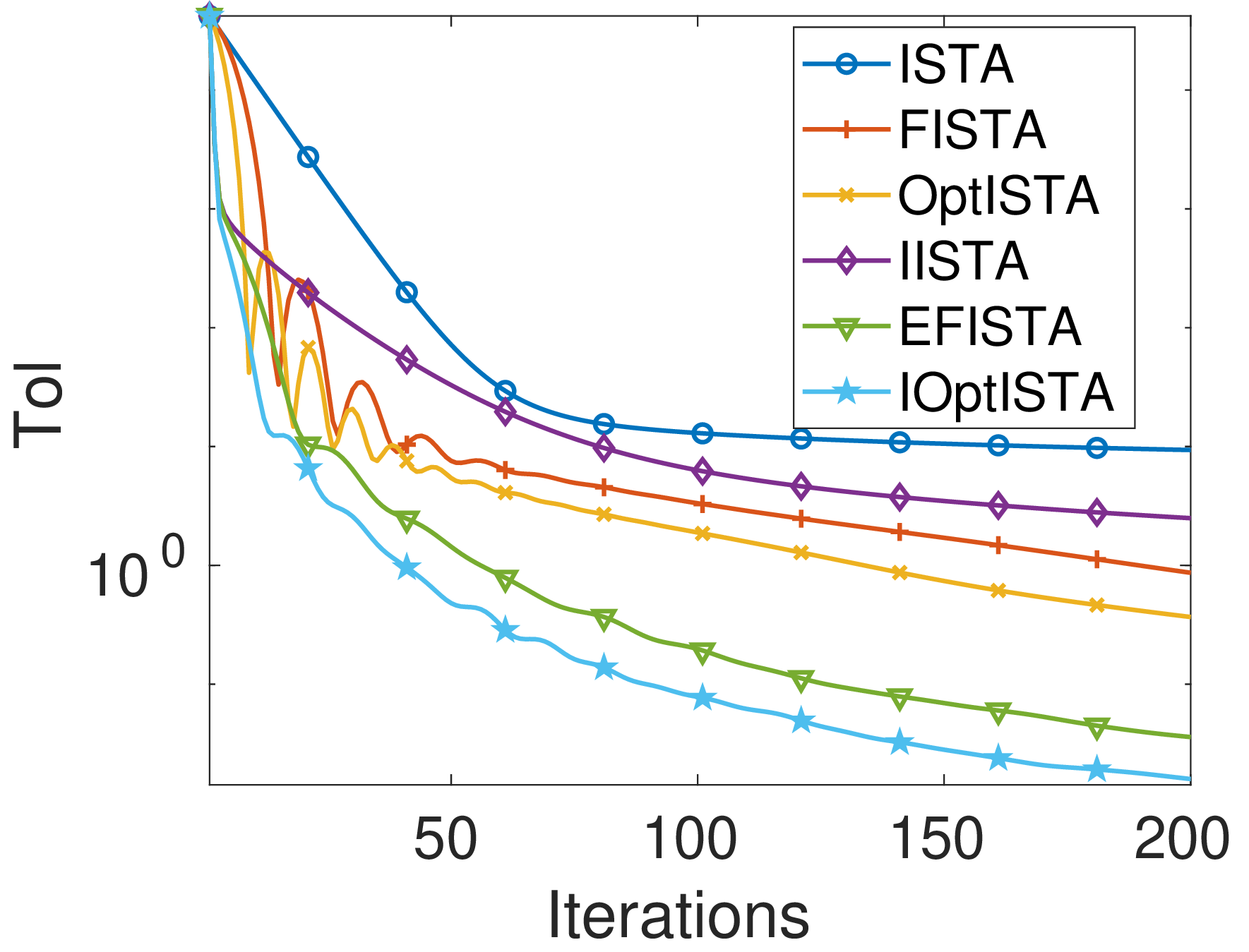} & \includegraphics[width=0.32\textwidth, height=3.5cm]{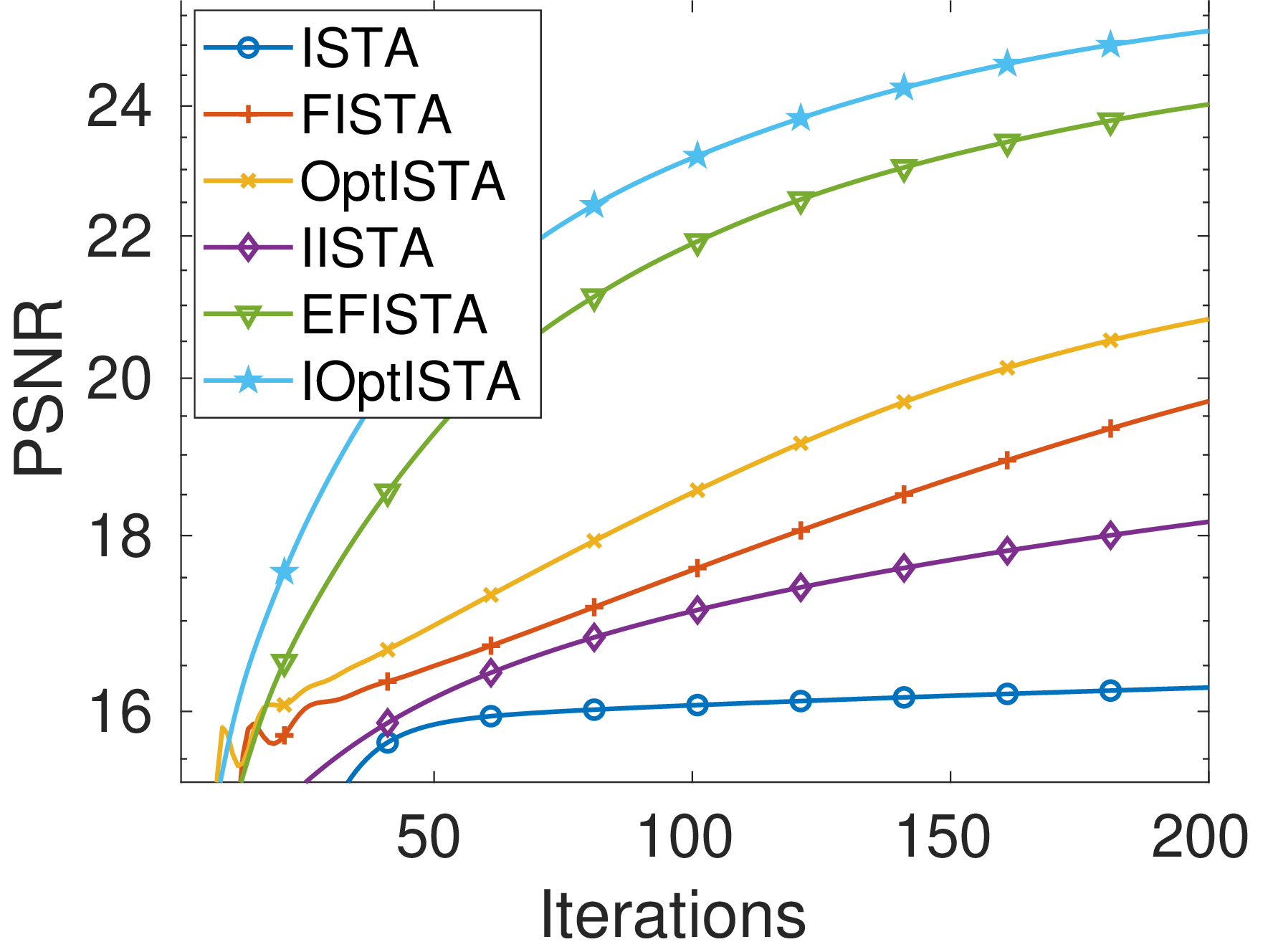}& \includegraphics[width=0.32\textwidth, height=3.5cm]{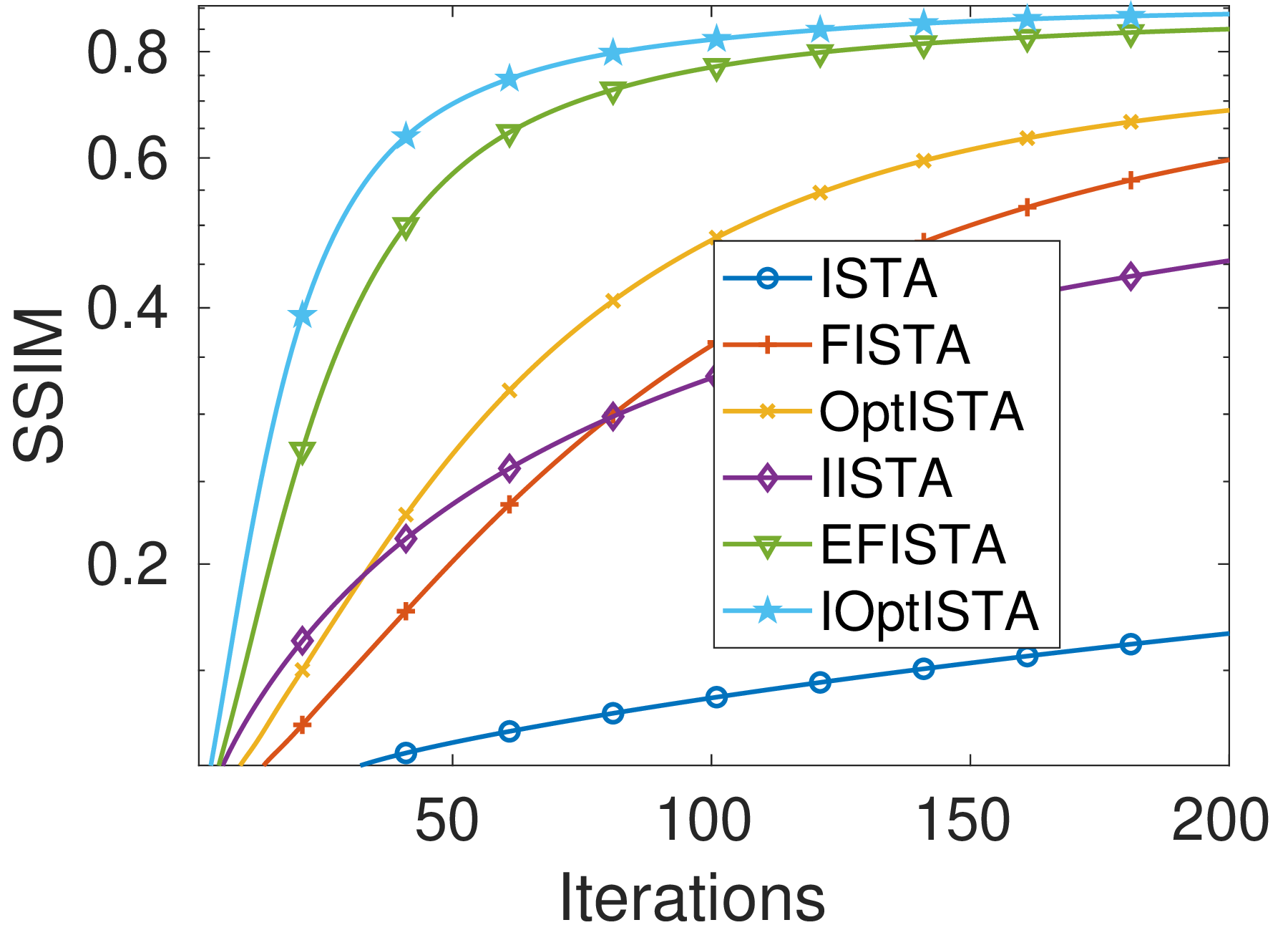}\\
	(a) Tol & (b) PNSR & (c) SSIM
\end{tabular}
\caption{The Tol, PSNR and SSIM results for Figure \ref{details_images}(c)  with ``$k = \text{fspecial}(\text{`disk'}, 13)$'' and $\varepsilon\sim\mathcal{N}(0, 1e-4)$.}
\label{tps_results_tv_img02}
\end{figure}
\begin{figure}[!ht]
\setlength\tabcolsep{2pt}
\centering
\begin{tabular}{ccc} 
\includegraphics[width=0.32\textwidth, height=3.5cm]{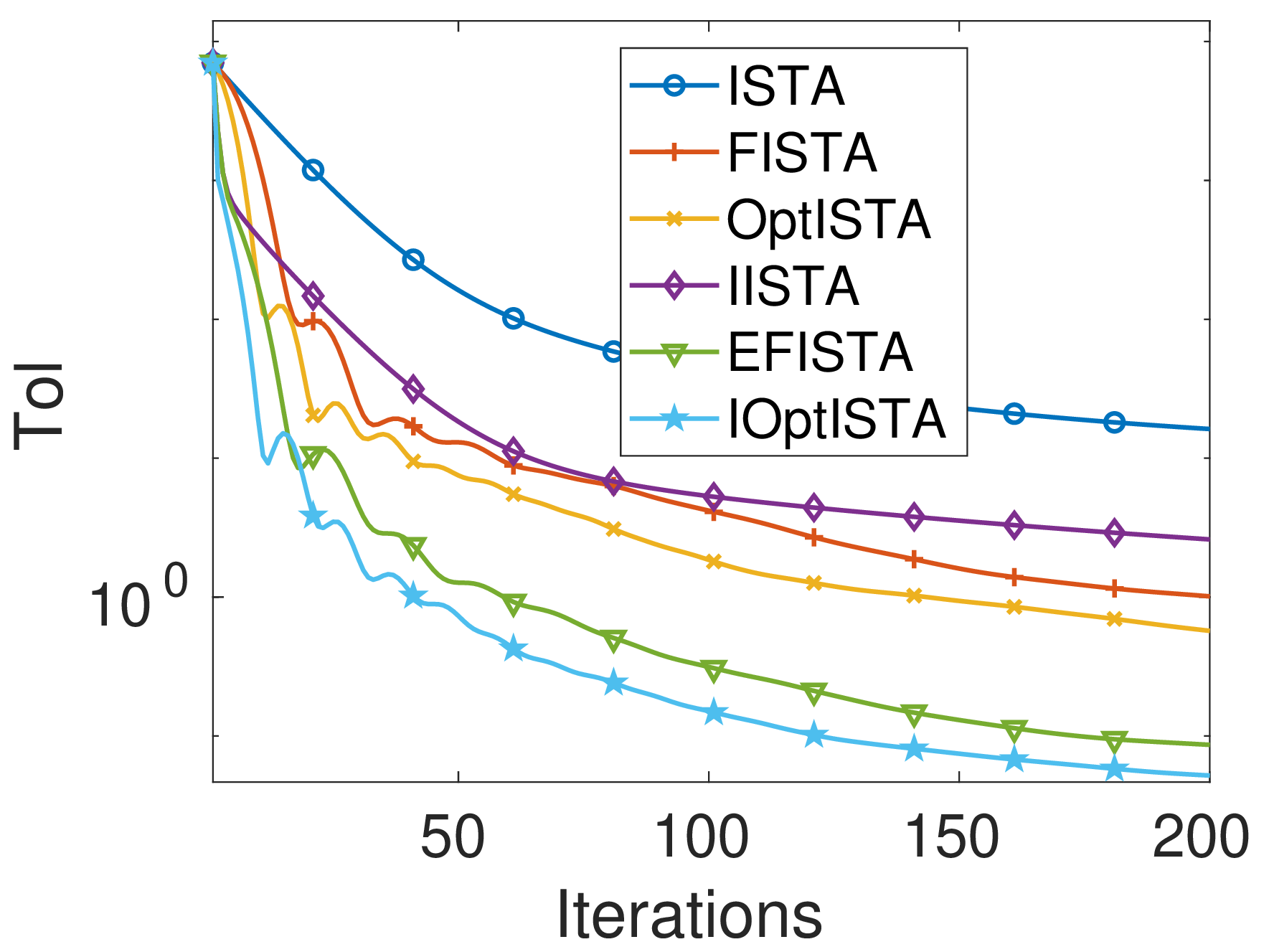} & \includegraphics[width=0.32\textwidth, height=3.5cm]{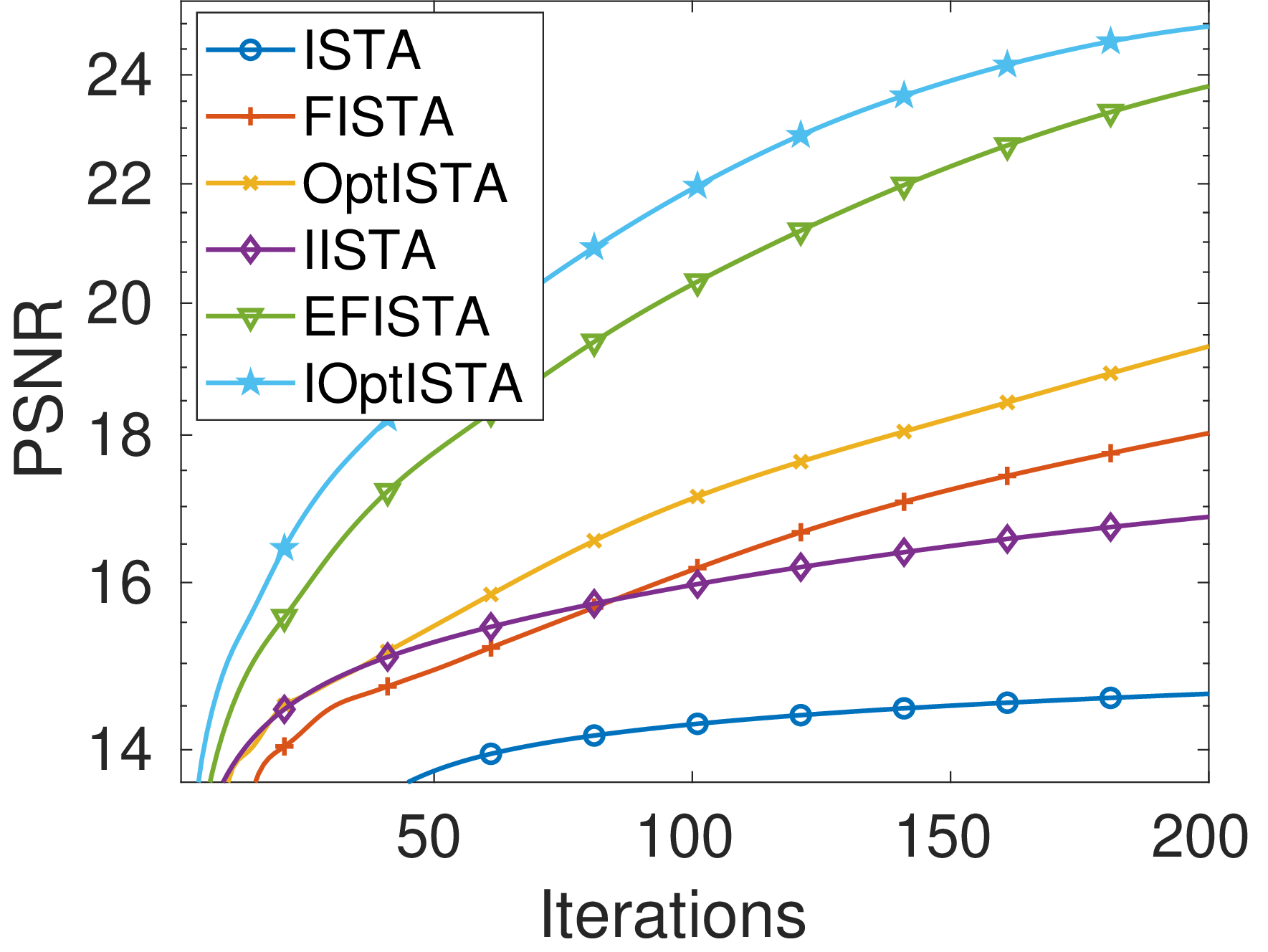}& \includegraphics[width=0.32\textwidth, height=3.5cm]{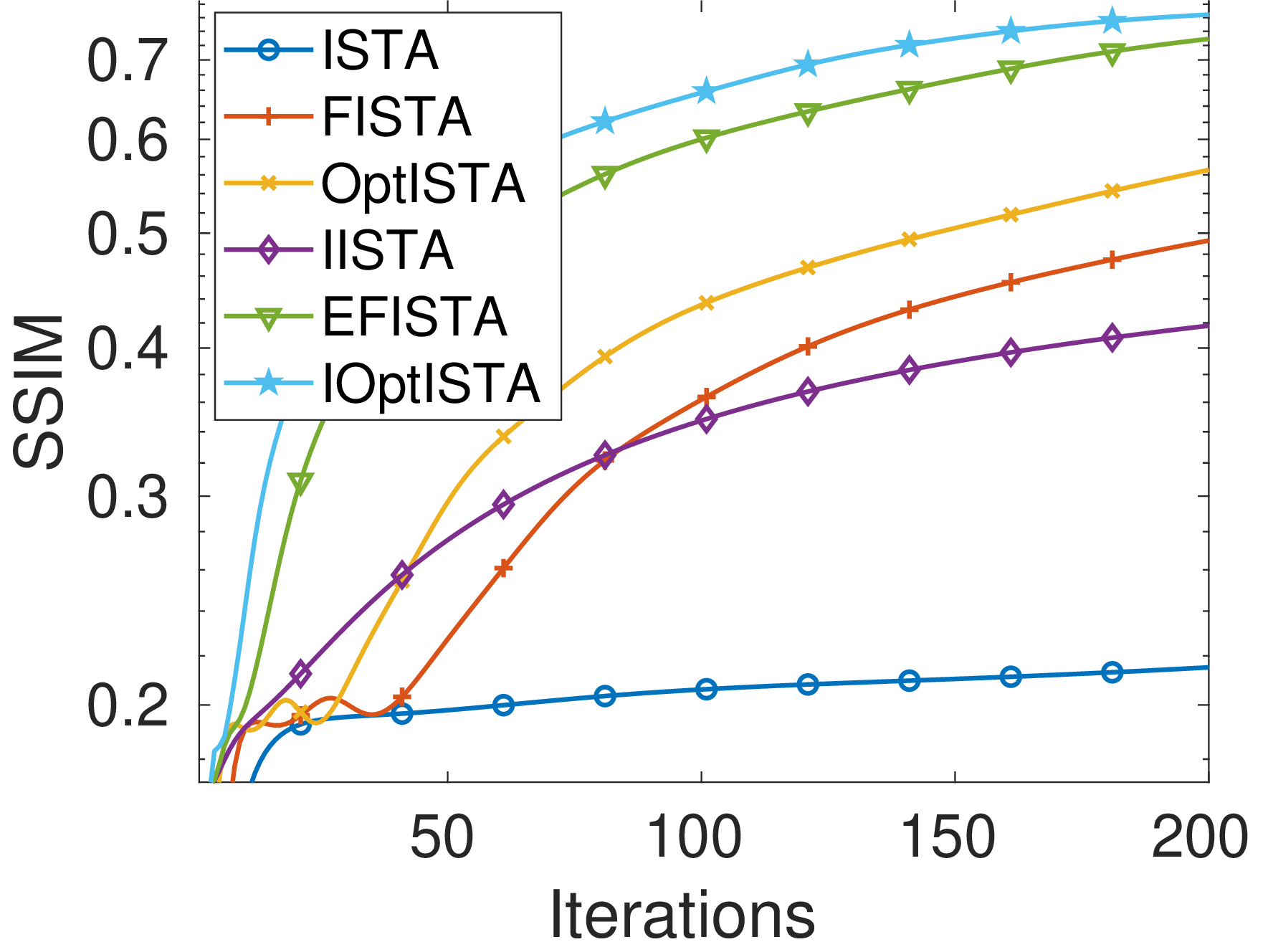}\\
	(d) Tol & (e) PNSR & (f) SSIM
\end{tabular}
\caption{The Tol, PSNR and SSIM results for Figure \ref{details_images}(d)  with ``$k = \text{fspecial}(\text{`disk'}, 7)$'' and $\varepsilon\sim\mathcal{N}(0, 1e-3)$.} 
\label{tps_results_tv_img05}
\end{figure}
\begin{figure}[!ht]
\setlength\tabcolsep{2pt}
\centering
\begin{tabular}{cccccc} 
\includegraphics[width=0.23\textwidth]{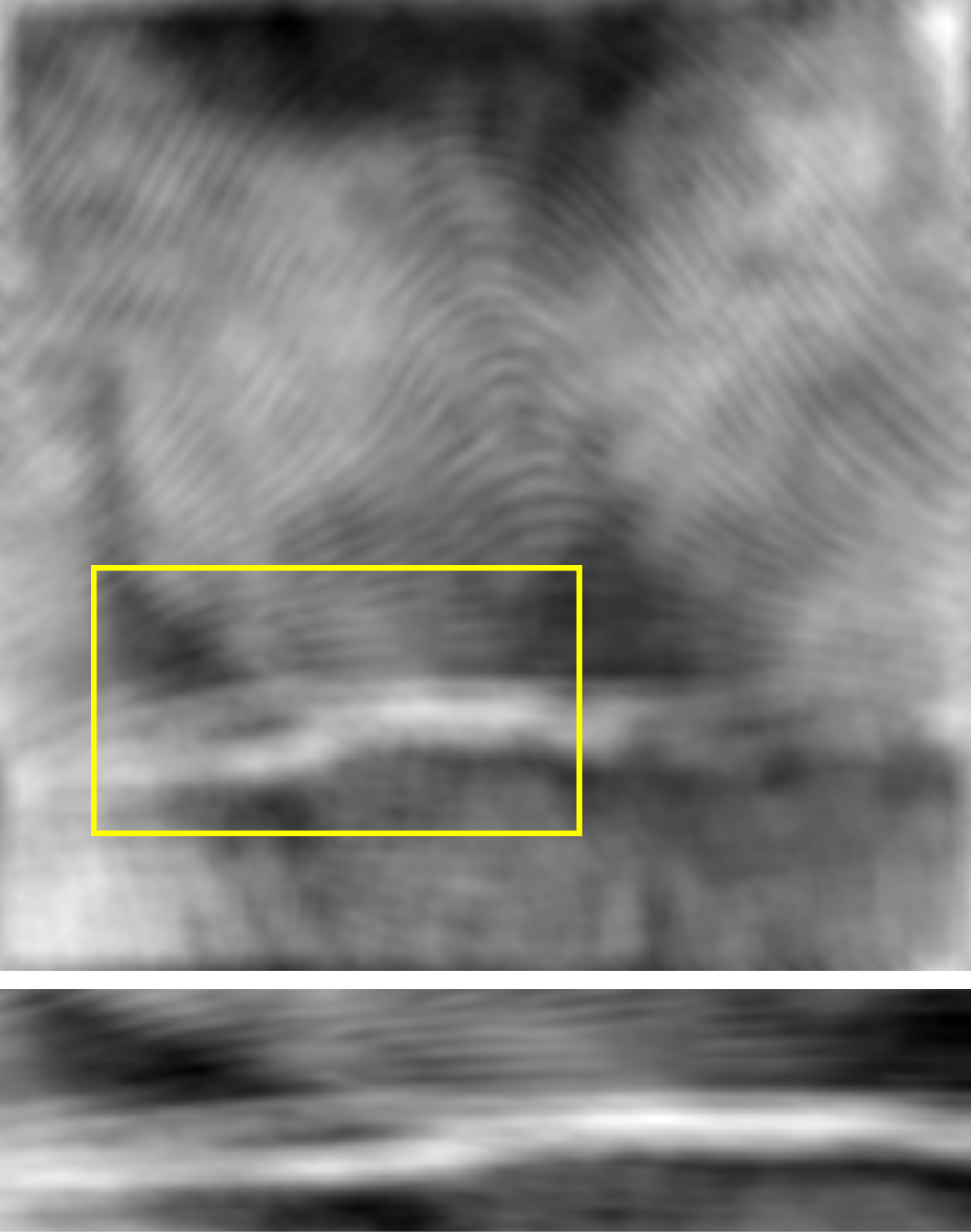} & \includegraphics[width=0.23\textwidth]{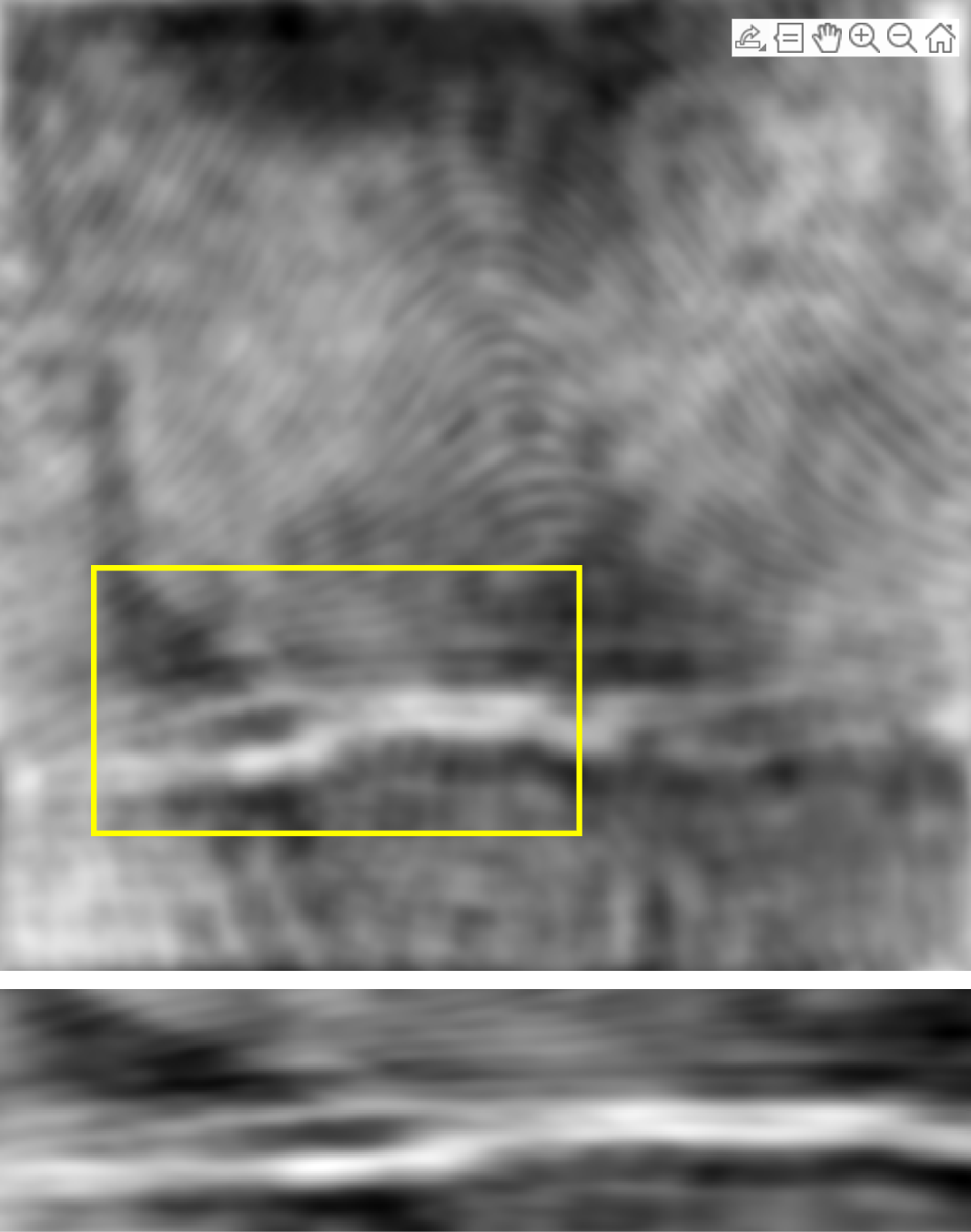}&
\includegraphics[width=0.23\textwidth]{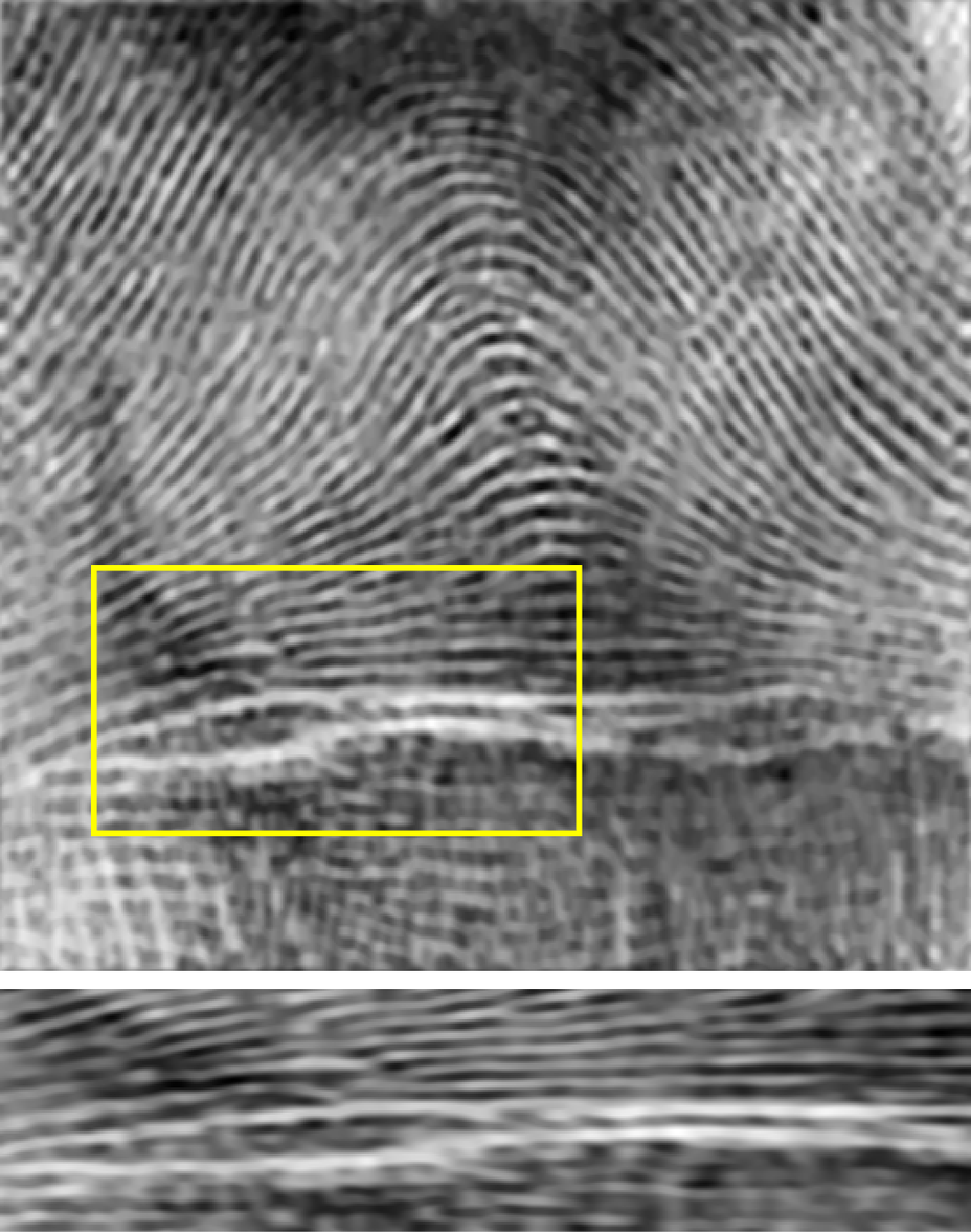}&
\includegraphics[width=0.23\textwidth]{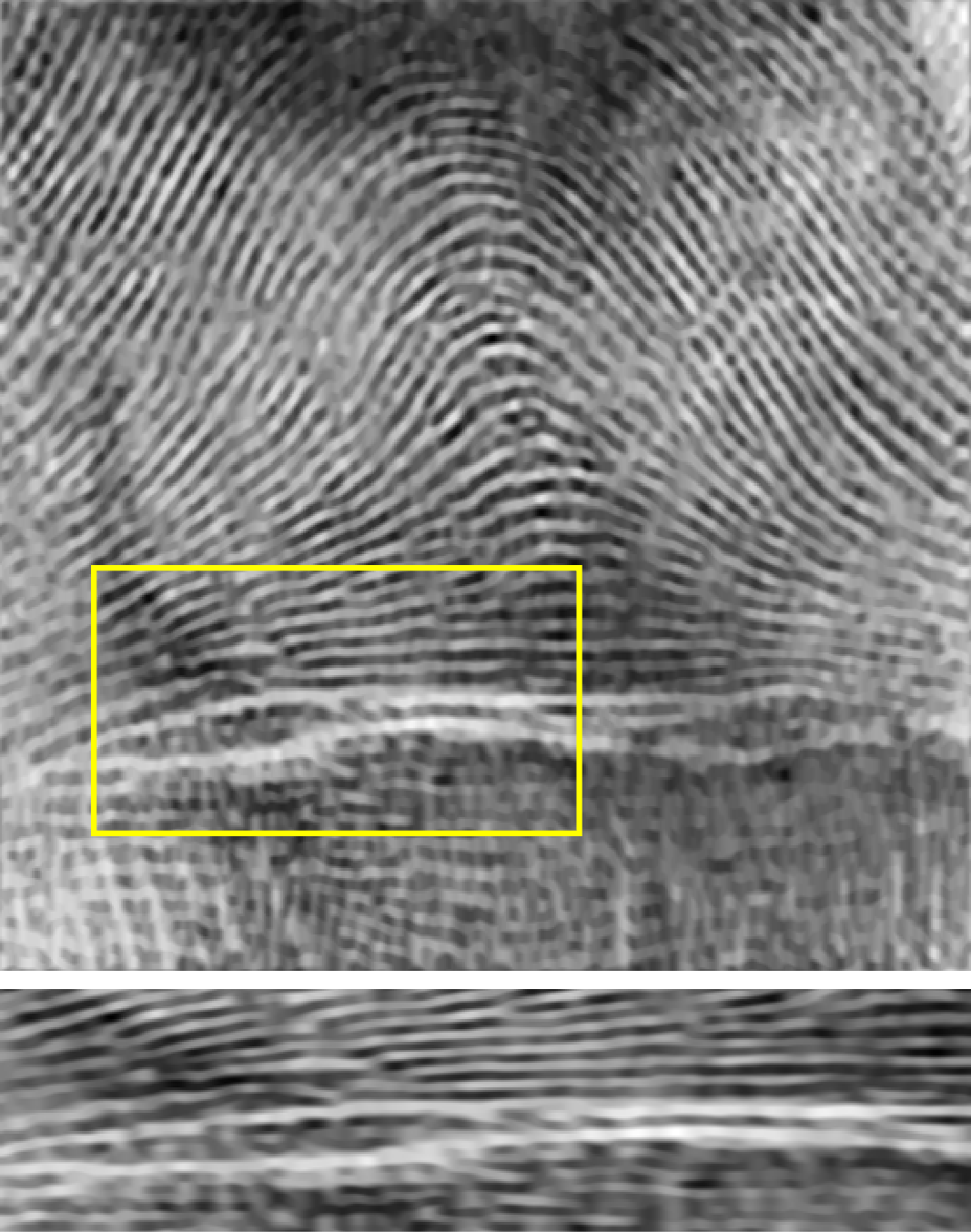}\\
(a) Noised & (b) ISTA & (c) FISTA&(d) OptISTA \\
\includegraphics[width=0.23\textwidth]{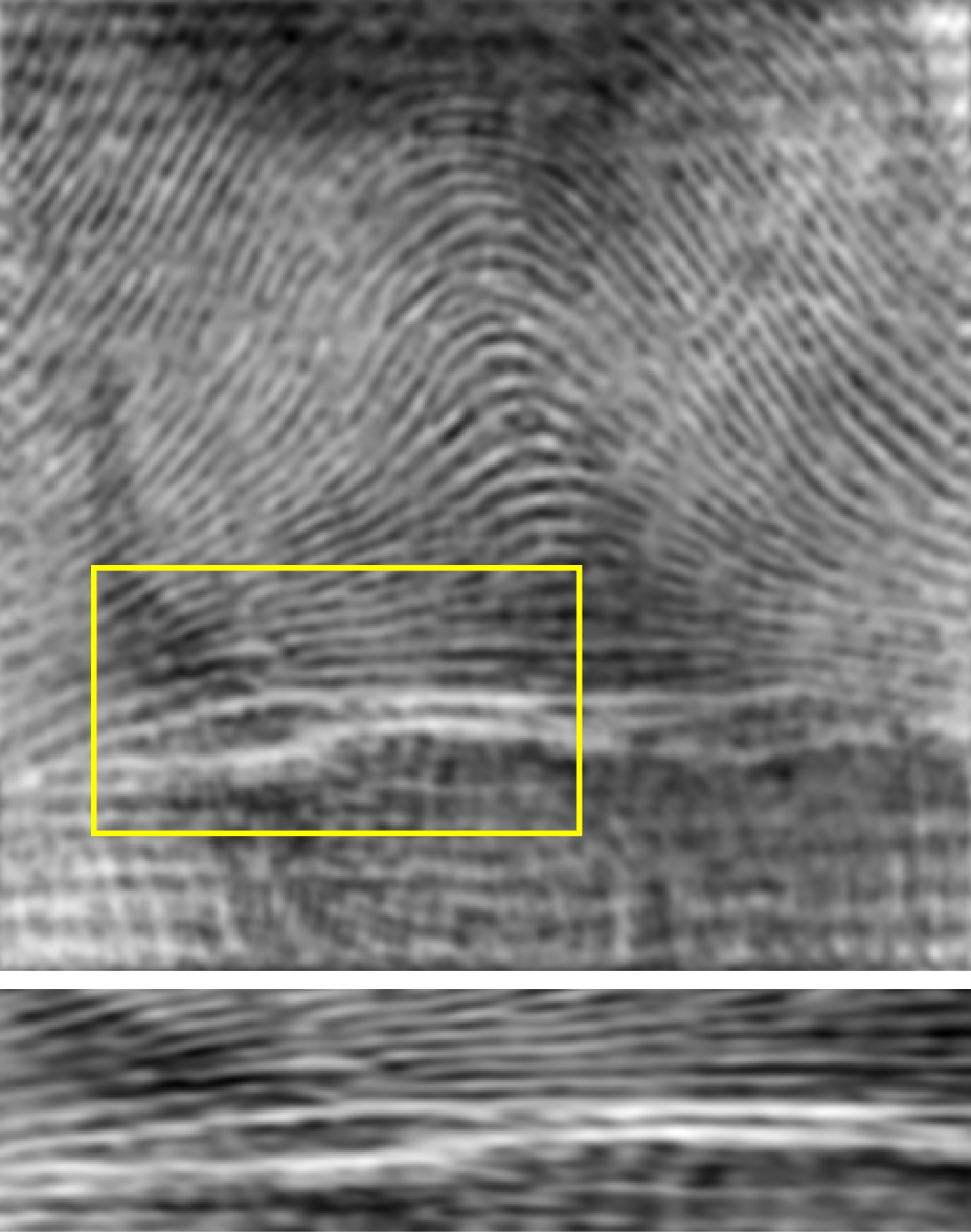} & \includegraphics[width=0.23\textwidth]{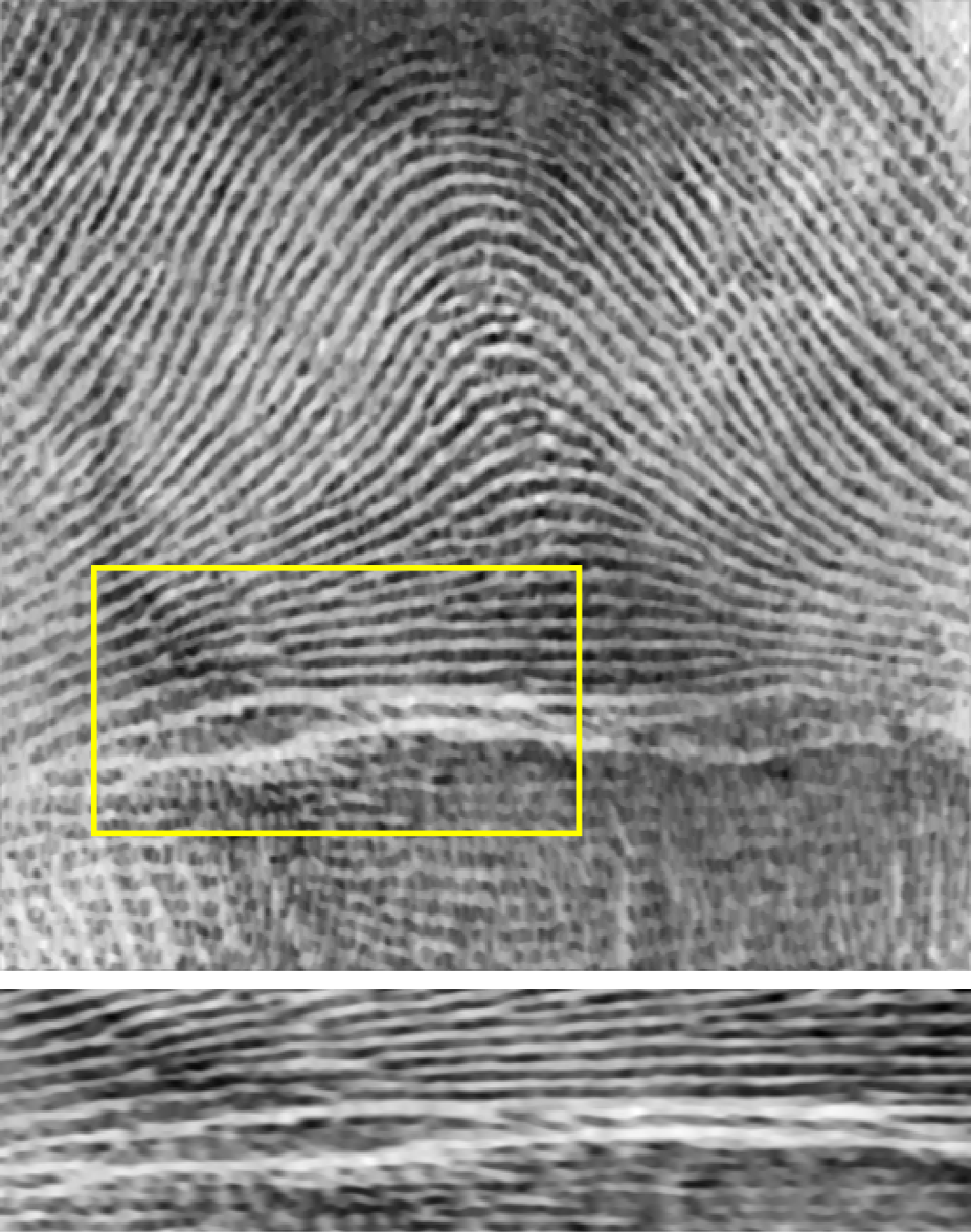}&
\includegraphics[width=0.23\textwidth]{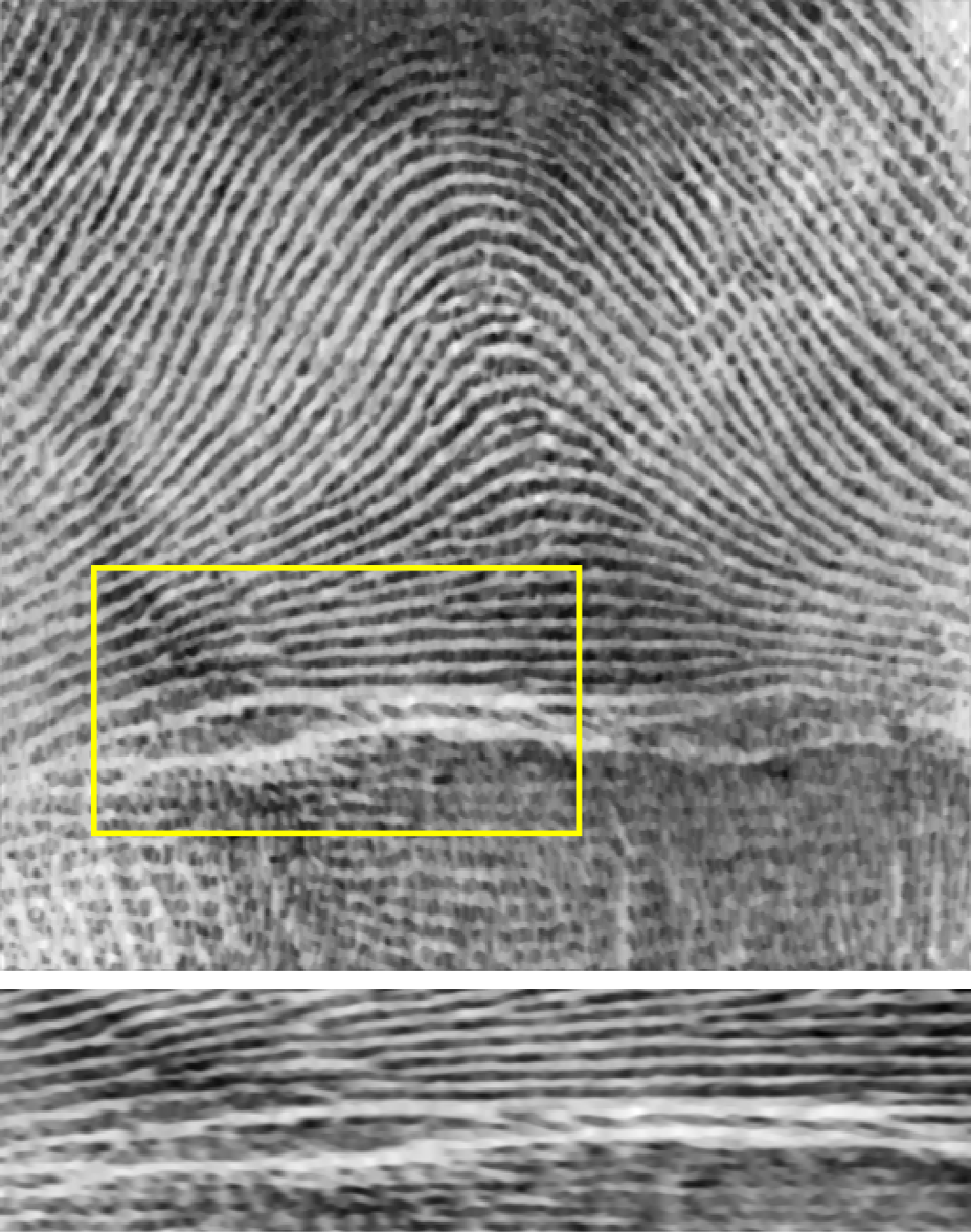}&
\includegraphics[width=0.23\textwidth]{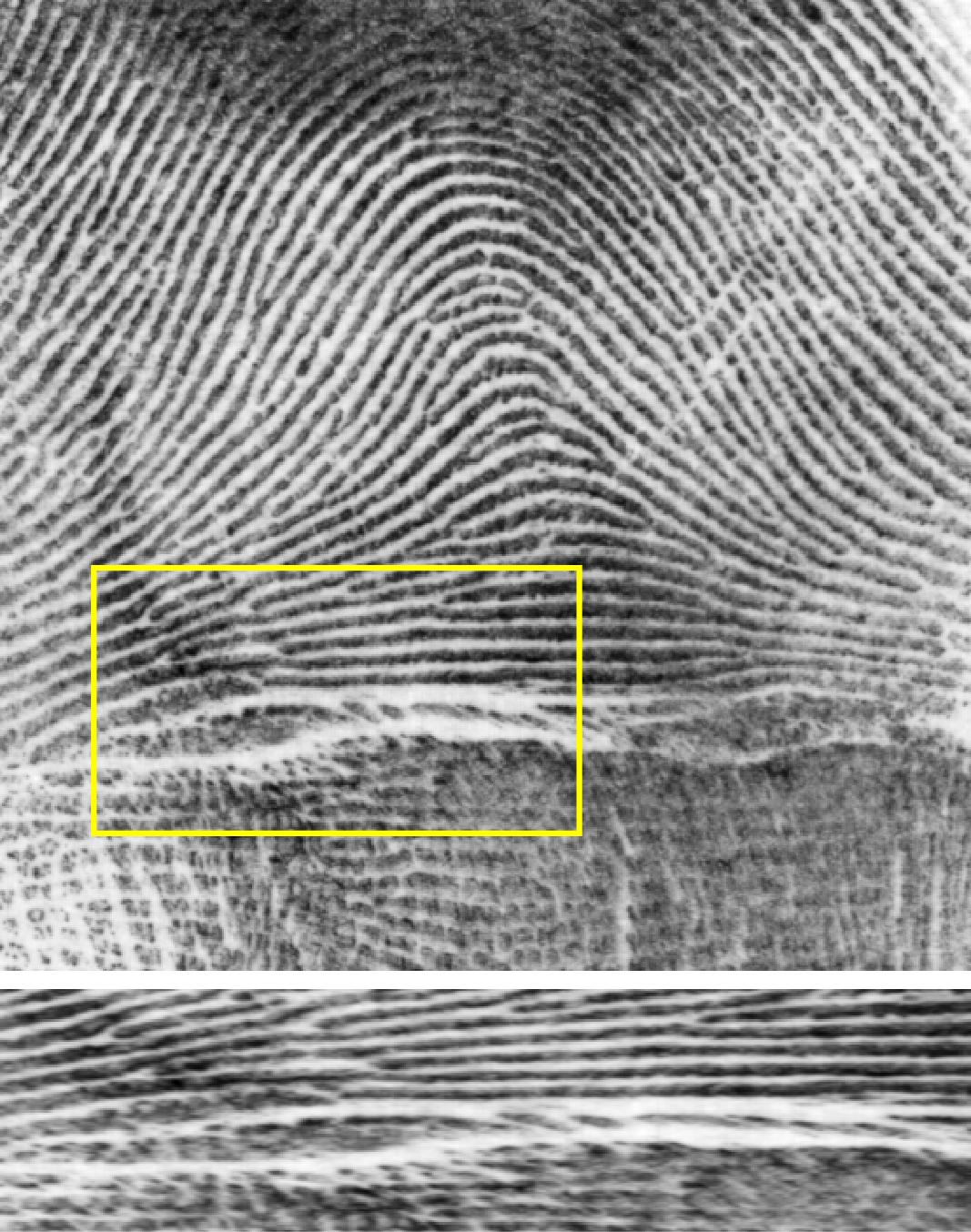}\\
(e) IISTA & (f) EFISTA & (g) IOptISTA & (h) Original 
\end{tabular}
\caption{Blurred and restored images for the Figure \ref{details_images}(c)  with ``$k = \text{fspecial}(\text{`disk'}, 13)$'' and $\varepsilon\sim\mathcal{N}(0, 1e-4)$.} 
\label{results_img01_TV}
\end{figure}
\begin{figure}[!ht]
\setlength\tabcolsep{2pt}
\centering
\begin{tabular}{cccccc} 
\includegraphics[width=0.23\textwidth]{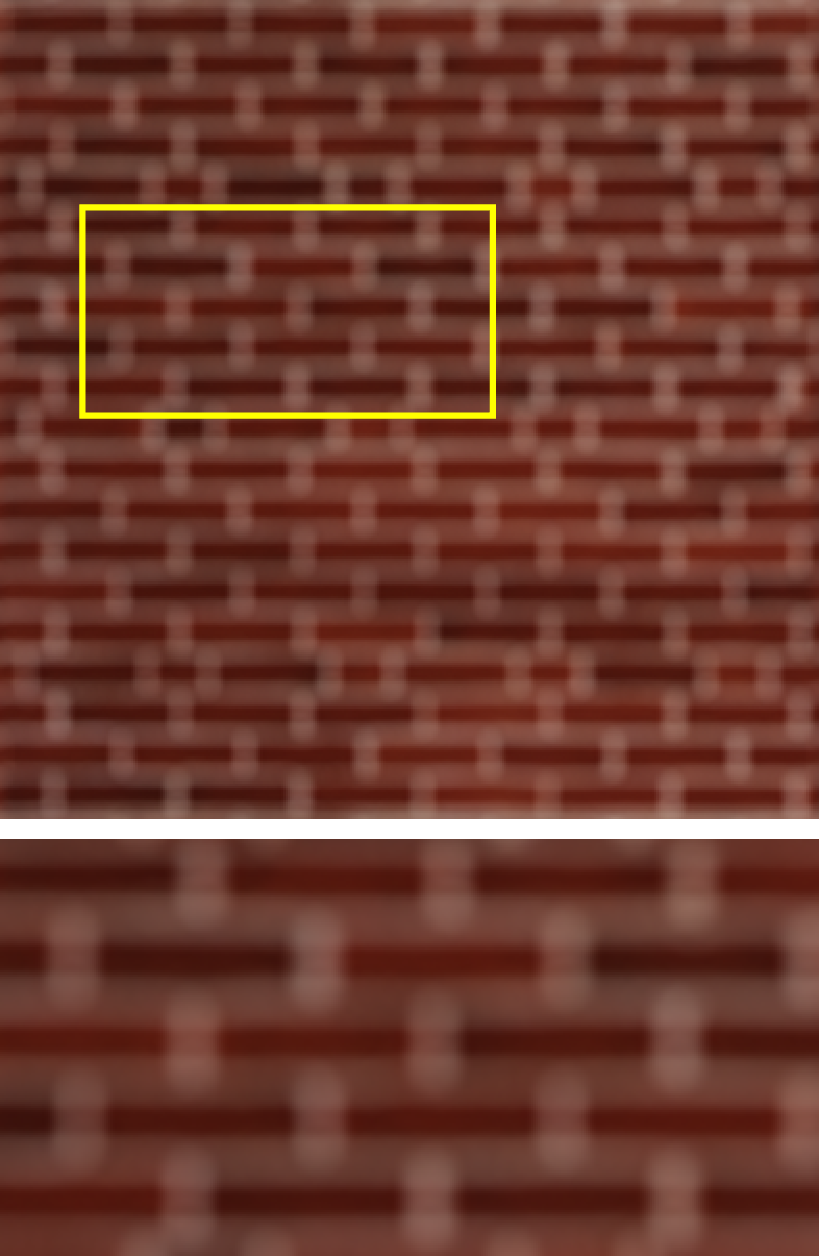} & \includegraphics[width=0.23\textwidth]{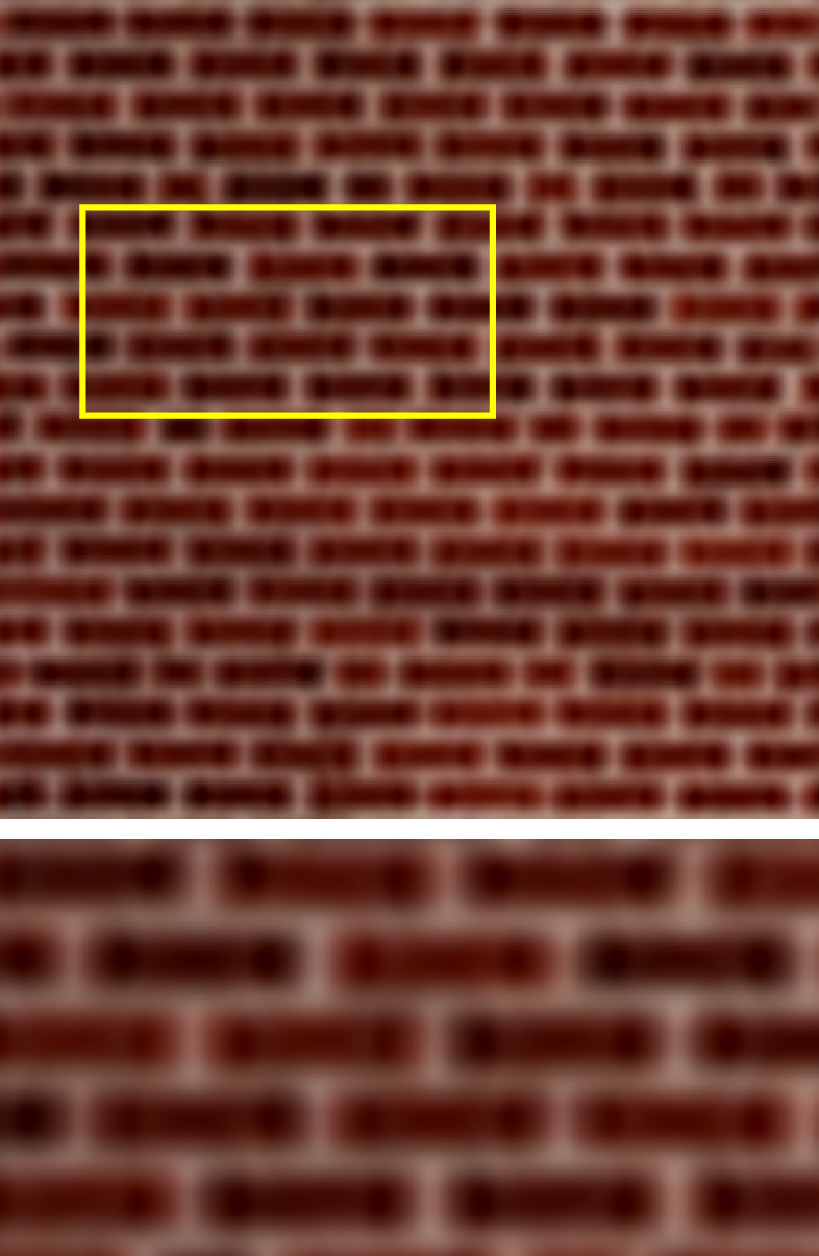}&
\includegraphics[width=0.23\textwidth]{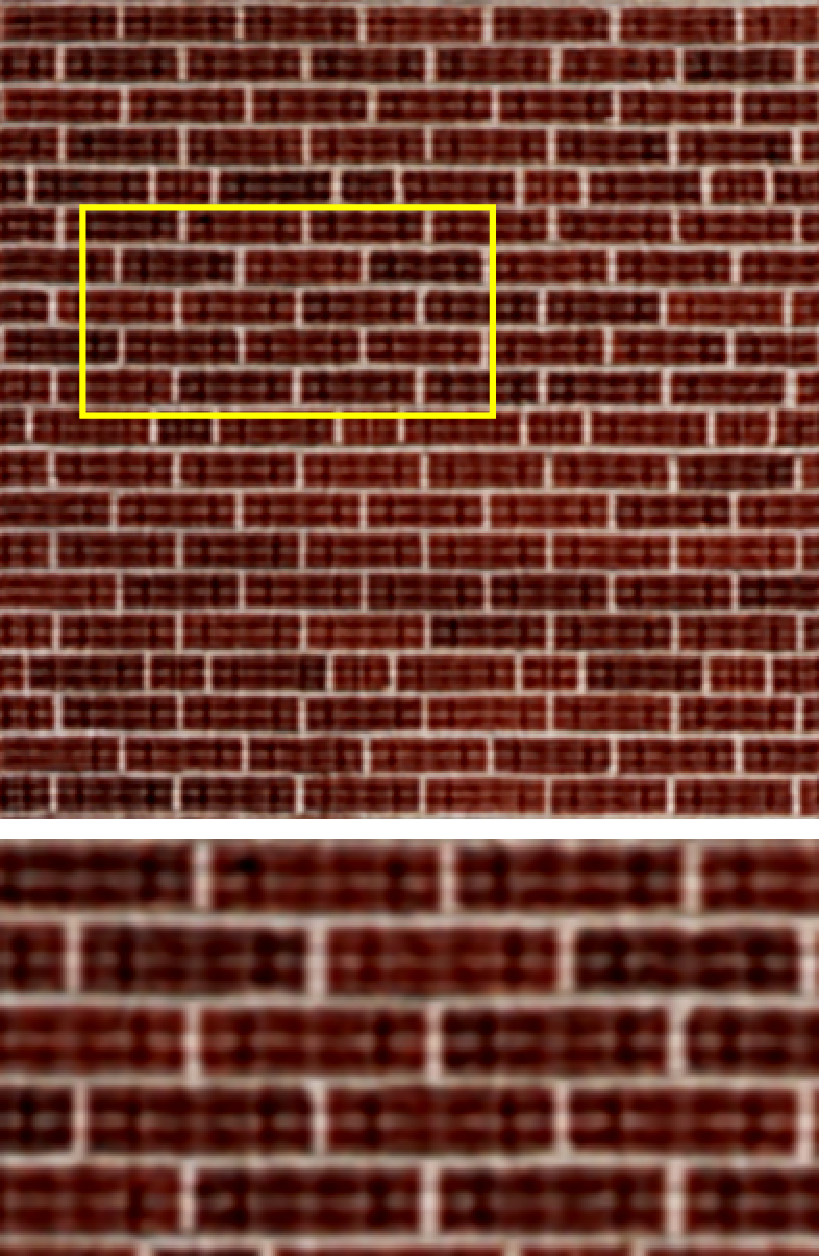}&
\includegraphics[width=0.23\textwidth]{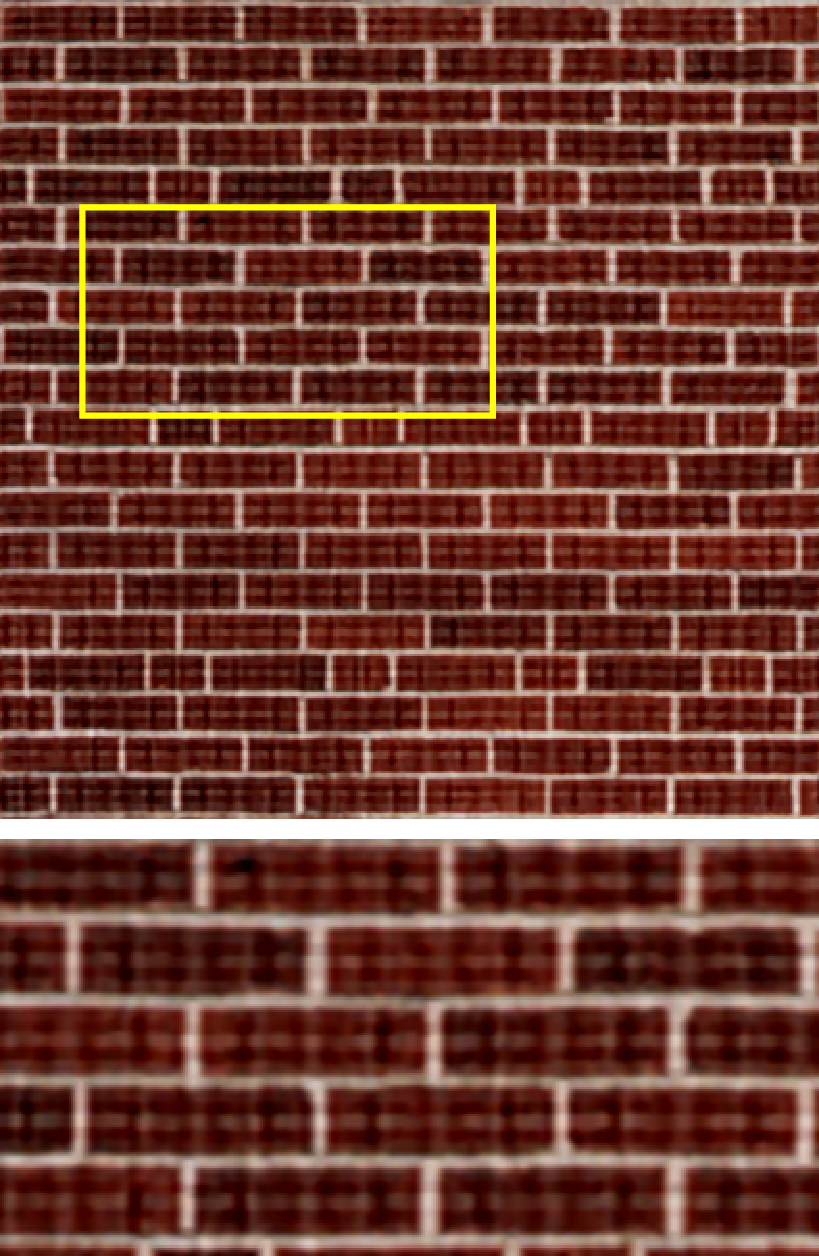}\\
(a) Noised & (b) ISTA & (c) FISTA&(d) OptISTA \\
\includegraphics[width=0.23\textwidth]{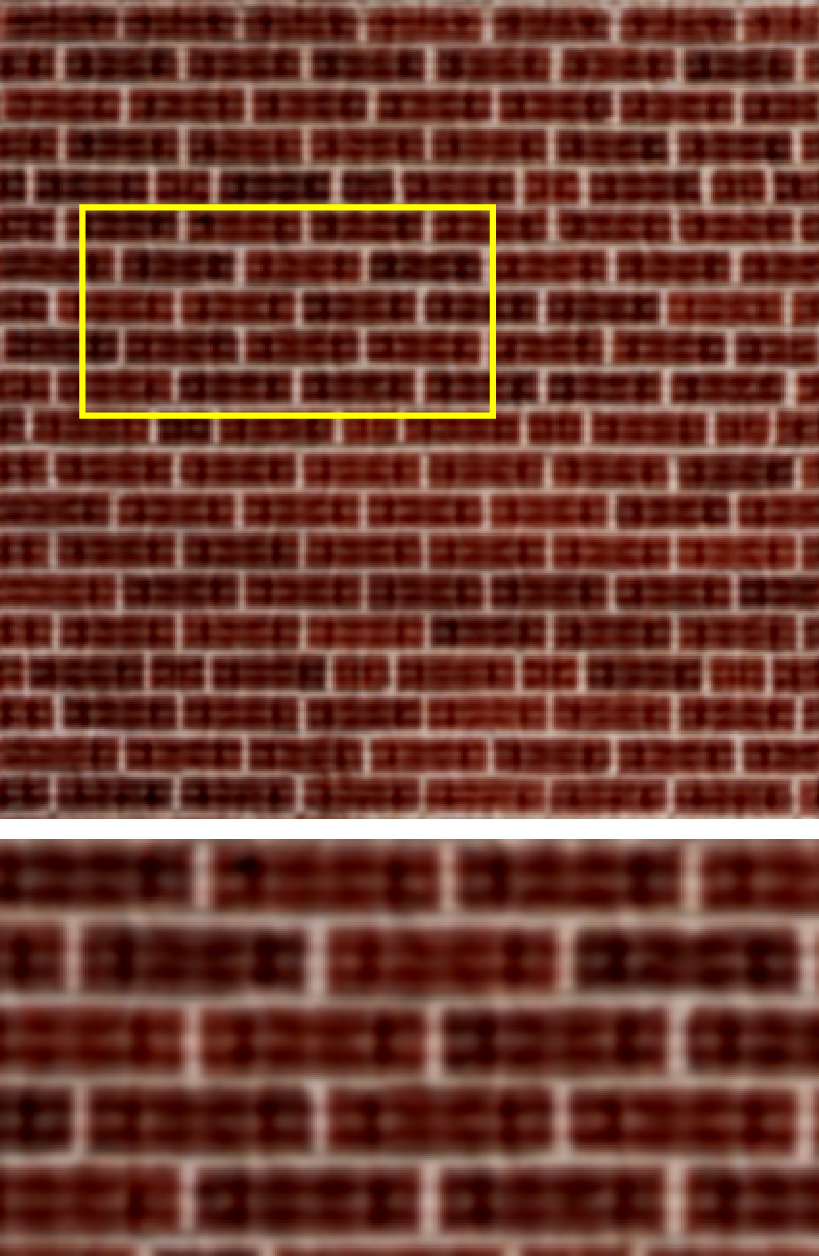} & \includegraphics[width=0.23\textwidth]{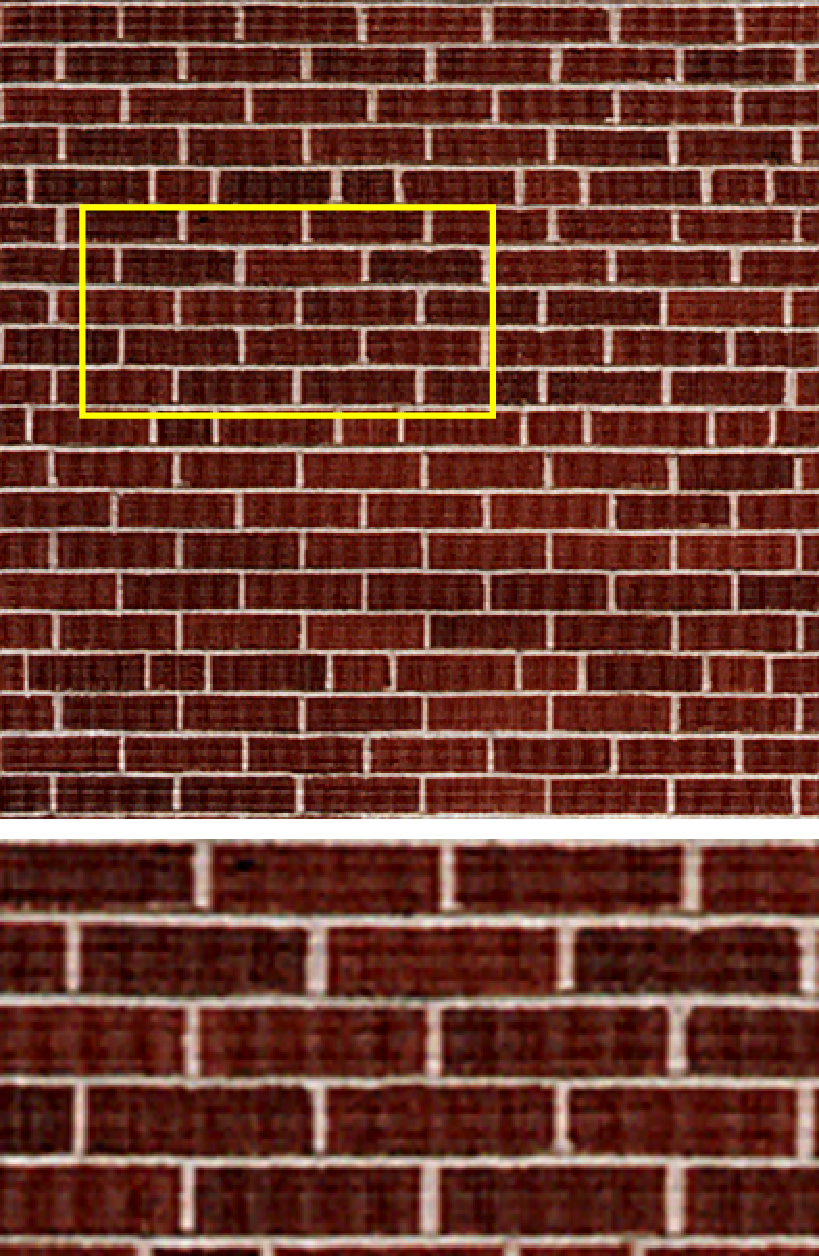}&
\includegraphics[width=0.23\textwidth]{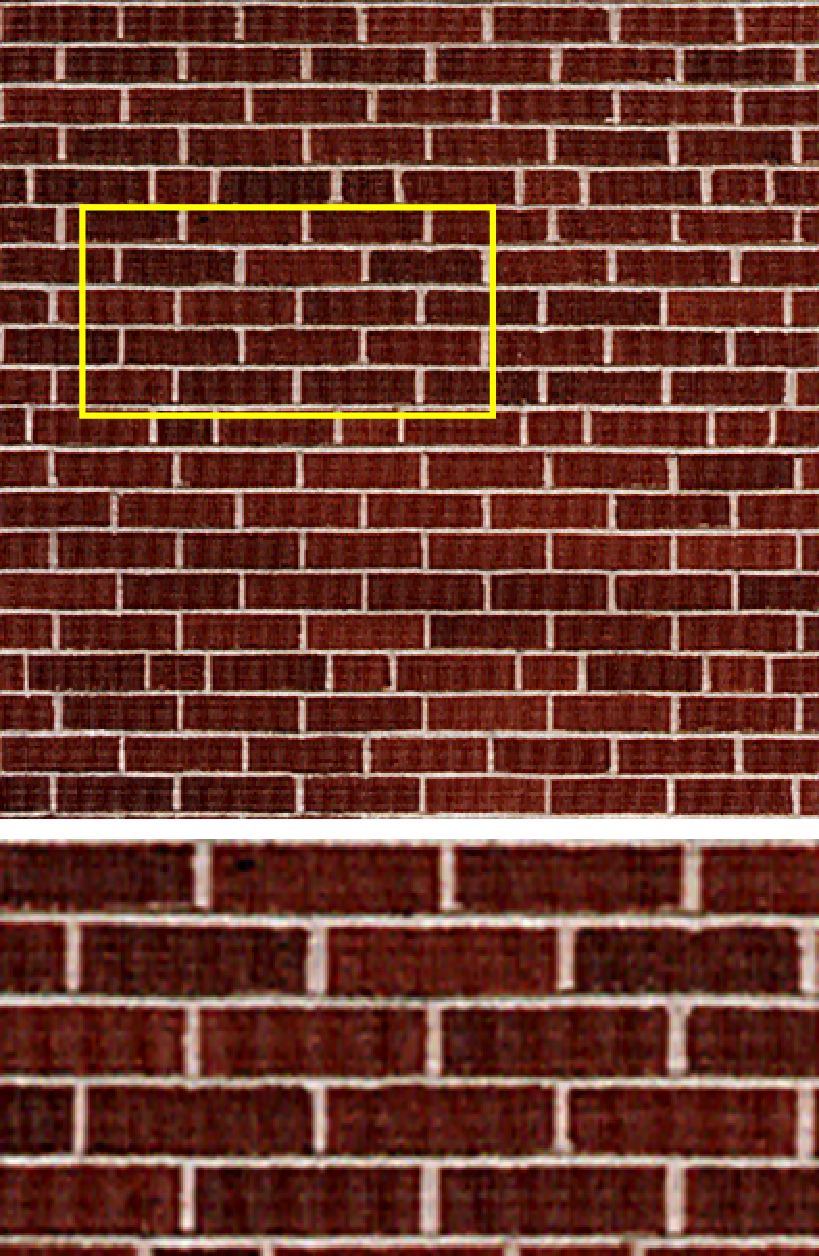}&
\includegraphics[width=0.23\textwidth]{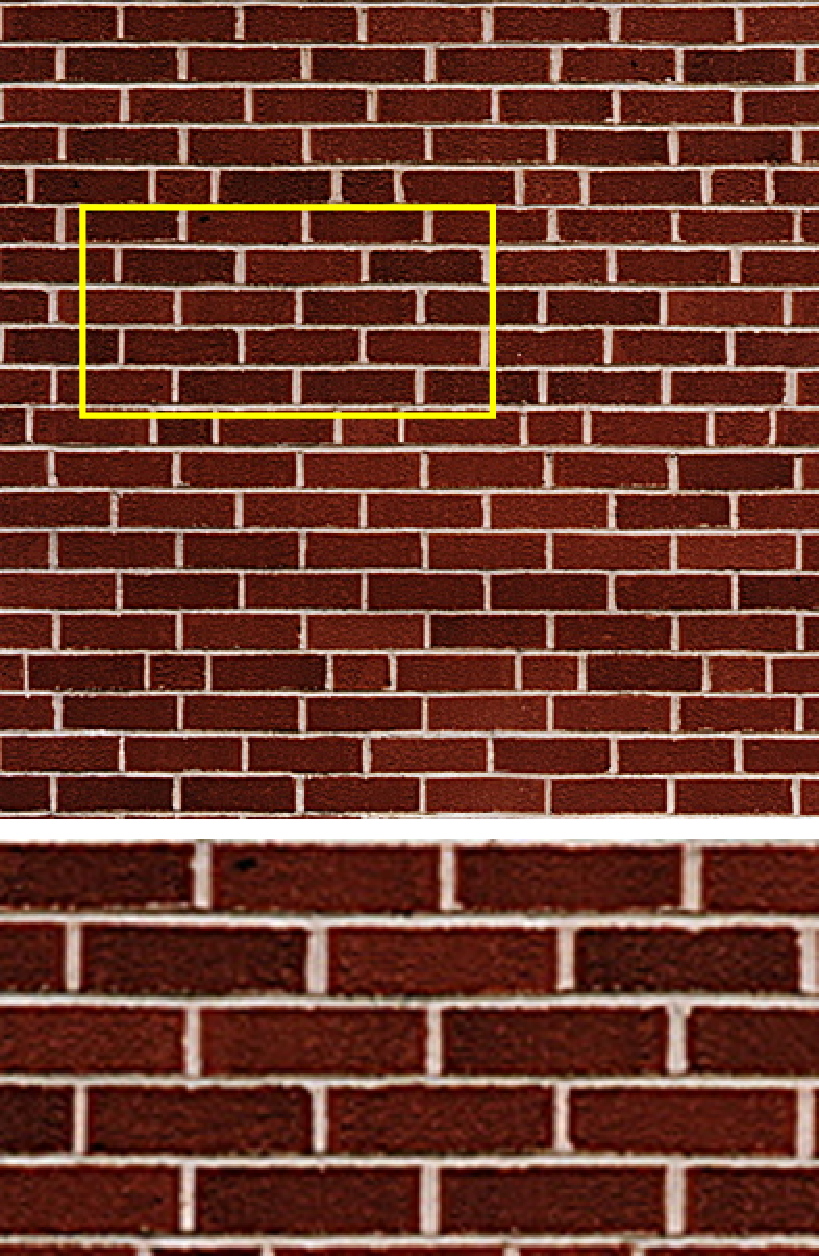}\\
(e) IISTA & (f) EFISTA & (g) IOptISTA & (h) Original 
\end{tabular}
\caption{Blurred and restored images for the Figure \ref{details_images}(d)  with ``$k = \text{fspecial}(\text{`disk'}, 7)$'' and $\varepsilon\sim\mathcal{N}(0, 1e-3)$.} 
\label{results_img04_TV}
\end{figure}

In summary, a comprehensive suite of numerical experiments has substantiated the effectiveness of the proposed algorithms. This paper establishes that integrating the weighting matrix technique with the optimal gradient method markedly enhances convergence and elevates numerical performance.

\section{Conclusion and future works}\label{conclusion-part}
This paper addressed the non-blind image deblurring problem by proposing an improved optimal proximal gradient algorithm for its solution. Specifically, we developed the IOptISTA and MOptISTA algorithms, whose superior convergence speed has been demonstrated through numerical experiments. These experiments, employing $l_1$ norm and TV norm regularizations, confirmed the efficiency of our proposed methods, achieving lower error rates and higher PSNR and SSIM values compared to existing methods, thus showcasing their superior restoration performance.

In the future, we aim to conduct a comprehensive theoretical convergence analysis of the proposed algorithms to complement the experimental findings. Additionally, we plan to explore the application of these methods to nonconvex optimization problems in image and video deblurring, further extending their utility in computational imaging.

\bigskip
\section*{Declarations}
\noindent{\bf Conflict of Interest:} On behalf of all authors, the corresponding author states that there is no conflict of interest. 

\noindent{\bf Funding:} {This research is supported by the National Natural Science Foundation of China (NSFC) grant 12171145, 12331011, 12401415, 12471282.}

\noindent{\bf Data  Availability Statement:} Data will be available upon reasonable request. 

\newpage
\appendix
\section{Proof of Lemma \ref{lemma1}} \label{pf_lemma1}
Before to prove Lemma \ref{lemma1}, we need the following lemmas.
\begin{lemma} \label{lemma2}
(Lemma 16 in \cite{JangGR23})
Let $\{t_{i,j}\}_{i\in[1:K],j\in[0:K-1]}$ and $\{s_{i,j}\}_{i\in[1:K],j\in[0:K-1]}$ be
\[
t_{i+1,j}=\begin{cases}
t_{j+1,j}+\sum_{k=j+1}^{i}\left(\frac{2\alpha_j}{\alpha_{k+1}}-\frac{1}{\alpha_{k+1}}t_{k,j}
\right),\quad &\textup{ if } j\in[0:i-1], \\
1+\frac{2\alpha_i-1}{\alpha_{i+1}},\quad &\textup{ if } j=i,\end{cases}
\]
\[
s_{i+1,j}=\begin{cases}
   t_{i+1,j}-t_{i,j},\quad   &\textup{ if } j\in[0:i-1],\\
   t_{i+1,j},\quad  &\textup{ if } j=i.
\end{cases}
\]
Then,
 \[
 t_{i+1,j}=\sum_{k=j+1}^{i+1} s_{k,j} \textup{ for } j\in[0:i]
\textup{ and }
s_{i+1,j}=\begin{cases}
    \frac{\alpha_i-1}{\alpha_{i+1}}s_{i,j},& \textup{ if } j\in[0:i-2],\\
    \frac{\alpha_i-1}{\alpha_{i+1}}\left(s_{i,i-1}-1\right),& \textup{ if } j=i-1,\\
    1+\frac{2\alpha_i-1}{\alpha_{i+1}}, &\textup{ if }  j=i.
\end{cases}
\]
\end{lemma}
Now we convert $x$ and $y$-iterates of IOptISTA (Algorithm \ref{IOptISTA}) into the following form.
\begin{lemma}\label{lemma3}
Let
\[
\gamma_i=\frac{2\alpha_i}{\alpha_K^2}\left(\alpha_K^2-2\alpha_i^2+\alpha_i\right) \textup{ for } i\in[0:K-1],
\]
\[
t_{i+1,j}=\begin{cases} t_{j+1,j}+\sum_{k=j+1}^{i}\left(\frac{2\alpha_j}{\alpha_{k+1}}-\frac{1}{\alpha_{k+1}}t_{k,j}\right), &\textup{ if } j\in[0:i-1],
\\ 1+\frac{2\alpha_i-1}{\alpha_{i+1}}, &\textup{ if } j=i.\end{cases}
\]
Denote $h'(y_{j+1}):=\frac{1}{\eta_{j}\gamma_{j}}(y_{j}-\eta_{j}\gamma_{j}W_{n}\nabla f(x_{j})-y_{j+1})\in \partial h(y_{j+1})$,  then,  $\{x_1,\ldots,x_{K}\}$ and  $\{y_1,\ldots,y_K\}$ of the following  form:
\begin{equation}
    \begin{aligned}
    y_{i+1}&=x_0-\sum_{j\in[0:i]}\gamma_j\eta_{j} W_{n}\nabla f(x_j)- \sum_{j\in[0:i]}\gamma_j\eta_{j}h'(y_{j+1}) \quad\textup{ for } i\in[0:K-1],\\
     x_{i+1}&=x_0-\sum_{j\in[0:i]}t_{i+1,j}\eta_{j}W_{n}\nabla f(x_j)- \sum_{j\in[0:i]}t_{i+1,j}\eta_{j}h'(y_{j+1})
     \quad \textup{ for } i\in[0:K-1]
    \end{aligned}
        \label{equivalent_xy}
\end{equation}
is equal to $x$-iterates and $y$-iterates generated by Algorithm \ref{IOptISTA} respectively.
\end{lemma}
\begin{proof}
Denote $\{\hat{x}_{1},\dots, \hat{x}_{K}\}$ and $\{\hat{y}_{1},\dots, \hat{y}_{K}\}$ to be sequence generated by Algorithm \ref{IOptISTA}, it shows that $x_{0}=y_{0}$ and $\hat{y}_{1}=y_{1}$. We can get that
\[
\hat{y}_{1}=\text{Prox}_{\gamma_{0}\eta_{0}h}(y_{0}-\gamma_{0}\eta_{0}W_{n}\nabla f(x_{0}))=x_{0}-\gamma_{0}\eta_{0}W_{n}\nabla f(x_{0})-\gamma_{0}\eta_{0}h'(\hat{y}_{1}).
\]
For $x$-iterate, we have
\begin{align*}
\hat{x}_{1}=&x_{1}+\frac{\alpha_{0}-1}{\alpha_{1}}(z_{1}-z_{0})+\frac{\alpha_{0}}{\alpha_{1}}(z_{1}-z_{0})\\
=&x_{0}+\frac{1}{\gamma_{0}}(\hat{y}_{1}-y_{0})+\frac{1}{\alpha_{1}}(x_{0}+\frac{1}{\gamma_{0}}(\hat{y}_{1}-y_{0})-x_{0})\\
=&x_{0}-\frac{1}{\gamma_{0}}\gamma_{0}\eta_{0}(W_{n}\nabla f(x_{0})+h'(\hat{y}_{1})) -\frac{1}{\alpha_{1}}\frac{1}{\gamma_{0}}\gamma_{0}\eta_{0}(W_{n}\nabla f(x_{0})+h'(\hat{y}_{1}))\\
=& x_{0}-\eta_{0}(1+\frac{1}{\alpha_{1}})W_{n}\nabla f(x_{0}) -\eta_{0}(1+\frac{1}{\alpha_{1}})h'(\hat{y}_{1})\\
=&x_{0}-t_{1,0}\eta_{0}W_{n}\nabla f(x_{0}) -t_{1,0}\eta_{0}h'(\hat{y}_{1})\\
=&x_{1}.
\end{align*}
Now we suppose $x_{k}=\hat{x}_{k}$ and $y_{k}=\hat{y}_{k}$ for $1\le k\le i$ and $1\le i\le K-1$. Then we have $\hat{y}_{i+1}=y_{i+1}$ from the uniqueness of the proximal operator, and 
\begin{align*}
\hat{y}_{i+1}=&\text{Prox}_{\gamma_{i}\eta_{i} h}(\hat{y}_{i}-\gamma_{i}\eta_{i}W_{n}\nabla f(x_{i})\\
=&y_{i}-\gamma_{i}\eta_{i}W_{n}\nabla f(x_{i}) -\gamma_{i}\eta_{i}h'(\hat{y}_{i+1})\\
=&x_{0}-\sum_{j\in[0:i-1]}\gamma_{i}\eta_{i}W_{n}\nabla f(x_{j}) -\sum_{j\in[0:i-1]}\gamma_{j}\eta_{j}h'(y_{j+1})-\gamma_{i}\eta_{i}W_{n}\nabla f(x_{i}) -\gamma_{i}\eta_{i}h'(\hat{y}_{i+1})\\
=&x_{0}-\sum_{j\in[0:i]}\gamma_{j}\eta_{j}W_{n}\nabla f(x_{j})-\sum_{j\in[0:i]}\gamma_{j}\eta_{j}h'(y_{j+1})\\
=&y_{i+1}.
\end{align*}
For $x$-iterate, it shows that
\begin{align*}
z_{i+1}=&x_{0}-\sum_{j\in[0:i-1]}t_{i,j}\eta_{j}\nabla f(x_{j}) -\sum_{j\in[0:i-1]}t_{i,j}\eta_{j}h'(\hat{y}_{j+1}) -\frac{1}{\gamma_{i}}(\gamma_{i}\eta_{i}W_{n}\nabla f(x_{i})+\gamma_{i}\eta_{i}h'(y_{i+1}))\\
=& x_{0}-\sum_{j\in[0:i-1]}t_{i,j}\eta_{j}W_{n}\nabla f(x_{j}) - \sum_{j\in[0:i-1]}t_{i,j}\eta_{j}h'(\hat{y}_{j+1}) -\eta_{i}W_{n}\nabla f(x_{i})-\eta_{i}h'(y_{i+1}).
\end{align*}
Based on the above equation, we have
\begin{align*}
\hat{x}_{i+1}=&z_{i+1}+\frac{\alpha_{i}-1}{\alpha_{i+1}}(z_{i+1}-z_{i})+\frac{\alpha_{i}}{\alpha_{i+1}}(z_{i+1}-\hat{x}_{i})\\
=&x_{0}-\sum_{j\in[0:i-1]}t_{i,j}\eta_{j}W_{n}\nabla f(x_{j}) -\sum_{j\in[0:i-1]}t_{i,j}\eta_{j}h'(y_{j+1}) -\eta_{i}W_{n}\nabla f(x_{i}) -\eta_{i}h'(y_{i+1})\\
&-\frac{\alpha_{i}-1}{\alpha_{i+1}}(\sum_{j\in[0:i-2]}s_{i,j}\eta_{j}W_{n}\nabla f(x_{j})+\sum_{j\in[0:i-2]}s_{i,j}\eta_{j}h'(y_{j+1}))\\
&-\frac{\alpha_{i}-1}{\alpha_{i+1}}((t_{i,i-1}-1)\eta_{i-1}W_{n}\nabla f(x_{i-1}) +(t_{i,i-1}-1)\eta_{i-1}h'(y_{i}))\\
&-\frac{\alpha_{i}-1}{\alpha_{i+1}}(\eta_{i}W_{n}\nabla f(x_{i}) +\eta_{i}h'(y_{i+1})) -\frac{\alpha_{i}}{\alpha_{i+1}}(\eta_{i}W_{n}\nabla f(x_{i}) +\eta_{i}h'(y_{i+1})).
\end{align*}
From Lemma \ref{lemma2}, it shows that
\begin{align*}
&\frac{\alpha_{i}-1}{\alpha_{i+1}} s_{i,j}=s_{i+1,j}\quad \text{ for } j\in[0:i-2],\\
&\frac{\alpha_{i}-1}{\alpha_{i+1}}(t_{i,i-1}-1)=s_{i+1,i-1}.
\end{align*}
Thus, we have
\begin{align*}
\hat{x}_{i+1}=&x_{0}-\sum_{j\in[0:i-1]}t_{i,j}\eta_{j}W_{n}\nabla f(x_{j}) -\sum_{j\in[0:i-1]}t_{i,j}\eta_{j}h'(y_{j+1}) -\eta_{i}W_{n}\nabla f(x_{i}) -\eta_{i}h'(y_{i+1})\\
&-\sum_{j\in[0:i-2]}s_{i+1,j}\eta_{j}W_{n}\nabla f(x_{j}) -\sum_{j\in[0:i-2]}s_{i+1,j}\eta_{j}h'(y_{j+1}) -s_{i+1,i-1}\eta_{i-1}W_{n}\nabla f(x_{i-1})\\
&-s_{i+1,i-1}\eta_{i-1}h'(y_{i})-\frac{\alpha_{i}-1}{\alpha_{i+1}}(\eta_{i}W_{n}\nabla f(x_{i})+\eta_{i}h'(y_{i+1}))\\
&-\frac{\alpha_{i}}{\alpha_{i+1}}(\eta_{i}W_{n}\nabla f(x_{i})+\eta_{i}h'(y_{i+1}))\\
=& x_{0}-\sum_{j\in[0:i-1]}t_{i+1,j}\eta_{j}W_{n}\nabla f(x_{j}) -\sum_{j\in[0:i-1]}t_{i+1,j}\eta_{j}h'(y_{j+1}) \\
&-\eta_{i}(1+\frac{2\alpha_{i}-1}{\alpha_{i+1}})W_{n}\nabla f(x_{i}) -\eta_{i}(1+\frac{2\alpha_{i}-1}{\alpha_{i+1}})h'(y_{i+1})\\
=&x_{0}- \sum_{j\in[0:i]}t_{i+1,j}\eta_{j}W_{n}\nabla f(x_{j}) -\sum_{j\in[0:i]}t_{i+1,j}\eta_{j}h'(y_{j+1})\\
=&x_{i+1}.
\end{align*}
Based on mathematical induction, we finish the proof.
\end{proof}

Now we are ready to prove Lemma \ref{lemma1}.
\begin{proof}
From Lemma 18 in \cite{JangGR23}, it shows that 
\[
\frac{\alpha_{j+1}}{2\alpha_{j}-1}(t_{K,j}-1) -1 = \frac{\alpha_{j+1}-1}{2\alpha_{j+1}-1}(t_{K,j+1}-1).
\]
Based on the above equation, we can get that  the equation 
\[
\gamma_{j}=\frac{2\alpha_{j}}{\alpha_{K}^{2}}(\alpha_{K}^{2}-2\alpha_{j}^{2}+\alpha_{i}) =t_{K,j}
\]
holds for all $j\in[0:K-1]$. From \eqref{equivalent_xy} in Lemma \ref{lemma3} with $i=K-1$, we have 
\begin{align*}
x_{K}=&x_0-\sum_{j\in[0:K-1]}t_{K,j}\eta_{j}W_{n}\nabla f(x_j)- \sum_{j\in[0:K-1]}t_{K,j}\eta_{j}h'(y_{j+1})\\
=&x_0-\sum_{j\in[0:K-1]}\gamma_{j}\eta_{j}W_{n}\nabla f(x_j)- \sum_{j\in[0:K-1]}\gamma_{j}\eta_{j}h'(y_{j+1})\\
=&y_{K}.
\end{align*}
This completes the proof.
\end{proof}
\bibliographystyle{plain}

\bibliography{IOPGA}

\end{document}